\documentclass[11pt]{article}
\usepackage{fullpage}
\usepackage[T1]{fontenc}
\usepackage[utf8]{inputenc} 

\usepackage{microtype}
\usepackage{graphicx}
\usepackage{placeins}    
\usepackage{dblfloatfix} 
\usepackage{booktabs}

\usepackage{amsmath,amssymb,mathtools,amsthm,dsfont}

\usepackage{algorithm}
\usepackage{algpseudocode}

\usepackage{subcaption}

\usepackage[numbers]{natbib}
\usepackage{url}

\usepackage{xcolor}
\usepackage[textsize=tiny]{todonotes}

\usepackage{authblk}

\usepackage[hidelinks]{hyperref}
\usepackage[nameinlink,capitalize,noabbrev]{cleveref}

\usepackage{tabularx}

\newcommand{\E}{\mathbb E}
\newcommand{\R}{\mathbb R}
\newcommand{\N}{\mathbb N}
\newcommand{\A}{\mathbb A}

\newcommand{\Var}{\operatorname{Var}}
\newcommand{\Cov}{\operatorname{Cov}}

\newcommand{\1}{\mathds{1}}

\theoremstyle{plain}
\newtheorem{theorem}{Theorem}[section]
\newtheorem{proposition}[theorem]{Proposition}
\newtheorem{lemma}[theorem]{Lemma}

\theoremstyle{definition}
\newtheorem{definition}[theorem]{Definition}
\newtheorem{assumption}[theorem]{Assumption}
\theoremstyle{remark}
\newtheorem{remark}[theorem]{Remark}
\renewcommand{\P}{\mathbb P}

\makeatletter
\newcommand*\bigcdot{\mathpalette\bigcdot@{.9}}
\newcommand*\bigcdot@[2]{\mathbin{\vcenter{\hbox{\scalebox{#2}{$\m@th#1\bullet$}}}}}
\makeatother

\title{An Approximate Ascent Approach To Prove Convergence of PPO}

\author[1]{Leif D\"{o}ring}
\author[1]{Daniel Schmidt}
\author[1]{Moritz Melcher}
\author[2]{Sebastian Kassing}
\author[1]{Benedikt Wille}
\author[1]{Tilman Aach}
\author[1]{Simon Weissmann}

\date{}

\affil[1]{\normalsize
  Institute of Mathematics,  
  University of Mannheim, 
    68138 Mannheim, Germany\\

    \texttt{\{leif.doering, daniel.schmidt, moritz.melcher, benedikt.wille, tilman.aach,  simon.weissmann\}@uni-mannheim.de}
}
\affil[2]{\normalsize
  Department of Mathematics \& Informatics,  
  University of Wuppertal,
    42119 Wuppertal, Germany\\
  \texttt{\{kassing\}@uni-wuppertal.de}
}

\begin{document}

\maketitle

\begin{abstract}
Proximal Policy Optimization (PPO) is among the most widely used deep reinforcement learning algorithms, yet its theoretical foundations remain incomplete. Most importantly, convergence and understanding of fundamental PPO advantages remain widely open. Under standard theory assumptions we show how PPO's policy update scheme (performing multiple epochs of minibatch updates on multi-use rollouts with a surrogate gradient) can be interpreted as approximated policy gradient ascent. We show how to control the bias accumulated by the surrogate gradients and use techniques from random reshuffling to prove a convergence theorem for PPO that sheds light on PPO's success. Additionally, we identify a previously overlooked issue in truncated Generalized Advantage Estimation commonly used in PPO. The geometric weighting scheme induces infinite mass collapse onto the longest $k$-step advantage estimator at episode boundaries. Empirical evaluations show that a simple weight correction can yield substantial improvements in environments with strong terminal signal, such as Lunar Lander.
\end{abstract}

\section{Introduction}
Reinforcement learning (RL) has emerged as a powerful paradigm for training autonomous agents to make sequential decisions by interacting with their environment~\citep{sutton2018reinforcement}. In recent years, policy gradient methods have become the foundation for many successful applications, ranging from game playing~\citep{silver2016mastering, berner2019dota} to robotics~\citep{levine2016end, openai2019solving} and large language model alignment~\citep{ouyang2022training, christiano2017deep}. 
At its core, policy gradient methods aim to optimize a policy $\pi_\theta$ by following the gradient of the parametrized expected total return $J(\theta)$. The policy gradient theorem~\citep{sutton1999policy, williams1992simple} provides an unbiased estimator of this gradient, which can be computed using samples from the current policy. 
Actor-critic methods leverage this idea by alternating between collecting rollouts, estimating a value function (critic), and updating the policy. A canonical example is A2C~\cite{A2C}, which performs a single update per batch before collecting fresh on-policy data.
Among these methods, Proximal Policy Optimization (PPO) \cite{PPO} has become one of the most widely adopted algorithms due to its simplicity, stability, and empirical performance. PPO was introduced as a practical, first-order approximation of Trust Region Policy Optimization (TRPO) \cite{TRPO}. TRPO proposes to improve a reference policy $\pi_{\theta_{\text{old}}}$ by maximizing a surrogate objective subject to a trust-region constraint that limits the change of the policy:
\begin{align*}
&\mathrm{maximize}_\theta \,\hat{\mathbb{E}}_t\big[r_t(\theta)\hat{\mathbb A}^{\pi_\text{old}}_t\big],
\\ \text{ subject to } \quad&\hat{\mathbb{E}}_t\big[\mathrm{KL}\!\left(\pi_{\theta_{\text{old}}}(\cdot \,;\, s_t)\,\|\,\pi_{\theta}(\cdot \,;\, s_t)\right)\big]\ \le\ \delta,
\end{align*}
where $r_t(\theta) = \frac{ \pi_\theta(a_t;s_t)}{\pi_{\theta_{\text{old}}}(a_t;s_t)}$ and $\hat{\A}^{\pi_\text{old}}_t$ is an advantage estimate. 
Solving the constrained problem, however, requires second-order information. 
PPO was introduced as an implementable relaxation of the trust-region principle. 
In the clipped variant from \citep{PPO}, the objective is
\begin{align*}
L(\theta) = \hat{\mathbb{E}}_t\big[\min\big(r_t(\theta)\hat{\mathbb A}_t, \text{clip}(r_t(\theta), 1-\epsilon, 1+\epsilon)\hat{\mathbb A}_t\big)\big],
\end{align*}
where $\hat{\mathbb A}_t$ is obtained from (truncated) Generalized Advantage Estimation (GAE) \cite{GAE}, computed under the old policy $\pi_{\theta_{\text{old}}}$. While the practical success of PPO speaks for itself, the theoretical understanding of PPO remains largely open and even the decisive advantages in practice are hard to identify \cite{engstrom2020implementation}. For instance, there seems to be no convergence result that takes into account the sample reuse with a transition buffer which is shuffled randomly in each epoch. 
The main reason for lack of theory might be that the connection to TRPO is rather heuristic and thus hard to use as a basis for theorems. We contribute to the fundamental questions:\\

\emph{What is a good theoretical grounding of PPO and what can be learned from theory for practical applications?}\\

Our paper changes perspective. We ignore the connection to TRPO and solely rethink PPO as policy gradient with well-organized sample reuse. PPO has a cyclic structure, with one A2C update step followed by a number of surrogate gradients steps, see Figure \ref{fig:pfeile}. The relation of A2C and first cycle steps of PPO was observed earlier in  \cite{A2CisPPO}. 
Blue arrows in the visualization represent A2C gradient steps, orange arrows additional PPO surrogate gradient steps which become less trustable as cycles progress.\\
 \begin{figure}[t]
 \vspace{-2mm}
    \centering
    \includegraphics[scale=0.22]{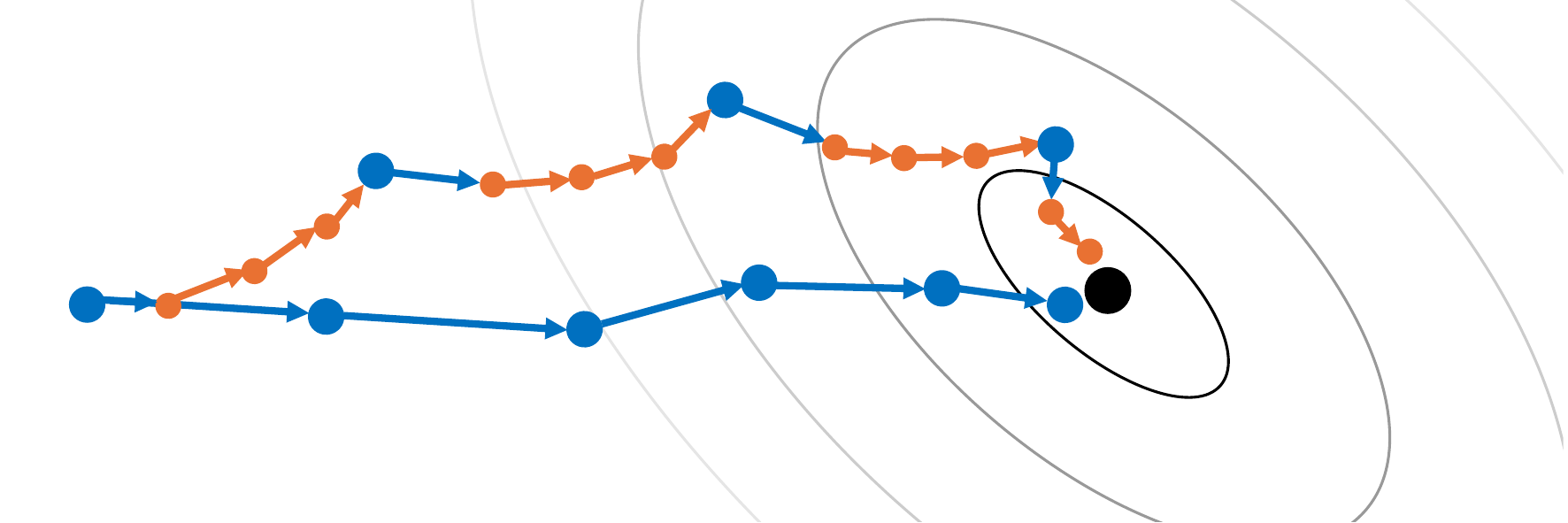}
   \vspace{-2mm}
\caption{Schematic view of PPO vs. A2C policy parameter updates. Blue arrows are policy gradient $\nabla_\theta J(\theta)$ steps, orange arrows cycles of increasingly biased surrogate gradient $g_\text{PPO}^\text{clip}(\theta,\theta_\text{old})$ steps. For stochastic approximations blue dots are resampling times.
} 
 \label{fig:pfeile}
 \end{figure}

\emph{Main contributions} to PPO theory and practice:
\vspace{-3mm}
\begin{itemize}\setlength{\itemsep}{0.2em}
    \setlength{\parskip}{0pt}
    \setlength{\parsep}{0pt}
    \item A formalization through gradient surrogates $g_\text{old}^\text{clip}$ of $\nabla_\theta J$   
    is provided (Section \ref{sec:4}) so that their stochastic approximations are close to practical PPO implementations, using most features of PPO's update mechanism (we skip KL-regularization and asymmetric clipping). 
    \item Bias estimates of the form $|\nabla_\theta J(\theta)-g_\text{ppo}^\text{clip}(\theta,\theta_\text{old})|\leq C|\theta-\theta_\text{old}|$ for exact gradients are derived in Theorem~\ref{Thm:1}. This shows how trustworthy the orange arrows in Figure \ref{fig:pfeile} can be.
    \item Convergence proofs are presented in Theorem~\ref{thm:convdeterministc} and \ref{Thm:3} that show the effect of  additional biased surrogate gradients on stochastic and deterministic policy gradient. We connect PPO to random reshuffling (RR) theory. The analysis shows that PPO’s cycle-based update structure implicitly  controls the effective step length through aggregation of clipped gradient estimates.
    \item Practical application: Our consequent finite-time modeling of PPO highlights a side-effect of PPO's truncating of GAE at finite horizons. We call the effect tail-mass collapse and suggest a simple fix. Experiments show  significant improvement on Lunar Lander.
\end{itemize}
\section{Policy Gradient Basics}
While many policy gradient results are stated in the infinite-horizon discounted setting, we directly work with the finite-horizon truncation, to stay close to actual PPO implementations. We assume finite state and action spaces $\mathcal S$ and $\mathcal A$ and a fixed initial distribution $\mu$. Value functions are denoted by $V^\pi_t(s)=\E^\pi[\sum_{i=t}^{T-1}\gamma^{i-t} R_i\,|\,S_t=s]$ and $Q^\pi_t(s,a)=\E^\pi[\sum_{i=t}^{T-1}\gamma^{i-t} R_i\,|\,S_t=s,A_t=a]$. Furthermore, we work with a differentiable parametrized policy class $\{\pi_\theta\}_{\theta \in \R^d}$ with so-called score function $\nabla_\theta \log \pi_\theta(a\, ;\,s)$. The optimization goal is to maximize the parametrized value function $J(\theta):=V^{\pi_\theta}_0(\mu)=\sum_{s\in \mathcal S} V^{\pi_\theta}_0(s)\mu(s)$. If rewards are assumed bounded, then $J_\ast := \sup_{\theta}\ J(\theta)$ exists. By the likelihood-ratio identity, the policy gradient admits the stochastic gradient representation $\nabla_\theta J(\theta)=\sum_{t=0}^{T-1}\gamma ^t\,  \E^{\pi_\theta}[\nabla_\theta \log\pi_\theta(A_t\,;\,S_t)\,  R^T_t]$ with rewards-to-go defined as $R_t^T\coloneq\sum_{i=t}^{T-1}\gamma^{i-t}R_i$. 
The resulting simple policy gradient estimator is well-known to be too noisy for practical optimization due to high variance.  To reduce variances, the commonly used policy gradient representation uses averaged rewards-to-go and subtracts a baseline. In the discounted finite-time setting this is $\nabla_\theta J(\theta)=\sum_{t=0}^{T-1} \gamma^t \, \E^{\pi_\theta}[\nabla_\theta \log\pi_\theta(A_t\,;\,S_t) \, \mathbb A_{t}^{\pi_\theta}(S_t,A_t)]$, where $\mathbb A^{\pi_\theta}_t=Q_t^{\pi_\theta}-V_t^{\pi_\theta}$ is called the advantage function. Algorithmically, the advantage policy gradient creates a structural difficulty. In order to improve the actor using gradient ascent, the current policy needs to be evaluated, i.e. $\mathbb A_t^{\pi_\theta}$ needs to be computed/estimated. The algorithmic solution is what is called actor-critic. Advantage actor-critic algorithms alternate between gradient steps to improve the policy and estimation steps to create estimates $\hat{\mathbb A}_t^{\pi_\theta}$. A2C \cite{A2C} implements actor-critic using neural networks for advantage modeling and GAE for advantage estimation.
\begin{remark}
    It is known that implementations of actor-critic algorithms usually ignore the discount factor $\gamma^t$, see for instance \cite{Thomas} for theoretical considerations and \cite{zhang2022deeperlookdiscountingmismatch} and \cite{che2023correctingdiscountfactormismatchonpolicy} for experiments. Since the factor is a mathematical necessity for the approximated infinite-horizon problems, we keep $\gamma^t$ and acknowledge that omitting $\gamma$ is not  harmful. This is in line with the formalism-implementation mismatch discussed in \cite{formalismimplementationgap}, who suggested to focus on structural understanding of RL algorithms than artifacts that improve benchmarks.
\end{remark}

\section{Related Work}
The present article continues a long line of theory papers proving convergence of policy gradient algorithms, but to the best of our knowledge there is little work on PPO style policy update mechanisms. The analysis of policy gradient is generally challenging due to the non-convex optimization landscape, one typically relies on $L$-smoothness of the objective, which holds under reasonably strong assumptions on the policy class. Some works concern convergence to stationary points \citep{Zhang, Papini}. Under strong policy assumptions, such as a tabular softmax parametrization, one can prove additional structure conditions including gradient domination, and deduce convergence to global optima and rates (e.g.  \citep{agarwal, mei2022global, samu}). For theoretical results concerning actor-critic algorithms we refer, for instance, to \cite{Kumar} and references therein. Most theory articles analyze convergence in the discounted infinite-time MDP setting. Since implementations force truncation for PPO, we decided to work in finite time. In finite time, optimal policies are not necessarily stationary, a policy gradient algorithm to find non-stationary policies was developed in \cite{sara1}. In the spirit of PPO the present article analyzes the search for optimal stationary policies.

In contrast to the vast literature on vanilla policy gradient convergence theory on PPO is more limited. Reasons are the clipping mechanism and, most importantly, surrogate bias and reuse of data.  \cite{liu} gave a convergence proof of a neural PPO variant, using infinite dimensional mirror descent. For two recent convergence results we refer to \cite{PPO2} and \cite{PPO1} noting that both do not allow for reuse of reshuffled data. \cite{PPO2} essentially proves that surrogate gradient steps do not harm the original policy gradient scheme, while \cite{PPO1} work in a specific policy setting (in particular probabilities bounded away from $0$ that allow to show gradient domination properties. We are not aware of results incorporating sample reuse. To incorporate multi-sample use our work is build on previous results on random reshuffling in the finite-sum setting relevant for supervised learning, e.g. \cite{Mishenko}. We also refer to \cite{RR} and references therein.

Finally, since we also contribute to the practical use of GAE estimators, let us mention some related work. While GAE was introduced for infinite-time MDPs, it was suggested in \cite{PPO} to be used for finite time by direct truncation at $T$. Truncation of GAE to subsets of trajectories was recently used in the context of LLMs \cite{bytedanceseed2025truncated} but also for classical environments \cite{Song}. 
To our knowledge, both our observation of tail-mass collapse and the reweighting of the collapsed mass are novel.\\

Typical assumptions in the mentioned theory articles are bounded rewards and bounded/Lipschitz score functions. We will work under these assumptions as they allow us to use ascent inequalities. Additionally we assume access to an well-behaved critic.
\begin{assumption}\label{Ass}
    \begin{itemize}\setlength{\itemsep}{0.2em}
    \setlength{\parskip}{0pt}
    \setlength{\parsep}{0pt}
        \item Bounded rewards: $|R_t|\leq R_\ast$.
        \item Bounded score function: $\forall \theta\in\R^d, s\in \mathcal S, a\in \mathcal A$:
        $$\big|\nabla_\theta \log \pi_\theta(a;s)\big|\leq \Pi_\ast$$
        \item Lipschitz score function: $\forall \theta\in\R^d, s\in \mathcal S, a\in \mathcal A$: $$\big|\nabla_\theta \log \pi_\theta(a;s)-\nabla_{\theta'} \log \pi_{\theta'}(a;s)|\leq L_s\,\big|\theta-\theta'|$$ 
        \item Access to advantage estimates $\hat{\mathbb A}_t$ that are bounded by $A_\ast$ with uniform estimation bias $\delta$.
    \end{itemize}
\end{assumption}

\section{Rethinking PPO}\label{sec:4}
In contrast to the origins of policy gradient schemes, PPO was not introduced as a rollout based stochastic gradient approximation, but rather as a direct algorithm with (a very successful) focus on implementation details. 
For our analysis, we introduce PPO differently by deriving surrogates $g_\text{PPO}^\text{clip}$ of the exact policy gradients $\nabla_\theta J$ for which PPO is the natural stochastic approximation. We first motivate why adding   $g_\text{PPO}^\text{clip}$ surrogate gradient steps to policy gradient is reasonable and then show that biased gradient ascent with cyclic use of $\nabla_\theta J$ and  $g_\text{PPO}^\text{clip}$ (see Figure \ref{fig:pfeile}) indeed has theoretical advantages (Section \ref{sec:deterministic}). Finally, we give a convergence result for PPO (Section \ref{sec:6}).

Here is a starting point. It is well-known that multi-use of data is a successful convergence speedup in supervised learning, in particular in combination with random reshuffling. In random reshuffling (RR), mini-batches of samples are reused for multiple SGD steps, with reshuffling between entire passes over the data (called  epochs). In online RL this is problematic because gradient steps depend on the current sampling policy. A cyclic variant is required, where the inner loop performs RR policy updates on data collected at the beginning of a loop. A principled way to decouple sampling and updating is importance sampling (IS):
\begin{align*}
         \nabla_\theta J(\theta)
        &=\sum_{t=0}^{T-1} \gamma^t \E^{\pi_{\theta_\text{old}}}\Big[\prod_{i=0}^t
        \frac{\pi_{\theta}(A_i\,;\,S_i))}{\pi_{\theta_\text{old}}(A_i\,;\,S_i))}
         \nabla_\theta \log\pi_\theta(A_t\,;\,S_t) \,  \mathbb A_{t}^{\pi_\theta}(S_t,A_t)\Big], 
    \end{align*}
where $\pi_{\theta_\text{old}}$ is a policy used for sampling rollouts. In principle one could try to implement data reuse as follows: sample rollouts at some parameter instance $\theta_\text{old}$, use mini-batches to estimate gradients and perform IS-corrected gradient steps for a few passes over the rollout data. Then denote the current parameter as $\theta_\text{old}$ and start the next cycle. While policy gradient (or A2C) performs a sampled gradient step with fresh rollout data at $\theta_\text{old}$, policy gradient with sample reuse performs one sampled policy gradient step at $\theta_\text{old}$ and additionally a cycle of IS-gradient steps using fixed rollouts.\\

Problems: (i) Importance sampling weights might pile up over $t$ and force huge variances when $\pi_\theta$ strides away from the sampling distribution $\pi_{\theta_\text{old}}$. (ii) Policy gradients involve $\mathbb A^{\pi_\theta}_t$ although data comes from $\pi_{\theta_\text{old}}$, which can not be estimated using GAE from $\pi_{\theta_\text{old}}$ rollouts. (iii) Importance ratios force rollout-based estimators while transition-based estimators have smaller variances (reduced time-correlations).\\

We now argue that PPO can be understood as an  implementable response to (i)-(iii). To address (i)-(iii) one
\begin{itemize}\setlength{\itemsep}{0.2em}
    \setlength{\parskip}{0pt}
    \setlength{\parsep}{0pt}
    \item drops all but one IS ratio and clip the remaining ratio,
    \item replaces $\mathbb{A}^\pi_t$ by $\mathbb{A}^{\pi_{\theta_\text{old}}}_t$ to allow GAE  from $\pi_{\theta_\text{old}}$ rollouts,
    \item interprets $\frac{1}{T}\sum_{t=0}^T$ as an expectation and estimate it by sampling uniformly from a transition buffer obtained by flattening the rollout buffer (requires the first step).
\end{itemize}
The first two lead to the approximation
  \begin{align*}
        &\quad\sum_{t=0}^{T-1} \gamma^t \E^{\pi_{\theta_{\text{old}}}}\Big[
        \frac{\pi_{\theta}(A_t\,;\,S_t)}{\pi_{\theta_{\text{old}}}(A_t\,;\,S_t)}\mathds{1}_{|\frac{\pi_\theta(A_t;S_t)}{\pi_\text{old}(A_t;S_t)}-1|\leq \epsilon}
        \times \nabla_\theta \log\pi_\theta(A_t\,;\,S_t) \, \mathbb A_{t}^{\pi_{\theta_{\text{old}}}}(S_t,A_t)\Big]\\
         &=\sum_{t=0}^{T-1} \gamma^t \E^{\pi_{\theta_\text{old}}}\Big[
        \frac{\nabla_\theta\pi_{\theta}(A_t\,;\,S_t)}{\pi_{\theta_{\text{old}}}(A_t\,;\,S_t)}
     \times \mathds{1}_{|\frac{\pi_\theta(A_t;S_t)}{\pi_\text{old}(A_t;S_t)}-1|\leq \epsilon}
        \mathbb A_{t}^{\pi_{\theta_\text{old}}}(S_t,A_t)\Big]\\
     &=:g^\text{clip}_{\text{PPO}}(\theta,{\theta_{\text{old}}})
    \end{align*}
of $\nabla_\theta J(\theta)$, where we used that $\pi_\theta(a;s)\nabla_\theta \log\pi_\theta(a;s)= \nabla_\theta \pi_\theta(a;s)$. This is exactly a formal expression for the expected PPO gradient surrogate. The uniform sampling view will lead to the true PPO sampler in Section \ref{sec:6}. Compared to PPO there are two minor changes.
\begin{remark}
    Discounting by $\gamma^{t}$ is ignored in PPO, not here. Our theory can be equally developed with $\gamma=1$. Next, we clip IS-ratios independently of the advantage, while PPO uses asymmetric clipping. We do not get into asymmetric clipping because the choice is very much problem dependent (see for instance \cite{tapered} in the LLM context).
\end{remark}
Let us emphasize that delayed advantages and  dropping/clipping IS ratios introduce bias. In order to not prevent convergence, the bias should not grow too fast during update cycles. While the fact is known, one main result of this paper provides bounds on the surrogate gradient bias:
\begin{theorem}[Surrogate gradient bias control]\label{Thm:1}
    Under Assumption \ref{Ass}, one has
    \begin{align*}
        \big|\nabla_\theta J(\theta)-g^\text{clip}_{\text{PPO}}(\theta,{\theta_{\text{old}}})\big|\leq R\,|\theta-\theta_\text{old}|,
    \end{align*}
    with a constant $R$ that is detailed in  Theorem \ref{Thm:1'}.
\end{theorem}
We detailed constants so that the interested reader can readily use the bound for $\gamma=1$ or $T=\infty$. 
\begin{proof}[Sketch of proof]
We use the performance difference lemma to write 
\begin{align*}
    \nabla_\theta J(\theta)
    &=\nabla_\theta \big( V(\pi_\theta)-V(\pi_{\theta_\text{old}})\big)
    = \nabla_\theta \sum_{t=0}^{T-1}\gamma^t \E^{\pi_\theta}\big[\mathbb A_t^{\pi_{\theta_\mathrm{old}}}(S_t,A_t)\big],
\end{align*}
from which it follows that 
\begin{align*}
  \nabla_\theta J(\theta) - g_{\text{PPO}}(\theta,\theta_\text{old})
  = \sum_{t=0}^{T-1} \gamma^t \,\nabla_\theta \left(\E^{\pi_\theta}[g_t^\theta(S_t)] - \E^{\pi_{\theta_\mathrm{old}}}[g_t^\theta(S_t)]\right),
\end{align*}
with $g_t^\theta(s)\coloneq \E_{A\sim\pi_\theta(\,\cdot\,;\,s)}[\mathbb A_t^{\pi_{\theta_\mathrm{old}}}(s,A)]$. Here $g_\text{PPO}$ denotes the surrogate without clipping. The righthand side can be estimated with the total variation distance between $\pi_{\theta_\text{old}}$ and $\pi_\theta$ which for bounded score functions is linearly bounded in $|\theta-\theta_\text{old}|$. Finally, some importance ratio computations are used to include the clipping to the estimate. For the full details we refer to Appendix \ref{appendix:C}.
\end{proof}
Another view at Figure \ref{fig:pfeile} now better explains the intention of the figure. Per cycle exact gradient PPO performs one A2C step with additional (orange) surrogate gradients that become more biased (less aligned towards the optimum) as the scheme departs from the resampling points.

The theorem shows that as long as parameters remain in proximity (trust region) to the sampling parameters, the bias is small and policy gradient will not be harmed by additional surrogate gradient steps. Thus, adding sampled PPO-type surrogate gradient steps to A2C is sample-free, not necessarily dangerous, but has a number of advantages. These include variance reduction (less time-correlations) by using mini-batches of transitions instead of full rollouts  and more value network updates (e.g.  \cite{wang2025improvingvalueestimationcritically}).

\section{Deterministic Convergence}\label{sec:deterministic}
Before we turn to PPO in the light of cyclic RR SGD, let us discuss the simpler exact gradient situation. Suppose we have explicit access to gradients $\nabla_\theta J$ and additionally to exact surrogate gradients $g^\text{clip}_{\text{PPO}}$. We ask if there can be advantages to use the surrogate gradients when trying to optimize $J$. To mimic the situation of PPO later, we assume that the surrogate gradients can be used for free, i.e., we only count the number $C$ of   gradient steps using true policy gradients $\nabla_\theta J$. We assume that $J$ is $L$-smooth (see Proposition \ref{lem:smooth}) and compare the method to a standard gradient ascent method where the optimal step-size is known to be $\eta=\frac{1}{L}$. It turns out that this question is highly dependent on the problem parameters, such as the Lipschitz constant of the gradient and the error at initialization.
As an example take $f(x)=-x^2$, then gradient ascent converges in one step. Additional biased gradient step worsen the convergence.\\

In practice, the smoothness constant $L$ is unknown, and for $\eta\ll\frac{1}{L}$ the situation is much clearer. In this regime, additional biased gradient steps can, in fact, be beneficial. Suppose that there are cycles $c=0,...,C-1$ of length $K$. The update rule is
\begin{align*}
    \theta_{c,e+1}&=\theta_{c,e}+\eta\, g_\text{PPO}^\text{clip}(\theta_{c,e},\theta_{c,0}), \quad e =0, \dots, K-1 \\
    \theta_{c+1,0}&=\theta_{c,K}
\end{align*}
Since $g_\text{PPO}^\text{clip}(\theta_{c,0},\theta_{c,0})=\nabla_\theta J(\theta_{c,0})$, the cyclic surrogate gradient ascent method performs correct gradient steps followed by increasingly biased surrogate gradient steps  (compare Figure \ref{fig:pfeile} with $K=4$). With the ascent lemma and Theorem~\ref{Thm:1}, we can derive the following convergence of $J$ along the parameter sequence.
\begin{theorem}[Deterministic PPO convergence]
\label{thm:convdeterministc}
    Suppose $C$ is the number of cycles, $K$ the cycle length, $R$ is from Theorem \ref{Thm:1}, $G$ is from Proposition \ref{prop:surrogategradbound}, $\Delta_0:=J_\ast-J(\theta_{0,0})$ is the initial optimality gap, and $L$ is the Lipschitz constant of $\nabla_\theta J$ ($J$ is $L$-smooth, see Proposition~\ref{lem:smooth}). If the learning rate $\eta$ is smaller than $\frac1L$, then
 \begin{align*}  
 \min_{c=0,\dots, C-1,\ e=0,\dots, K-1} |\nabla_\theta J(\theta_{c,e})|^2\le \frac{2\Delta_0}{\eta CK} + \frac 16 \eta^2(K-1)(2K-1) R^2G^2.
 \end{align*}
\end{theorem}
The proof is given in Appendix \ref{app:proofdeterministic}. Choosing $K=1$ recovers the standard convergence rate of gradient ascent for $L$-smooth functions, where the optimal step-size is given by $\eta=\frac{1}{L}$. For $K>1$ let us consider the looser upper bound $\frac{2\Delta_0}{\eta CK} + \frac 13 \eta^2K^2 R^2G^2$. Optimizing for fixed cycle length $K$ yields an optimal step size $\eta_\ast =\min(\frac{c}{K},\frac1L)$ and optimizing for fixed $\eta\le \frac 1L$ yields an optimal cycle length\footnote{For simplicity, we allow $K_*$ to be a non-integer here.} $K_\ast = \frac{c}{\eta}$ with $c= (\frac{3\,\Delta_0}{C R^2 G^2})^{\frac{1}{3}}$. If $\eta^* < \frac 1L$, both cases result in 
    \begin{align*}
        \min_{0\le c\le C-1,\ 0\le e\le K-1} |\nabla_\theta J(\theta_{c,e})|^2\le \left(\frac{3\Delta_0 RG}{C}\right)^{\frac23}\,.
    \end{align*}
    In conclusion, for a \emph{fixed} budget of $C$ exact gradient steps, additional biased gradient step improve the convergence if $\eta \ll \frac{1}{L}$ or, in the case of optimal (but typically unknown) step-size $\eta=\frac{1}{L}$, if $\Delta_0$ and $L$ are large compared to $R$ and $G$.    
    As we discuss in the next section, this is exactly the advantage of PPO. Estimated surrogate gradients can help compensate overly small learning rates but come at no additional sampling cost, they only use rollouts generated for the first gradient step of a cycle (blue dots in Figure \ref{fig:pfeile}).

\section{Reshuffling Analysis: Convergence of PPO}\label{sec:6}

\begin{algorithm}[t]
\caption{Cyclic Reshuffled PPO Surrogate Ascent}
\label{algo:cyclic}
\begin{algorithmic}[1] 
\Require Initial $\theta$, stepsize $\eta$, cycles $C$, epochs $K$, batch size $B$, $m := N/B$.
\For{$c = 0,\ldots, C-1$} \Comment{cycles}
    \State $\theta_{\mathrm{old}} \gets \theta$ \Comment{fixed sampling parameter}
    \State Sample $n$ rollouts under $\pi_{\theta_{\mathrm{old}}}$
    \State Estimate advantages $\hat{\mathbb A}$ \Comment{critic step, e.g. using GAE}
    \State Fill buffer $\{(s^i,a^i,r^i,\hat{\mathbb A}^i,t^i)\}_{i=0}^{N-1}$
    \For{$e = 0,\ldots, K-1$} \Comment{epoch}
        \State Draw a random permutation $\sigma=(\sigma_0,\ldots,\sigma_{N-1})$
        \For{$k = 0,\ldots, m-1$} \Comment{minibatch updates}
            \State $\mathcal B_k \gets \{\sigma_{kB},\ldots,\sigma_{(k+1)B-1}\}$
            \State Compute surrogate $\hat g^{\mathrm{clip}}(\theta,\theta_{\mathrm{old}};\mathcal B_k)$ as in Eq.~\eqref{eq:minibatch-surrogate}
            \State $\theta \gets \theta + \eta\,\hat g^{\mathrm{clip}}(\theta,\theta_{\mathrm{old}};\mathcal B_k)$ \Comment{this is $\theta_{c,e,k+1}$}
        \EndFor
    \EndFor
\EndFor
\end{algorithmic}
\end{algorithm}

We now turn towards PPO, replacing the exact surrogate gradients in the deterministic analysis with sampled surrogate gradients using transition buffer samples. We emphasize that our contribution is primarily conceptual. Apart from the minor modifications (discounting in gradients, and symmetrically clipping around~$1$) this is a formalization of the standard PPO policy update mechanism. Based on this formalization, our main contribution is Theorem \ref{Thm:3} below.
\begin{remark}
    For the analysis, we assume access to reasonably well behaved advantage estimators. The abstract condition is for instance fulfilled in the toy assumption of an exact critic. The assumption made is as weak as possible for mathematical tractability and uniformly controls the estimation error. While the assumption is undesirable, the current state of deep learning theory (in value prediction) makes it unavoidable.
\end{remark}
We start by rewriting $g^\text{clip}_{\text{PPO}}(\theta,\theta_{\text{old}})$.
Choose time-steps $U\sim\mathcal{U}\{0,\dots,T-1\}$ uniformly and independent of the process.
Then, the surrogate time-sum can be written as a uniform expectation,
\begin{align*}
g^\text{clip}_{\text{PPO}}(\theta,\theta_{\text{old}})
&= T\;\E_{U\sim \mathcal U}\Big[\E^{\pi_{\theta_\text{old}}}\Big[\gamma^{U}\,
\frac{\nabla_\theta\pi_{\theta}(A_{U};S_{U})} {\pi_{\theta_\text{old}}(A_{U};S_{U})}
\mathds{1}_{\big|\frac{\pi_\theta(A_{U};S_{U})}{\pi_{\theta_\text{old}}(A_{U};S_{U})}-1\big|\le \epsilon}\;
\mathbb A_{U}^{\pi_{\theta_\text{old}}}(S_{U},A_{U})\Big]\Big],
\end{align*}
i.e., as a double expectation over a uniformly sampled time index and the MDP under $\pi_{\theta_{\text{old}}}$. In practice, this joint expectation is approximated cycle-wise from sampled transitions. Within a cycle, one fixes a sampling parameter $\theta_{\text{old}}$, collects $n$ rollouts of length $T$ under $\pi_{\theta_{\text{old}}}$, and computes all advantages using truncated (or finite-time) GAE (see Section \ref{sec:fGAE}). Next, one flattens the resulting data into a transition buffer $\{(s^i,a^i,r^i,\hat{\mathbb A}^i,t^i)\}_{i=0}^{N-1}$ of size $N:=nT,$  where $(s^i,a^i,r^i,t^i)$ range over all state-action-reward-time tuples  encountered in the rollouts and $\hat{\mathbb A}_i$ denotes the advantage estimate for $\mathbb A_{t_i}^{\pi_{\theta_{\text{old}}}}(s^i,a^i)$ computed from the rollout (e.g. in practice with GAE). Note that we append the standard transition buffers with the time-index of transitions in order to allow discounting of the gradient. This does not pose any practical difficulty.

Within a cycle, PPO implementations perform multiple passes over the transition buffer using reshuffled minibatches. We formalize this mechanism as follows. Let $\sigma=(\sigma_0,\dots,\sigma_{N-1})$ be a random permutation of $\{0,\dots,N-1\}$ (reshuffling), and partition the permuted indices into consecutive minibatches of size~$B$: for $k=0,\dots,m-1$ with $m:=\frac NB$, define $\mathcal B_k := \{\sigma_{kB},\dots,\sigma_{(k+1)B-1}\} \subseteq\{1,\dots,N\}$.
A single PPO update step uses the minibatch sampled surrogate gradient
\begin{align}\label{eq:minibatch-surrogate}
\hat g^{\text{clip}}(\theta,\theta_{\text{old}};\mathcal B_k)= \frac{1}{B}\sum_{i\in \mathcal B_k} g_{\text{PPO}}^{(i),\text{clip}}(\theta,\theta_{\text{old}}),
\end{align}
where the per-transition contribution is
\[
g_{\text{PPO}}^{(i),\text{clip}}(\theta,\theta_{\text{old}})
:= T\gamma^{t^i}
\frac{\nabla_\theta \pi_\theta(a^i;s^i)}{\pi_{\theta_\text{old}}(a^i;s^i)}\,
\mathds{1}_{\big|\frac{\pi_\theta(a^i;s^i)}{\pi_{\theta_\text{old}}(a^i;s^i)}-1\big|\le \epsilon}
\hat{\mathbb A}^i.
\]
An epoch corresponds to one pass over the buffer, i.e., iterating with $m$ steps once over the index batches $\mathcal B_0,\dots,\mathcal B_{m-1}$ generated by $\sigma$. PPO repeats this for a number $K$ epochs within the same cycle, drawing a fresh permutation at the beginning of each epoch and then passes through the data. For the convenience of the reader, we give pseudocode of our interpretation of the PPO policy update in Algorithm~\ref{algo:cyclic}. 

The above procedure generates a sequence of parameters $\theta_{c,e,k}$ indexed by cycle $c$, epoch $e$, and $k$th mini-batch update within the epoch. In the following, we denote by $\theta_{c,e,k}$ the parameter before the $(k+1)$st minibatch update within epoch $e$ of cycle $c$. In particular, $\theta_{c,0,0}$ is the initialization of cycle $c$ and plays the role of $\theta_{\text{old}}$ for that cycle, $\theta_{c,e,0}$ is the epoch start-point and $\theta_{c,e,m}$ is the epoch end-point. PPO can thus be seen as a cyclic RR method. Recall that RR is an SGD-style method for finite-sum objectives where, at the start of each epoch, one  permutes the data points and then takes mini-batch gradient steps using each data points once. Unlike SGD, which samples indices independently with replacement in each step, RR samples without replacement within an epoch, which often reduces redundancy and improves convergence. Regarding convergence results for RR in supervised learning that motivated our convergence proof, see \cite{Mishenko} and references therein. Here is our main convergence result for PPO:
\begin{theorem}\label{Thm:3}
    Assume Assumption \ref{Ass} 
    and suppose the learning rate $\eta$ is smaller than $\frac{1}{Lm}$. Then, for arbitrary $p,q \in (0,1)$, it holds that
\begin{equation} \label{eq:Thm3}
\begin{split}
 \min_{c=0,\dots, C-1,\ e=0,\dots, K-1} \E\big[|\nabla J(\theta_{c,e,0})|^2\big]
&\le \frac{2\Delta_0}{\eta CKm}+ 6\eta^2(7B_1^2+L^2)K^2m^2G^2 \\
&\quad + 42 \eta^{2p}B_2^2\epsilon^{2(1-p)}|\mathcal A|^{2p}\Pi_\ast^{2p} K^{2p}m^{2p}  G^{2p}\\ 
&\quad + \frac{42\eta^q B_3^2|\mathcal A|^q \Pi_\ast^q K^q m^q  G^q }{\epsilon^q}+ \frac{12\sigma^2}{N}+12T^2\Pi_\ast^2\delta^2\,,
\end{split}
\end{equation}
with constants from Theorem \ref{thm:convdeterministc} and additional constants given in Appendix \ref{app:PPOconv}.
\end{theorem}

In contrast to Section \ref{sec:deterministic}, the stochastic setting presents challenges that require more careful treatment. Most notably, the iterates are random and  stochastically dependent on the samples collected at the beginning of each cycle. As a consequence, the bias term used in the deterministic analysis developed in Theorem~\ref{Thm:1} can no longer be handled directly, since both quantities are now random and coupled through the sampling process. To overcome this issue, our analysis instead works directly with the sampled gradients generated within each cycle. Rather than comparing exact surrogate gradients evaluated at random iterates to the true gradient, we instead compare sampled gradients at intermediate steps to the sampled gradient at the beginning of the cycle. This path-level bias decomposition (see Lemma ~\ref{lem:stability}) allows us to control the dependence introduced by fresh sampling while still retaining a meaningful notion of gradient consistency. 

Setting $p=q=1$ and assuming $\eta \le \frac{1}{Lm}$, we can rewrite the upper bound in \eqref{eq:Thm3} as
$$
    \frac{2\Delta_0}{\eta CKm}+ c_1 \eta^{2} K^{2}m^{2}+ \frac{c_2 \eta  K m  }{\epsilon}+ \frac{12\sigma^2}{N}+12T^2\Pi_\ast^2\delta^2
$$
for suitable constants $c_1,c_2$. To better quantify this bound for small step-sizes we balance the terms for $\frac{1}{\eta}$ and $\eta$ (suppressing $\eta^2$) which yields the suitable cycle size 
\[ K = \frac{1}{\eta m} \sqrt{\frac{2 \Delta_0 \epsilon}{C c_2}}
\]
with corresponding upper bound
\begin{align*} 
\min_{c=0,\dots, C-1,\ e=0,\dots, K-1} \E\big[|\nabla J(\theta_{c,e,0})|^2\big] 
\le 2\sqrt{\frac{2\Delta_0 c_2}{C\epsilon}}
+
\frac{2 c_1 \Delta_0 \epsilon}{C c_2}
+
\frac{12\sigma^2}{N}+12T^2\Pi_\ast^2\delta^2\,.
\end{align*}
As in the deterministic situation, the results indicate that  cycle-based update schemes mitigate sensitivity to step-size selection. Small learning rates can be offset by additional updates reusing the same rollouts, without degrading the convergence guarantee. This behavior can be interpreted as an implicit trust-region mechanism, where many small clipped updates adaptively control the effective step length.
\begin{figure*}[t]
    \centering
    \includegraphics[width=\linewidth]{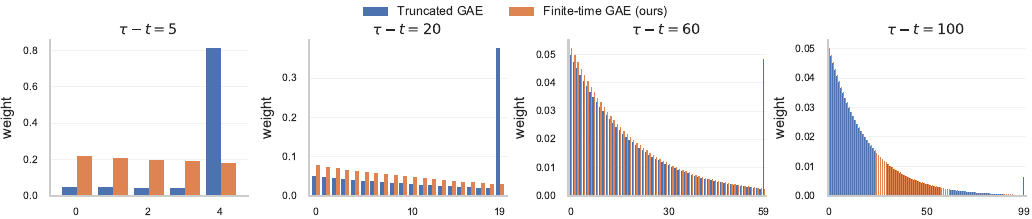}
    \vspace{-3mm}
    \caption{
    \textbf{GAE tail-mass collapse.}
    Comparison of the weights assigned to the $k$-step advantage estimators $\hat \A_t^{(k)}$ by truncated GAE and our finite-time (renormalized) estimator.
    Each panel corresponds to a different remaining horizon $\tau-t$ and shows how, for truncated GAE, the geometric tail mass collapses onto the last available estimator $\hat \A_t^{(\tau-t-1)}$, while the finite-time variant redistributes this mass over the observable $k$-step estimators via finite-horizon renormalization .
    }
    \label{fig:gae-tail-collapse}
\end{figure*}

\section{Finite-Time GAE}\label{sec:fGAE}
For our convergence theory we needed to work under abstract critic assumptions. In this section we reveal a theory-implementation gap that occurs in PPO (see (11), (12) in \cite{PPO}) when truncating the original GAE estimator. Details and proofs can be found in Appendix \ref{sec: finite-time-gae}.

Recall that original GAE for infinite MDPs works as follows. Motivated from the $k$-step Bellman expectation operator, the $k$-step forwards estimators $\mathbb A_t^{(k)}:=\sum_{t=0}^k \gamma^t R_t +\gamma^{k+1} V^\pi(S_{t+k+1})-V^\pi(S_t)$ are conditionally unbiased estimators of $\mathbb A^\pi(S_t,A_t)$. Replacing the true value function with a value network approximation $V$ the estimator is denoted by $\hat{\mathbb A}^{(k)}_t$. For large $k$ (close to the Monte Carlo advantage estimator) the estimation variance dominates, for small $k$ the value function approximation bias of bootstrapping dominates. Geometrical mixing of $k$-step estimators yields 
GAE $\hat{\mathbb A}_t^\infty:=(1-\lambda)\sum_{k=0}^{\infty} \lambda^k \hat{\mathbb A}_t^{(k)}$. There is a simple yet important trick that makes GAE particularly appealing. Using a telescopic sum cancellation shows $\hat \A_t^\infty
= \sum_{\ell=0}^{\infty}(\gamma\lambda)^\ell\,\delta_{t+\ell}$ with TD errors $\delta_t=R_t+\gamma V(S_{t+1})-V(S_t)$. For finite-time MDPs (or even terminated MDPs) the infinite time setting is not appropriate. 
In PPO (see (11) of \cite{PPO}) GAE is typically truncated by dropping TD errors after the rollout end $\tau$:
\begin{align*}
    \hat {\mathbb A}_t
:= \sum_{\ell=0}^{\tau-t-1}(\gamma\lambda)^\ell\,\delta_{t+\ell}.
\end{align*}
While in PPO $\tau=T$ is considered fixed, truncation can equally be applied at termination times. The truncated representation is particularly useful as it allows to backtrack $\hat{\mathbb A}_t=\delta_t+\gamma\lambda \hat{\mathbb A}_{t+1}$ using the terminal condition  $\hat{\mathbb A}_\tau:=0$. While the idea of GAE is a geometric mixture of $k$-step advantage estimators with weights $(1-\lambda)\lambda^k$, this breaks down when truncating. All mass of $k$-step estimators exceeding $\tau$ is collapsed onto the longest non-trivial estimator.
\begin{proposition}[Tail-mass collapse of truncated GAE]\label{prop:tail-collapse-kstep}
For $t\leq\tau-1$ the GAE estimator used in practice satisfies
\[
\hat \A_t
=
\sum_{k=0}^{\tau-t-2}\underbrace{(1-\lambda)\lambda^k}_{\text{GAE weights}}\,\hat \A_t^{(k)}
\;+\;
\underbrace{\lambda^{\tau-t-1}}_{\text{collapsed tail-mass}}\,\hat \A_t^{(\tau-t-1)}.
\]
\end{proposition}
We call this effect tail-mass collapse, see the blue bars of Figure \ref{fig:gae-tail-collapse}. Next, we suggest a new estimator that uses geometric weights normalized to fill only $\{0,....,\tau-t-1\}$.
\begin{definition}[Finite-time GAEs]
    We define the finite-time GAE estimators as
    \begin{align*}
        \hat \A_t^{\tau}
:=
\frac{1-\lambda}{1-\lambda^{\tau-t}}
\sum_{k=0}^{\tau-t-1}\lambda^k\,\hat \A_t^{(k)}.
    \end{align*}
    If $\tau=T$ the estimator is called fixed-time, otherwise termination-time GAE. 
\end{definition}
The orange bars in Figure \ref{fig:gae-tail-collapse} display the geometric weights of our finite-time GAE.
By renormalizing the geometric mass over the distinct $k$-step estimators supported by the available suffix $k\in\{0,\dots,\tau-t-1\}$, our estimator prevents the strong tail-mass collapse onto $\hat \A_t^{(\tau-t-1)}$ that occurs near the rollout end under truncated GAE (blue). 

Heuristically, the longest lookahead term $\hat \A_t^{(\tau-t-1)}$ is least affected by bootstrapping, hence it tends to incur smaller value-approximation (bootstrap) bias, but it typically has higher variance, since it aggregates the longest discounted sum of TD-errors.
Consequently, for fixed $\gamma$ and $\lambda$, our finite-time renormalization trades variance for bias and bootstrapping. At the same time, it restores the intended finite-horizon analogue of the geometric mixing interpretation of GAE, rather than implicitly collapsing the unobserved tail mass onto a single estimator. 
\begin{algorithm}[b]
\caption{Finite-time GAE (ours)}
\label{alg:gae}
\begin{algorithmic}[1]
\Require Policy $\pi_\theta$, value estimate $V$, discount $\gamma$, GAE parameter $\lambda$
\State Generate rollout $\{(s_t,a_t,r_t,v_t)\}_{t=0}^{\tau-1}$ with $v_t \gets V(s_t)$ for $t=0,\dots,\tau$
\State $\hat{\mathbb A}_\tau^\tau \gets 0$ \Comment{boundary condition}
\For{$t=\tau-1,\ldots,0$}
    \State $\delta_t \gets r_t + \gamma v_{t+1} - v_t$
    \State $\hat{\mathbb A}_t^\tau \gets \delta_t + \gamma\lambda\,
    \frac{1-\lambda^{\tau-t-1}}{1-\lambda^{\tau-t}}\,
    \hat{\mathbb A}_{t+1}^\tau$
\EndFor
\State \Return $\{\hat{\mathbb A}_t^\tau\}_{t=0}^{\tau-1}$
\end{algorithmic}
\end{algorithm}

As for the truncated GAE our finite-time GAE also satisfies a simple backwards recursion:\begin{figure*}[t]
    \includegraphics[width=\linewidth]{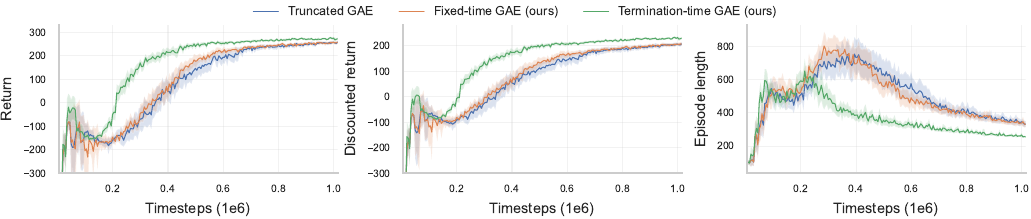}
    \caption{
    LunarLander-v3 learning curves with default PPO hyperparameters from Stable-Baselines3 Zoo. Left (middle): mean evaluation (discounted) return, i.e. sum of (discounted) rewards per episode. Right: mean evaluation episode length. Curves are averaged over 20 seeds with standard errors of seeds as the shaded regions. Both methods use identical default hyperparameters; only the advantage estimator differs (truncated GAE vs. our finite-time GAEs). Additional metrics and ablations are reported in Appendix \ref{sec: finite-time-gae}}
    \label{fig:lunar}
\end{figure*}
\begin{proposition}
 Using the terminal condition  $\hat \A_\tau^{\tau}:=0$, the finite-time estimator satisfies the recursion
\begin{align*}
\hat \A_t^{\tau}=\delta_t+\gamma\lambda\,
\frac{1-\lambda^{\tau-t-1}}{1-\lambda^{\tau-t}}\,
\hat \A_{t+1}^{\tau}, \quad t = \tau -1, \dots , 0.
\end{align*}
\end{proposition}
To highlight the simple adaptation to truncated GAE we provide pseudocode in Algorithm \ref{alg:gae}. Further implementation details can be found in Appendix \ref{sec: finite-time-gae}. 

In Appendix \ref{sec: toy gae} we perform a simplified toy example computation to understand the variance effect of tail-mass collapse reweighting. 
It turns out that near the episode end covariances withing GAE reduce, see Figure \ref{fig:heatmaps} so that our finite-time GAE estimator should be beneficial in environments that crucially rely on the end of episodes, e.g. Lunar Lander, where actions shortly before landing are crucial.
\begin{figure}[t]
    \centering
     \includegraphics[]
    {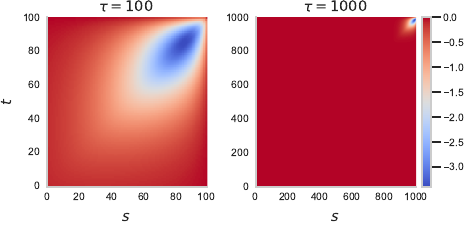}
    \vspace{-5mm}
    \caption{Covariance heatmaps differences for truncated GAE (PPO) vs.\ finite-time GAE (us) at fixed $\gamma = 0.999$, $\lambda=0.95$ and $\tau=200$ (left) and $\tau=1000$ (right) under simplified assumptions (see Appendix \ref{sec: finite-time-gae}). Color scale shows 
    $\Cov[\hat\A_t^\tau, \hat\A_s^\tau]-\Cov[\hat\A_t, \hat\A_s]$.}
    \label{fig:heatmaps}
\end{figure}

\textbf{Experiment:}
We evaluate this effect on LunarLander-v3, using the Stable-Baselines3 PPO implementation \cite{stablebaslines3} and modifying only the advantage estimation (as in Algorithm~\ref{alg:gae}).
In Figure \ref{fig:lunar} we report out-of-the-box results under a standard hyperparameter setting, comparing truncated GAE (blue), our fixed-time variant with $\tau=T=1000$ (green), and our termination-time variant with $\tau=\min\{T,\text{termination}\}$ (orange).
The learning curves show that the termination-time estimator learns substantially faster. It reaches high returns earlier and achieves shorter episode lengths (faster landing). 
A plausible explanation is that the termination-aware variant reduces the variance of the estimates precisely in the high-impact regime where $\tau-t$ is small, yielding more stable policy updates and faster learning.
In contrast, the fixed-horizon GAE performs similarly to truncated GAE, which is consistent with the theory. 
Appendix \ref{sec: finite-time-gae} provides robustness checks, including experiments with hyperparameter optimized separately per estimator.
As sanity check we ran a small experiment on Ant in Appendix \ref{sec: continuous control}; finite-time GAE also performs very well.
\section{Conclusion and Future Work}
This article contributes to the theory gap of PPO, under usual policy gradient assumptions. All appearing constants are huge and should be seen as giving structural understanding rather than direct practical insight (as always in policy gradient theory). We provided a bias analysis (Theorem \ref{Thm:1}) from which convergence statements can be derived, in the exact gradient setting (Theorem \ref{thm:convdeterministc}) and in the original PPO setting with RR (Theorem \ref{Thm:3}). The estimates shed light on the fact that additional biased PPO updates can improve the learning. PPO compensates small (safer) step-sizes by additional (free) biased gradient steps. While this is theoretical, we also identify a tail-mass collapse of truncated GAE used in practice. It is appealing that a tiny change in the GAE significantly improves e.g. Lunar Lander training (and Ant). Given the hardness of the problem and the length of our technical arguments we leave further steps to future work.

There is a lot of current interest in rigorous convergence for policy gradient algorithms. The biased policy gradient interpretation of PPO opens the door to optimization theory, but also a clean view on how to apply stochastic arguments, for instance, from random reshuffling. We believe that our paper might initiate interesting future work. (i) Regularization is a particularly active field. Since our interpretation of PPO is close to policy gradient theory, it sounds plausible that KL-regularization can be added to the analysis. (ii) Due to the increased interest in PPO variants without critic network, it would be interesting to see how our analysis applies to variants of GRPO. (iii) What kind of asymmetric clipping can be analysed formally, can one understand formal differences? (iv) We believe that our analysis from Theorem \ref{Thm:3} can be improved. First, by proving variance reduction effects of multi-rollout flattening and secondly, using less comparison in the RR analysis with the cycle start. 

For  finite-time GAE, next steps will contain a comprehensive experimental study to understand when our finite-time GAE performs better or worse than truncated GAE. On the theory side it would be interesting to see if in toy examples  one could quantify the bias-variance trade-offs in GAE.

\bibliographystyle{plain} 
\bibliography{bib} 

@INPROCEEDINGS{MuJoCo,
  author={Todorov, Emanuel and Erez, Tom and Tassa, Yuval},
  booktitle={2012 IEEE/RSJ International Conference on Intelligent Robots and Systems}, 
  title={MuJoCo: A physics engine for model-based control}, 
  year={2012},
  volume={},
  number={},
  pages={5026-5033},
  doi={10.1109/IROS.2012.6386109}}

@inproceedings{RR,
  title={How Good is SGD with Random Shuffling?},
  author={Itay Safran and Ohad Shamir},
  booktitle={Annual Conference Computational Learning Theory},
  year={2019},
  url={https://api.semanticscholar.org/CorpusID:199064321}
}

@InProceedings{Yuan,
  title = 	 { A general sample complexity analysis of vanilla policy gradient },
  author =       {Yuan, Rui and Gower, Robert M. and Lazaric, Alessandro},
  booktitle = 	 {Proceedings of The 25th International Conference on Artificial Intelligence and Statistics},
  pages = 	 {3332--3380},
  year = 	 {2022},
  editor = 	 {Camps-Valls, Gustau and Ruiz, Francisco J. R. and Valera, Isabel},
  volume = 	 {151},
  series = 	 {Proceedings of Machine Learning Research},
  month = 	 {28--30 Mar},
  publisher =    {PMLR},
  url = 	 {https://proceedings.mlr.press/v151/yuan22a.html},
}

@misc{bytedanceseed2025truncated,
  title         = {Truncated Proximal Policy Optimization},
  author        = {{ByteDance Seed}},
  year          = {2025},
  eprint        = {2506.15050},
  archivePrefix = {arXiv},
  primaryClass  = {cs.LG},
  note          = {arXiv:2506.15050}
}

@inproceedings{PPO2,
 author = {Jin, Ruinan and Li, Shuai and Wang, Baoxiang},
 booktitle = {International Conference on Learning Representations},
 editor = {B. Kim and Y. Yue and S. Chaudhuri and K. Fragkiadaki and M. Khan and Y. Sun},
 pages = {11594--11611},
 title = {On Stationary Point Convergence of PPO-Clip},
 volume = {2024},
 year = {2024}
}

@misc{PPO1,
  title         = {Non-Asymptotic Global Convergence of PPO-Clip},
  author        = {Liu, Yin and Dai, Qiming and Zhang, Junyu and Wen, Zaiwen},
  year          = {2025},
  eprint        = {2512.16565},
  archivePrefix = {arXiv},
  primaryClass  = {cs.LG},
  note          = {arXiv:2512.16565}
}

@article{sara1,
  title={Beyond Stationarity: Convergence Analysis of Stochastic Softmax Policy Gradient Methods},
  author={Sara Klein and Simon Weissmann and Leif D{\"o}ring},
  journal={ICLR},
  year={2024},
}

@article{Zhang,
author = {Zhang, Kaiqing and Koppel, Alec and Zhu, Hao and Ba\c{s}ar, Tamer},
title = {Global Convergence of Policy Gradient Methods to (Almost) Locally Optimal Policies},
journal = {SIAM Journal on Control and Optimization},
volume = {58},
number = {6},
pages = {3586-3612},
year = {2020},
doi = {10.1137/19M1288012},

URL = { 
    
        https://doi.org/10.1137/19M1288012
    
    

},
eprint = { 
    
        https://doi.org/10.1137/19M1288012
    
    

}
}

@InProceedings{mei2022global,
  title = 	 {On the Global Convergence Rates of Softmax Policy Gradient Methods},
  author =       {Mei, Jincheng and Xiao, Chenjun and Szepesvari, Csaba and Schuurmans, Dale},
  booktitle = 	 {Proceedings of the 37th International Conference on Machine Learning},
  pages = 	 {6820--6829},
  year = 	 {2020},
  volume = 	 {119},
  series = 	 {Proceedings of Machine Learning Research},
  month = 	 {13--18 Jul},
  publisher =    {PMLR},
  url = 	 {https://proceedings.mlr.press/v119/mei20b.html}
}

@misc{wang2025improvingvalueestimationcritically,
      title={Improving Value Estimation Critically Enhances Vanilla Policy Gradient}, 
      author={Tao Wang and Ruipeng Zhang and Sicun Gao},
      year={2025},
      eprint={2505.19247},
      archivePrefix={arXiv},
      primaryClass={cs.LG},
      url={https://arxiv.org/abs/2505.19247}, 
}

@misc{zhang2022deeperlookdiscountingmismatch,
      title={A Deeper Look at Discounting Mismatch in Actor-Critic Algorithms}, 
      author={Shangtong Zhang and Romain Laroche and Harm van Seijen and Shimon Whiteson and Remi Tachet des Combes},
      year={2022},
      eprint={2010.01069},
      archivePrefix={arXiv},
      primaryClass={cs.LG},
      url={https://arxiv.org/abs/2010.01069}, 
}

@inproceedings{che2023correctingdiscountfactormismatchonpolicy,
author = {Che, Fengdi and Vasan, Gautham and Mahmood, A. Rupam},
title = {Correcting discount-factor mismatch in on-policy policy gradient methods},
year = {2023},
articleno = {168},
numpages = {23},
location = {Honolulu, Hawaii, USA},
series = {ICML'23}
}

@article{sutton2018reinforcement,
  title={Reinforcement learning: An introduction},
  author={Sutton, Richard S and Barto, Andrew G},
  journal={MIT press},
  year={2018}
}

@article{silver2016mastering,
  title={Mastering the game of Go with deep neural networks and tree search},
  author={Silver, David and Huang, Aja and Maddison, Chris J and Guez, Arthur and Sifre, Laurent and Van Den Driessche, George and Schrittwieser, Julian and Antonoglou, Ioannis and Panneershelvam, Veda and Lanctot, Marc and others},
  journal={nature},
  volume={529},
  number={7587},
  pages={484--489},
  year={2016},
  publisher={Nature Publishing Group}
}

@misc{berner2019dota,
      title={Dota 2 with Large Scale Deep Reinforcement Learning}, 
      author={OpenAI and : and Christopher Berner and Greg Brockman and Brooke Chan and Vicki Cheung and Przemysław Dębiak and Christy Dennison and David Farhi and Quirin Fischer and Shariq Hashme and Chris Hesse and Rafal Józefowicz and Scott Gray and Catherine Olsson and Jakub Pachocki and Michael Petrov and Henrique P. d. O. Pinto and Jonathan Raiman and Tim Salimans and Jeremy Schlatter and Jonas Schneider and Szymon Sidor and Ilya Sutskever and Jie Tang and Filip Wolski and Susan Zhang},
      year={2019},
      eprint={1912.06680},
      archivePrefix={arXiv},
      primaryClass={cs.LG},
      url={https://arxiv.org/abs/1912.06680}, 
}

@article{levine2016end,
  title={End-to-end training of deep visuomotor policies},
  author={Levine, Sergey and Finn, Chelsea and Darrell, Trevor and Abbeel, Pieter},
  journal={The Journal of Machine Learning Research},
  volume={17},
  number={1},
  pages={1334--1373},
  year={2016},
  publisher={JMLR. org}
}

@article{openai2019solving,
  title={Solving rubik's cube with a robot hand},
  author={OpenAI and Akkaya, Ilge and Andrychowicz, Marcin and Chociej, Maciek and Litwin, Mateusz and McGrew, Bob and Petron, Alex and Paino, Alex and Plappert, Matthias and Powell, Glenn and others},
  journal={arXiv preprint arXiv:1910.07113},
  year={2019}
}

@article{ouyang2022training,
  title={Training language models to follow instructions with human feedback},
  author={Ouyang, Long and Wu, Jeffrey and Jiang, Xu and Almeida, Diogo and Wainwright, Carroll and Mishkin, Pamela and Zhang, Chong and Agarwal, Sandhini and Slama, Katarina and Ray, Alex and others},
  journal={Advances in Neural Information Processing Systems},
  volume={35},
  pages={27730--27744},
  year={2022}
}

@inproceedings{christiano2017deep,
  title={Deep reinforcement learning from human preferences},
  author={Christiano, Paul F and Leike, Jan and Brown, Tom and Martic, Miljan and Legg, Shane and Amodei, Dario},
  booktitle={Advances in neural information processing systems},
  pages={4299--4307},
  year={2017}
}

@article{PPO,
  title={Proximal policy optimization algorithms},
  author={Schulman, John and Wolski, Filip and Dhariwal, Prafulla and Radford, Alec and Klimov, Oleg},
  journal={arXiv preprint arXiv:1707.06347},
  year={2017}
}

@inproceedings{sutton1999policy,
  title={Policy gradient methods for reinforcement learning with function approximation},
  author={Sutton, Richard S and McAllester, David and Singh, Satinder and Mansour, Yishay},
  booktitle={Advances in neural information processing systems},
  pages={1057--1063},
  year={1999}
}

@article{williams1992simple,
  title={Simple statistical gradient-following algorithms for connectionist reinforcement learning},
  author={Williams, Ronald J},
  journal={Machine learning},
  volume={8},
  number={3},
  pages={229--256},
  year={1992},
  publisher={Springer}
}

@inproceedings{A2C,
  title={Asynchronous methods for deep reinforcement learning},
  author={Mnih, Volodymyr and Badia, Adria Puigdomenech and Mirza, Mehdi and Graves, Alex and Lillicrap, Timothy and Harley, Tim and Silver, David and Kavukcuoglu, Koray},
  booktitle={International conference on machine learning},
  pages={1928--1937},
  year={2016},
  organization={PMLR}
}

@inproceedings{engstrom2020implementation,
  title={Implementation matters in deep policy gradients: A case study on ppo and trpo},
  author={Engstrom, Logan and Ilyas, Andrew and Santurkar, Shibani and Tsipras, Dimitris and Janoos, Firdaus and Rudolph, Larry and Madry, Aleksander},
  booktitle={International Conference on Learning Representations},
  year={2020}
}

@article{Mishenko,
  title={Random reshuffling: Simple analysis with vast improvements},
  author={Mishchenko, Konstantin and Khaled, Ahmed and Richt{\'a}rik, Peter},
  journal={Advances in Neural Information Processing Systems},
  volume={33},
  pages={17309--17320},
  year={2020}
}

@misc{formalismimplementationgap,
      title={The Formalism-Implementation Gap in Reinforcement Learning Research}, 
      author={Pablo Samuel Castro},
      year={2025},
      eprint={2510.16175},
      archivePrefix={arXiv},
      primaryClass={cs.LG},
      url={https://arxiv.org/abs/2510.16175}, 
}

@misc{Song,
      title={Partial advantage estimator for proximal policy optimization}, 
      author={Xiulei Song and Yizhao Jin and Greg Slabaugh and Simon Lucas},
      year={2023},
      eprint={2301.10920},
      archivePrefix={arXiv},
      primaryClass={cs.LG},
      url={https://arxiv.org/abs/2301.10920}, 
}

@misc{GAE,
      title={High-Dimensional Continuous Control Using Generalized Advantage Estimation}, 
      author={John Schulman and Philipp Moritz and Sergey Levine and Michael Jordan and Pieter Abbeel},
      year={2018},
      eprint={1506.02438},
      archivePrefix={arXiv},
      primaryClass={cs.LG},
      url={https://arxiv.org/abs/1506.02438}, 
}

@article{Papini,
author = {Papini, Matteo and Pirotta, Matteo and Restelli, Marcello},
title = {Smoothing policies and safe policy gradients},
year = {2022},
issue_date = {Nov 2022},
publisher = {Kluwer Academic Publishers},
address = {USA},
volume = {111},
number = {11},
issn = {0885-6125},
url = {https://doi.org/10.1007/s10994-022-06232-6},
doi = {10.1007/s10994-022-06232-6},
journal = {Mach. Learn.},
month = nov,
pages = {4081–4137},
numpages = {57}
}

@inproceedings{samu,
  title     = {REINFORCE Converges to Optimal Policies with Any Learning Rate},
  author    = {Robertson, Samuel and Chu, Thang and Dai, Bo and Schuurmans, Dale and Szepesvari, Csaba and Mei, Jincheng},
  booktitle = {Advances in Neural Information Processing Systems (NeurIPS)},
  year      = {2025}
}

@article{Kumar,
author = {Kumar, Harshat and Koppel, Alec and Ribeiro, Alejandro},
title = {On the sample complexity of actor-critic method for reinforcement learning with function approximation},
year = {2023},
issue_date = {Jul 2023},
address = {USA},
volume = {112},
number = {7},
issn = {0885-6125},
journal = {Mach. Learn.},
month = feb,
pages = {2433–2467},
numpages = {35}
}

@article{agarwal,
  author  = {Alekh Agarwal and Sham M. Kakade and Jason D. Lee and Gaurav Mahajan},
  title   = {On the Theory of Policy Gradient Methods: Optimality, Approximation, and Distribution Shift},
  journal = {Journal of Machine Learning Research},
  year    = {2021},
  volume  = {22},
  number  = {98},
  pages   = {1--76},
  url     = {http://jmlr.org/papers/v22/19-736.html}
}

@article{stablebaslines3,
author = {Raffin, Antonin and Hill, Ashley and Gleave, Adam and Kanervisto, Anssi and Ernestus, Maximilian and Dormann, Noah},
title = {Stable-baselines3: reliable reinforcement learning implementations},
year = {2021},
issue_date = {January 2021},
publisher = {JMLR.org},
volume = {22},
number = {1},
issn = {1532-4435},
journal = {J. Mach. Learn. Res.},
month = jan,
articleno = {268},
numpages = {8}
}

@inproceedings{Thomas,
author = {Thomas, Philip S.},
title = {Bias in natural actor-critic algorithms},
year = {2014},
publisher = {JMLR.org},
booktitle = {Proceedings of the 31st International Conference on International Conference on Machine Learning - Volume 32},
pages = {I–441–I–448},
location = {Beijing, China},
series = {ICML'14}
}

@InProceedings{fisher,
  title = 	 {Stochastic Policy Gradient Methods: Improved Sample Complexity for {F}isher-non-degenerate Policies},
  author =       {Fatkhullin, Ilyas and Barakat, Anas and Kireeva, Anastasia and He, Niao},
  booktitle = 	 {Proceedings of the 40th International Conference on Machine Learning},
  pages = 	 {9827--9869},
  year = 	 {2023},
  volume = 	 {202},
  series = 	 {Proceedings of Machine Learning Research},
  month = 	 {23--29 Jul},
  publisher =    {PMLR},
}

@misc{tapered,
      title={Tapered Off-Policy REINFORCE: Stable and efficient reinforcement learning for LLMs}, 
      author={Nicolas Le Roux and Marc G. Bellemare and Jonathan Lebensold and Arnaud Bergeron and Joshua Greaves and Alex Fréchette and Carolyne Pelletier and Eric Thibodeau-Laufer and Sándor Toth and Sam Work},
      year={2025},
      eprint={2503.14286},
      archivePrefix={arXiv},
      primaryClass={cs.LG},
      url={https://arxiv.org/abs/2503.14286}, 
}

@misc{A2CisPPO,
      title={A2C is a special case of PPO}, 
      author={Shengyi Huang and Anssi Kanervisto and Antonin Raffin and Weixun Wang and Santiago Ontañón and Rousslan Fernand Julien Dossa},
      year={2022},
      eprint={2205.09123},
      archivePrefix={arXiv},
      primaryClass={cs.LG},
      url={https://arxiv.org/abs/2205.09123} 
}

@inproceedings{liu,
 author = {Liu, Boyi and Cai, Qi and Yang, Zhuoran and Wang, Zhaoran},
 booktitle = {Advances in Neural Information Processing Systems},
 editor = {H. Wallach and H. Larochelle and A. Beygelzimer and F. d\textquotesingle Alch\'{e}-Buc and E. Fox and R. Garnett},
 pages = {},
 publisher = {Curran Associates, Inc.},
 title = {Neural Trust Region/Proximal Policy Optimization Attains Globally Optimal Policy},
 url = {https://proceedings.neurips.cc/paper_files/paper/2019/file/227e072d131ba77451d8f27ab9afdfb7-Paper.pdf},
 volume = {32},
 year = {2019}
}

@book{mcmixtimes,
  series   = {ProQuest Ebook Central},
  isbn     = {9781470442323},
  year     = {2017},
  title    = {Markov chains and mixing times},
  edition  = {Second edition},
  language = {eng},
  address  = {Providence, Rhode Island},
  author   = {Levin, David Asher},
}

@inproceedings{TRPO,
author = {Schulman, John and Levine, Sergey and Moritz, Philipp and Jordan, Michael and Abbeel, Pieter},
title = {Trust region policy optimization},
year = {2015},
publisher = {JMLR.org},
booktitle = {Proceedings of the 32nd International Conference on International Conference on Machine Learning - Volume 37},
pages = {1889–1897},
numpages = {9},
location = {Lille, France},
series = {ICML'15}
}

@inproceedings{aoarl,
  author    = {Kakade, Sham and Langford, John},
  title     = {Approximately Optimal Approximate Reinforcement Learning},
  year      = {2002},
  isbn      = {1558608737},
  publisher = {Morgan Kaufmann Publishers Inc.},
  address   = {San Francisco, CA, USA},
  booktitle = {Proceedings of the Nineteenth International Conference on Machine Learning},
  pages     = {267–274},
  numpages  = {8},
  series    = {ICML '02}
}

\newpage
\appendix
\onecolumn


\section{Notation and preliminary results}
Let us fix some notation. We will denote by $|\cdot|$ the Euclidean norm, by $|\cdot|_\infty$ the maximum norm on $\R^d$, and by $\|\cdot\|_\infty$ the maximum norm over the state and/or action space. The gradient $\nabla_\theta \pi_\theta(s,a)$ refers to the derivative with respect to the policy parameter. We sometimes drop the identifier $\theta$ from the gradient to avoid confusion. Recall that we consider discounted finite-horizon MDPs, with value functions
\begin{align*}
    V^\pi(s)\coloneq\E^\pi_s\Big[\sum_{t=0}^{T-1}\gamma^t R_t \mid S_0=s\Big]
\end{align*} 
and
\begin{align*}
 V_t^\pi(s)\coloneq\E^\pi\Big[\sum_{i=t}^{T-1}\gamma^{i-t} R_i\mid S_t=s\Big]\quad\text{and}\quad
Q_t^\pi(s,a)\coloneq\E^\pi\Big[\sum_{i=t}^{T-1}\gamma^{i-t} R_i \mid S_t=s,A_t=a\Big].
\end{align*}
For completeness, we also set $V_T^\pi\equiv 0$ and $Q_T^\pi\equiv 0$. 
We also define the advantage $\mathbb A_t^\pi(s,a)\coloneq Q_t^\pi(s,a)-V_t^\pi(s)$ and denote the marginal state distribution by $d_t^{\pi,s'}$, i.e., 
\begin{align*}
d_t^{\pi,s'}(s)=\P^\pi_{s'}(S_t=s).    
\end{align*}
Throughout this section, we treat the advantage function as known, as if we had access to a perfect critic. We will always work with a continuously differentiable parametrized family of policies $\{\pi_\theta\}_{\theta \in \R^d}$, and we abbreviate $J(\theta)= V^{\pi_\theta}(\mu):=\sum_s V^{\pi_\theta}(s)\mu(s)$ and $d_t^\pi\coloneq d_t^{\pi,\mu}$ for some fixed initial state distribution $\mu$. We will always assume that
\begin{equation}\label{eq:full-support}
\pi_{\theta}(a\,;\, s) > 0 \quad \text{for all } \theta \in \R^d, (s,a)\in\mathcal S\times\mathcal A,
\end{equation}
to ensure that likelihood ratios are well-defined. This is typically fulfilled by a final application of a softmax normalization.

To prove convergence of PPO we will have to assume properties on the underlying Markov decision model (bounded rewards) and policy (bounded and Lipschitz continuous score function), which will imply $L$-smoothness of the parametrized value function, see Proposition~\ref{lem:smooth}. These assumptions are standard in the convergence analysis of policy gradient methods.
\begin{assumption}[Bounded rewards]\label{ass:boundedrewards}
    The rewards are uniformly bounded in absolute value by $R_\ast$.
\end{assumption}
Note that under Assumption \ref{ass:boundedrewards} the value function, the $Q$-function, and the advantage are also bounded. Most relevant for us, one has $||\mathbb A_t^\pi||_\infty\leq \frac{1-\gamma ^T}{1-\gamma} 2 R_\ast$ for all $t\leq T-1$. We use this bound in the deterministic setting. 
In the stochastic analysis we assume access to biased and bounded advantage estimates.
\begin{assumption}[Biased and bounded advantage estimates]\label{ass:unbiased-bounded-adv}
There exists constants $A_\ast<\infty$ and $\delta\ge0$ such that for any $\theta$ and every $t\in\{0,\dots,T-1\}$ we have access to
an advantage estimate $\hat{\A}_t$ satisfying
\[
\E^{\pi_\theta}[|\mathbb{E}^{\pi_{\theta}}\!\big[\hat{\A}_t \mid S_t, A_t\big]
- \A^{\pi_{\theta}}_t(S_t,A_t)|^2]\le \delta^2
\qquad\text{and}\qquad 
\big|\hat{\A}_t\big| \le A_\ast
\quad \text{a.s.}
\]
\end{assumption}
Assuming access to a theoretical critical trivially satisfies an unbiased and bounded advantage estimator assumption.

\begin{assumption}[Bounded score function]\label{ass:score}
    The score function is bounded, i.e. $$\Pi_\ast:= \sup_\theta ||\nabla_\theta \log(\pi_\theta)||_\infty<\infty.$$
\end{assumption}
Note that a bounded score function implies bounded gradients, since
\begin{align*}
    |\nabla_\theta \pi_\theta(a \,;\, s)|
    =\pi_\theta(a \,;\, s)|\nabla_\theta \log \pi_\theta(a \,;\, s)|\leq \Pi_\ast,
\end{align*}
and, using the mean-value theorem, Lipschitz continuity of the policies.
\begin{assumption}[Lipschitz score function]\label{ass:Lipschitz}
There exists $L>0$ such that for all $(s,a) \in \mathcal S \times \mathcal A$ and all $\theta,\theta' \in \R^d$,
\begin{align*}
|\nabla\log\pi_\theta(a\,;\, s)
      - \nabla\log\pi_{\theta'}(a\,;\, s)|
\le L_{\text{s}} |\theta-\theta'|.    
\end{align*}
\end{assumption}
We refer, for instance, to page 7 of \cite{Zhang} for a discussion of example policies that satisfy these assumptions.

In what follows, we denote by $\mathrm{TV}$ the total variation distance
\begin{align} \label{eq:TVdefinition}
    \mathrm{TV}(P,Q)=\frac{1}{2}\sum_{a\in \mathcal A}\big|P(a)-Q(a)\big|=\sup_{B\subseteq \mathcal A} \big|P(B)-Q(B)\big|=\frac{1}{2}\sup_{||f||_\infty\leq 1}\Big|\int_\mathcal A f dP-\int_\mathcal A f dQ\Big|
\end{align}
for probability measures $P$ and $Q$ on the finite action space $\mathcal A$.
\begin{lemma}\label{lem:TVPOLICY}
Under Assumption \ref{ass:score}, one has
\begin{align}
\mathrm{TV}\!\big(\pi_\theta(\,\cdot\,;\,s),\pi_{\theta'}(\,\cdot\,;\,s)\big)
&\le \tfrac{1}{2}\Pi_\ast\,|\theta-\theta'|, \quad \text{ for all } \theta, \theta' \in \R^d \text{ and } s \in \mathcal S. 
\end{align}
\end{lemma}
\begin{proof}
  Assumption~\ref{ass:score} together with $\nabla_\theta \pi_\theta(a \,;\, s)=\pi_\theta(s \,;\, a)\nabla_\theta \log \pi_\theta(s \,;\, a)$ implies for all $s \in \mathcal S$ that
\begin{align} \begin{split} \label{eq:lipschitzpolicy}
\mathrm{TV}\!\left(\pi_\theta(\,\cdot\,;\,s),\pi_{\theta'}(\,\cdot\,;\,s)\right)
&=\frac{1}{2}
\sum_a \bigl|\pi_\theta(a \,;\, s) - \pi_{\theta'}(a \,;\, s)\bigr|\\
&\le\frac{1}{2} \sum_a \int_0^1 
| \nabla \pi_{\varphi(t)}(a \,;\, s)^\top (\theta-\theta') | \, dt \\
&\le\frac{1}{2} \int_0^1 \underbrace{\sum_a 
\pi_{\varphi(t)}(a \,;\, s)}_{=1} \Pi_\ast |\theta-\theta'| \, dt= \frac{1}{2}\Pi_\ast |\theta - \theta'|,
\end{split}
\end{align}
where $\varphi(t)= (1-t) \theta+t\theta'$.
\end{proof}

\section{Properties of the parametrized value functions}
In this section, we collect basic properties of the value function, most importantly the $L$-smoothness. Although smoothness of the parametrized value function is well known in the literature, existing proofs typically rely on  slightly stronger assumptions or are given for either infinite-time horizon or finite-time non-discounted MDPs; see, for example, \cite{agarwal, Yuan, Papini, fisher}. For the reader's convenience, we provide self-contained proofs that differ from those in the cited works. The technique developed here will also be used below to prove the estimates for the gradient bias of PPO. 

First, we recall the standard policy gradient theorem for discounted finite-time MDPs:
\begin{proposition}\label{prop:pgt}
Under Assumptions \ref{ass:boundedrewards} and \ref{ass:score} the gradient of the value function with respect to the policy parameter is given by
    \begin{align} \label{eq:policygradient1}
        \nabla_\theta J(\theta)=\sum_{t=0}^{T-1} \gamma^t \E^{\pi_\theta}\Big[\nabla_\theta \log\big(\pi_\theta(A_t\,;\,S_t)\big) \mathbb A_{t}^{\pi_\theta}(S_t,A_t)\Big].
    \end{align}
\end{proposition}

\begin{lemma}[Lipschitz continuity of $J$]
\label{lem:value}
Under Assumptions \ref{ass:boundedrewards} and \ref{ass:score}, one has for all $t \in \{0,\ldots,T\}$,
$s \in \mathcal{S}$, and $\theta,\theta' \in \mathbb R^d$,
\begin{align} \label{eq:5634956926}
\big|V_t^{\pi_\theta}(s) - V_t^{\pi_{\theta'}}(s)\big|
\leq 
\frac{1-(T-t+1)\gamma^{T-t}+(T-t)\gamma^{T-t+1}}{(1-\gamma)^2}\,\Pi_\ast R_\ast\, |\theta - \theta'|.
\end{align}
In particular, 
\[
\big|J(\theta) - J(\theta')\big|
\le
\frac{1-(T+1)\gamma^T+T\gamma^{T+1}}{(1-\gamma)^2}\,\Pi_\ast R_\ast\,
|\theta - \theta'|.
\]
\end{lemma}

\begin{proof}
The proof proceeds by backward induction on $t$.
For $t=T$, the claim holds as $V_T \equiv 0$.

Assume that the bound holds at time $t+1$. To use the induction hypothesis, we apply the (finite-time) Bellman recursion
\[
V_t^{\pi_\theta}(s) - V_t^{\pi_{\theta'}}(s)
=
(T_{\pi_\theta} V_{t+1}^{\pi_\theta})(s)
-
(T_{\pi_{\theta'}} V_{t+1}^{\pi_{\theta'}})(s),
\]
where $(T_{\pi} V)(s) = \E_{a \sim \pi(\,\cdot \,;\, s), s' \sim p(\,\cdot \,;\, s,a)}[r(s,a) + \gamma  V(s')]$. We now decompose the difference
\[
|V_t^{\pi_\theta}(s) - V_t^{\pi_{\theta'}}(s)|
\le
\underbrace{|T_{\pi_\theta} (V_{t+1}^{\pi_\theta} - V_{t+1}^{\pi_{\theta'}})(s)|}_{\text{(A)}}
+
\underbrace{|(T_{\pi_\theta} - T_{\pi_{\theta'}}) V_{t+1}^{\pi_{\theta'}}(s)|}_{\text{(B)}}.
\]
Since $T_{\pi_\theta}$ is a max-norm contraction, we get
\[
\text{(A)} \le \gamma \|V_{t+1}^{\pi_\theta} - V_{t+1}^{\pi_{\theta'}}\|_\infty.
\]
To deal with term (B), we use the Bellman operator explicitly. By definition
\[
(T_{\pi} V)(s)
=
\sum_{a \in \mathcal A}
\pi(a \,;\, s)
\big(
\E[r(s,a)] + \gamma \E_{s' \sim p(\,\cdot \,;\, s,a)}[V(s')]
\big).
\]
Therefore,
\[
\begin{aligned}
(T_{\pi_\theta} - T_{\pi_{\theta'}}) V_{t+1}^{\pi_{\theta'}}(s)
\sum_a
\bigl(
\pi_\theta(a \,;\, s) - \pi_{\theta'}(a \,;\, s)
\bigr)
\big(
\E[r(s,a)] + \gamma\E_{s' \sim p(\,\cdot \,;\, s,a)}[V_{t+1}^{\pi_{\theta'}}(s')]
\big).
\end{aligned}
\]
Since the rewards are bounded, one gets for all $a$
$$
    \big|\E[r(s,a)] + \gamma\E_{s' \sim p(\,\cdot \,;\, s,a)}[V_{t+1}^{\pi_{\theta'}}(s')]\big| \le R_\ast \sum_{j=0}^{T-t-1} \gamma^j.
$$
Recalling from \eqref{eq:lipschitzpolicy} that $\sum_a \bigl|\pi_\theta(a \,;\, s) - \pi_{\theta'}(a \,;\, s)\bigr|\leq \Pi_\ast|\theta-\theta'|$, then gives
\[
\text{(B)} \le \Big(
\sum_a
|
\pi_\theta(a \,;\, s) - \pi_{\theta'}(a \,;\, s)
|
\Big) \Big( R_\ast \sum_{j=0}^{T-t-1} \gamma^j \Bigr) \le \Pi_\ast |\theta - \theta'| \Big( R_\ast \sum_{j=0}^{T-t-1} \gamma^j \Big).
\]
Combining these into the recurrence relation yields the statement for $t$, where the identity \eqref{eq:5634956926} follows from a straight-forward calculation using the formula for finite geometric series. Plugging $t=0$ into the formula then gives the result for $\lvert J(\theta)-J(\theta')\rvert$.
\end{proof}

Since MDPs are stochastic processes defined on the state–action product space, it is natural to ask for decompositions of associated quantities into components that depend solely on the state and components that depend on the action conditioned on the state. For the total variation distance, such a decomposition can be carried out as follows.

\begin{lemma}[Marginal decomposition of total variation]\label{lemma:marginal}
Let $\nu,\nu'$ be probability measures on $\mathcal S\times\mathcal A$ of the form
\[
\nu(s,a)=d(s)\pi(a\,;\, s),
\qquad
\nu'(s,a)=d'(s)\pi'(a\,;\, s).
\]
Then
\[
\mathrm{TV}(\nu,\nu')
\;\le\;
\mathrm{TV}(d,d')
+
\sup_{s\in\mathcal S}
\mathrm{TV}\!\left(\pi(\,\cdot\,;\, s),\pi'(\,\cdot\,;\, s)\right).
\]
\end{lemma}

\begin{proof}
By definition of the total variation distance between measures on $\mathcal S\times\mathcal A$,
\[
\mathrm{TV}(\nu,\nu')
=
\sup_{B\subset \mathcal S\times\mathcal A}
\Big|\sum_{(s,a)\in B}
\big(d(s)\pi(a\,;\, s)-d'(s)\pi'(a\,;\, s)\big)
\Big|.
\]

For a given set $B\subset\mathcal S\times\mathcal A$, define $B_s:=\{a\in\mathcal A:(s,a)\in B\}$.
Then
\[
\sum_{(s,a)\in B} d(s)\pi(a\,;\, s)
=
\sum_{s\in\mathcal S} d(s)\sum_{a\in B_s}\pi(a\,;\, s),
\]
and with a similar decomposition existing for $\nu'$.

Add and subtract the mixed term
$\sum_{s\in\mathcal{S}} d'(s)\sum_{a\in B_s}\pi(a\,;\, s)$ to obtain
\begin{align*}
&\sum_{s\in\mathcal{S}} d(s)\sum_{a\in B_s}\pi(a\,;\, s)
-
\sum_{s\in\mathcal{S}} d'(s)\sum_{a\in B_s}\pi'(a\,;\, s)\\ 
&=
\sum_{s\in\mathcal{S}} \bigl[d(s)-d'(s)\bigr]\sum_{a\in B_s}\pi(a\,;\, s)
+
\sum_{s\in\mathcal{S}} d'(s)\sum_{a\in B_s}\bigl[\pi(a\,;\, s)-\pi'(a\,;\, s)\bigr].
\end{align*}

Taking absolute values and using the triangle inequality yields
\[
\mathrm{TV}(\nu,\nu')
\le \sup_B
\Big|
\sum_{s\in\mathcal{S}} (d(s)-d'(s))\sum_{a\in B_s}\pi(a\,;\, s)
\Big| + \sup_B
\Big|
\sum_{s\in\mathcal{S}} d'(s)\sum_{a\in B_s}\big(\pi(a\,;\, s)-\pi'(a\,;\, s)\big)
\Big|.
\]
For the first summand, note that $0\le\sum_{a\in B_s}\pi(a\,;\, s)\le 1$ for all $s$.
Therefore, an upper bound is $
\sup_{C\subset\mathcal S}
\left|
\sum_{s\in C} \bigl[d(s)-d'(s)\bigr]
\right|
=
\mathrm{TV}(d,d').$ For the second term, we use the upper bound
\begin{align*}
\sum_{s\in\mathcal S} d'(s)
\sup_{C\subset\mathcal A}
\Big|
\sum_{a\in C}\bigl[\pi(a\,;\, s)-\pi'(a\,;\, s)\bigr]
\Big| 
=
\sup_{s\in\mathcal S}
\mathrm{TV}\!\big(\pi(\,\cdot\,;\, s),\pi'(\,\cdot\,;\, s)\big),
\end{align*}
where the last equality follows from the definition of total variation distance
on $\mathcal A$.
\end{proof}

\begin{lemma}\label{lem:marginal-tv-decomposition}
    The TV distance between the marginal state distributions can be decomposed as \[\mathrm{TV}\big(d_t^{\pi'},d_t^\pi\big)\leq \sum_{i=0}^{t-1} \E^\pi\left[\mathrm{TV}\big(\pi'(\,\cdot\,;\,S_i),\pi(\,\cdot\,;\,S_i)\big)\right].\]
\end{lemma}
\begin{proof}
    For all $t\in \{1, \dots, T\}$, we have $d_t^\pi(s) = \sum_{s'\in\mathcal{S}}d_{t-1}^\pi(s')\sum_{a'\in\mathcal{A}}\pi(a'\, ; \,s')p(s\,;\,s',a')$ and thus
    \begin{align*}
        \sum_{s\in\mathcal{S}} \big| d_t^{\pi'}(s) - d_t^\pi(s) \big| 
        &= \sum_{s\in\mathcal{S}} \Big|\sum_{s'\in\mathcal{S}}\Big( d_{t-1}^{\pi'}(s')\sum_{a'\in\mathcal{A}} \pi'(a'\,;\,s')p(s\,;\,s',a') - d_{t-1}^\pi(s')\sum_{a'\in\mathcal{A}}\pi(a'\,;\,s')p(s\,;\,s',a')\Big) \Big| \\
        &=\begin{multlined}[t]
            \sum_{s\in\mathcal{S}} \Big| \sum_{s'\in\mathcal{S}} \big(d_{t-1}^{\pi'}(s') - d_{t-1}^\pi(s')\big) \sum_{a'\in\mathcal{A}}\pi'(a'\,;\,s')p(s\,;\,s',a') \\
            + \sum_{s'\in\mathcal{S}} d_{t-1}^\pi(s') \sum_{a'\in\mathcal{A}}\big(\pi'(a'\,;\,s')-\pi(a'\,;\,s')\big)p(s\,;\,s',a') \Big|
        \end{multlined} \\
        &\leq \begin{multlined}[t]
            \sum_{s\in\mathcal{S}} \Big( \sum_{s'\in\mathcal{S}} \big\lvert d_{t-1}^{\pi'}(s') - d_{t-1}^\pi(s')\big\rvert \sum_{a'\in\mathcal{A}}\pi'(a'\,;\,s')p(s\,;\,s',a') \\
            + \sum_{s'\in\mathcal{S}} d_{t-1}^\pi(s') \sum_{a'\in\mathcal{A}}\big\lvert\pi'(a'\,;\,s')-\pi(a'\,;\,s')\big\rvert p(s\,;\,s',a') \Big)
        \end{multlined} \\
        &= \begin{multlined}[t]
            \sum_{s'\in\mathcal{S}} \big\lvert d_{t-1}^{\pi'}(s') - d_{t-1}^\pi(s')\big\rvert \sum_{a'\in\mathcal{A}}\pi'(a'\,;\,s') \sum_{s\in\mathcal{S}} p(s\,;\,s',a') \\
            + \sum_{s'\in\mathcal{S}} d_{t-1}^\pi(s') \sum_{a'\in\mathcal{A}}\big\lvert\pi'(a'\,;\,s')-\pi(a'\,;\,s')\big\rvert \sum_{s\in\mathcal{S}} p(s\,;\,s',a')
        \end{multlined} \\
        &= \sum_{s'\in\mathcal{S}} \big\lvert d_{t-1}^{\pi'}(s') - d_{t-1}^\pi(s')\big\rvert + \sum_{s'\in\mathcal{S}} d_{t-1}^\pi(s') \sum_{a'\in\mathcal{A}}\big\lvert\pi'(a'\,;\,s')-\pi(a'\,;\,s')\big\rvert,
    \end{align*}
    i.e. $\mathrm{TV}\big(d_t^{\pi'},d_t^\pi\big) \leq \mathrm{TV}\big(d_{t-1}^{\pi'},d_{t-1}^\pi\big) + \E^\pi\left[\mathrm{TV}\big(\pi'(\,\cdot\,;\,S_{t-1}),\pi(\,\cdot\,;\,S_{t-1})\big)\right]$. Using $d_0^{\pi'}=d_0^\pi$ and recursively applying this inequality gives the statement.
\end{proof}

\begin{proposition}[$L$-Smoothness of $J$]
\label{lem:smooth}
Under Assumption \ref{ass:boundedrewards}, \ref{ass:score}, and \ref{ass:Lipschitz}, one has for all $\theta,\theta' \in \R^d$,
\begin{align*}
\big|\nabla J(\theta) - \nabla J(\theta')\big|
&\le
R_\ast
\sum_{t=0}^{T-1}
\gamma^t
\Big(
L \sum_{k=0}^{T-t-1} \gamma^k
+
\Pi_\ast^2 \gamma
\sum_{k=0}^{T-t-2} \gamma^k \Big( \sum_{j=0}^{T-t-2-k} \gamma^j \Big)
+
(t+1)\Pi_\ast^2
\sum_{k=0}^{T-t-1} \gamma^k
\Big)
\big|\theta - \theta'\big|\\
&= \underbrace{R_\ast
\Bigg( \frac{L_\text{s}(1-\gamma^T)}{(1-\gamma)^2}
-\frac{L_\text{s}T\gamma^T}{1-\gamma}
+
\Pi_\ast^2
\Big( \frac{1+\gamma-2\gamma^T}{(1-\gamma)^3}
-\frac{(2T-1)\gamma^T}{(1-\gamma)^2}
-\frac{T^2\gamma^T}{1-\gamma}
\Big) \Bigg)}_{=:L} |\theta-\theta'|.
\end{align*}
\end{proposition}

\begin{proof}
    
We write the policy gradient in the score-function form
\[
\nabla J(\theta)
=
\sum_{t=0}^{T-1}
\gamma^t
\E_{s \sim d_t^{\pi_\theta},\, a \sim \pi_\theta(\,\cdot \,;\, s)}
\bigl[
\nabla_\theta \log \pi_\theta(a \,;\, s)
\, Q_{t}^{\pi_\theta}(s,a)
\bigr].
\]

For $t=0, \dots, T-1$ we write,
\[
\phi_t^\theta(s,a)
:=
\gamma^t
\nabla_\theta \log \pi_\theta(a \,;\, s)
\, Q_{t}^{\pi_\theta}(s,a), 
\]
so that $\nabla J(\theta) = \sum_{t=0}^{T-1} \E_{s \sim d_t^{\pi_\theta},\, a \sim \pi_\theta(\,\cdot \,;\, s)}[\phi_t^\theta(s,a)]$. For a fixed $t \in \{0, \dots, T-1\}$ we decompose

\begin{align*}
&\quad
\E_{s \sim d_t^{\pi_\theta},\, a \sim \pi_\theta(\,\cdot \,;\, s)}[\phi_t^\theta(s,a)]
-
\E_{s \sim d_t^{\pi_{\theta'}},\, a \sim \pi_\theta'(\,\cdot \,;\, s)}[\phi_t^{\theta'}(s,a)]\\
&=
\underbrace{
\E_{s \sim d_t^{\pi_\theta},\, a \sim \pi_\theta(\,\cdot \,;\, s)}[\phi_t^\theta(s,a) - \phi_t^{\theta'}(s,a)]
}_{(A)}+
\underbrace{
\E_{s \sim d_t^{\pi_\theta},\, a \sim \pi_\theta(\,\cdot \,;\, s)}[\phi_t^{\theta'}(s,a)]
-
\E_{s \sim d_t^{\pi_{\theta'}},\, a \sim \pi_{\theta'}(\,\cdot \,;\, s)}[\phi_t^{\theta'}(s,a)]
}_{(B)}.
\end{align*}

We first compute a bound for (A). For this, rewrite
\begin{align*}
\phi_t^\theta(s,a)- \phi_t^{\theta'}(s,a)
=
&\gamma^t
\bigl(
\nabla_\theta \log \pi_\theta(a \,;\, s)
-
\nabla_\theta \log \pi_{\theta'}(a \,;\, s)
\bigr)
Q_{t}^{\pi_\theta}(s,a) \\
&+
\gamma^t
\nabla_\theta \log \pi_{\theta'}(a \,;\, s)
\bigl(
Q_{t}^{\pi_\theta}(s,a)
-
Q_{t}^{\pi_{\theta'}}(s,a)
\bigr).
\end{align*}
Note that
\[
\sum_{k=0}^{T-t-2} \gamma^k \Big( \sum_{j=0}^{T-t-2-k} \gamma^j \Big)
= \frac{1-(T-t)\,\gamma^{\,T-t-1}+(T-t-1)\,\gamma^{\,T-t}}{(1-\gamma)^2},
\]
so that, by Lemma~\ref{lem:value}, one has
\[
\|Q_{t}^{\pi_\theta}
-
Q_{t}^{\pi_{\theta'}}
\|_\infty
\le
\gamma\| V_{t+1}^{\pi_\theta}-V_{t+1}^{\pi_{\theta'}} \|_\infty 
\le \big|\theta - \theta'\big|
\Pi_\ast R_\ast \gamma
\sum_{k=0}^{T-t-2} \gamma^k \Big( \sum_{j=0}^{T-t-2-k} \gamma^j \Big).
\]
Together with the Lipschitz and boundedness assumptions on the score function and the fact that $\|Q_{t}^{\pi_\theta}
\|_\infty
\le
R_\ast
\sum_{k=0}^{T-t-1} \gamma^k$, due to boundedness of the rewards,
we get
\[
\|(A)\|_\infty
\le
\gamma^t R_\ast
\Big(
L_\text{s} \sum_{k=0}^{T-t-1} \gamma^k
+
\Pi_\ast^2 \gamma
\sum_{k=0}^{T-t-2} \gamma^k \Big( \sum_{j=0}^{T-t-2-k} \gamma^j \Big)
\Big)
\big|\theta - \theta'\big|.
\]
We now turn to (B), the distribution shift. First, note that
\begin{align}\label{eqn:1}
\|(B)\|_\infty
\le 2
\|\phi_t^{\theta'}\|_\infty
\,
\mathrm{TV}\!\left(\nu(\theta)\,,\,\nu(\theta')\right),
\end{align}
where $\nu(\theta)$ is the measure on $\mathcal S \times \mathcal A$ that satisfies $\nu(s,a)=d_t^{\pi_\theta}(s)\pi_\theta(a \mid s)$. This can be seen using the dual characterization of total variation,
$
\mathrm{TV}(\mu,\nu)
\;=\;
\frac12
\sup_{\|f\|_\infty \le 1}
\left|
\int f \, d\mu
-
\int f \, d\nu
\right|$. Let $\phi:\mathcal S\times\mathcal A\to\mathbb R$ be a bounded measurable
function and define $f := \frac{\phi}{\|\phi\|_\infty}$ so that $\|f\|_\infty \le 1$. Then
\begin{align*}
\left|
\E_{\nu}[\phi]
-
\E_{\nu'}[\phi]
\right|
=
\Big|
\int \phi \, d\nu-
\int \phi \, d\nu'
\Big|
=
\|\phi\|_\infty
\Big|
\int f \, d\nu-
\int f \, d\nu'
\Big|
\le
2 \|\phi\|_\infty
\,\mathrm{TV}\!\left(\nu,\nu'\right).
\end{align*}
We now estimate the right-hand side of \eqref{eqn:1}. First, using Assumptions~\ref{ass:boundedrewards} and~\ref{ass:score} we get
\begin{align*}
\|\phi_t^{\theta'}\|_\infty
\le
\gamma^t
\Pi_\ast
R_\ast
\sum_{k=0}^{T-t-1} \gamma^k.
\end{align*}
Next, using Lemma~\ref{lemma:marginal} and Lemma~\ref{lem:TVPOLICY} gives
\begin{align*}
    \mathrm{TV}\!\left(\nu(\theta)\,,\,\nu(\theta')\right) & \le \mathrm{TV}\!\big(d_t^{\pi_\theta}\,,\,d_t^{\pi_\theta'}\big) + \frac{1}{2}\Pi_\ast |\theta-\theta'|.
\end{align*}
Moreover, Lemma~\ref{lem:marginal-tv-decomposition} together with Lemma \ref{lem:TVPOLICY} implies
that for all $t \in \{0, \dots, T\}$
\[
\mathrm{TV}\big(d_t^{\pi'},d_t^\pi\big)\leq \sum_{i=0}^{t-1} \E^\pi\left[\mathrm{TV}\big(\pi'(\,\cdot\,;\,S_i),\pi(\,\cdot\,;\,S_i)\big)\right] \le  \frac{t}{2}\Pi_\ast |\theta-\theta'|.
\]

Combining the above estimates yields
\[
\|(B)\|_\infty
\le 
\gamma^t
(t+1)\Pi_\ast^2
R_\ast 
\sum_{k=0}^{T-t-1} \gamma^k  |\theta - \theta'|
.
\]

Summing the bounds for $(A)$ and $(B)$ over $t=0,\dots,T-1$ yields the result, where the final identity in the statement of the proposition can be deduced by a careful computation, applying the formula for geometric series.
\end{proof}
In this work. we focus on discounted finite-time MDPs. However, it is natural to ask what the proof yields in the limiting infinite-horizon discounted case ($T\to\infty$, $\gamma<1$) and in the finite-horizon undiscounted case ($T<\infty$, $\gamma=1$).

\begin{remark} While our proof technique differs from the one used in~\cite{Zhang} for the infinite discounted setting, it recovers the exact same smoothness constant in the limit $T \to \infty$. In this regime, the smoothness estimate simplifies to
    \begin{align*}
        \big|\nabla J(\theta) - \nabla J(\theta')\big| \le R_\ast \Big( \frac{ L_\text{s}}{(1-\gamma)^2} + \frac{ \Pi_\ast^2 (1+\gamma)}{(1-\gamma)^3} \Big) \big|\theta - \theta'\big|
    \end{align*}
    which coincides with the bound stated in Lemma~3.2 of~\cite{Zhang}.
\end{remark}
\begin{remark}
    In the non-discounted finite-time setting ($\gamma=1$ and $T<\infty$) the same arguments work (the geometric sums simplify) and one gets
    \begin{align*}
            \big|\nabla J(\theta) - \nabla J(\theta')\big| &\le R_\ast\Big(\frac{L_\text{s}T(T+1)}{2}+\frac{\Pi_\ast^2T(2T^2+3T+1)}{6}\Big)\big|\theta - \theta'\big|,
    \end{align*}
    which is a bit smaller than the upper bound
    $$
    R_\ast\big(L_\text{s} T^2 +  \Pi_\ast^2 T^3\big) \big|\theta - \theta'\big|
    $$
   that was derived in \cite{Papini} under slightly stronger assumptions.
\end{remark}

The two remarks reflect the well-known correspondence between infinite-horizon discounted MDPs and finite-horizon undiscounted MDPs with effective horizon $T=(1-\gamma)^{-1}$. In particular, recall that the value function of an infinite-horizon discounted MDP coincides with that of an undiscounted MDP whose time horizon is an independent geometric random variable with expectation $(1-\gamma)^{-1}$.

\section{Policy gradient bias theory}\label{appendix:C}
We now come to one of the main contributions of this work: the bounds of the surrogate gradient bias used in PPO. In the next two sections we prove Theorem \ref{Thm:1}.

\subsection{Unclipped surrogate gradient bias}
In this section, we estimate the difference between the true policy gradient and the surrogate gradient 
\begin{align*}
g_{\text{PPO}}(\theta,\theta_\text{old})
&=
\sum_{t=0}^{T-1} \gamma^t \E^{\pi_{\theta_\mathrm{old}}}_\mu\Big[\frac{\nabla_\theta \pi_\theta(A_t\,;\,S_t)}{\pi_{\theta_\mathrm{old}}(A_t\,;\,S_t)} \mathbb A_t^{\pi_{\theta_\mathrm{old}}}(S_t,A_t)\Big].
\end{align*}
In the next section we transfer the bias bound to the clipped gradient $g^\text{clip}_{\text{PPO}}$ from PPO.

The estimates are based on a variant of the performance difference lemma (see \citep[Lemma 6.1]{aoarl} and \citep[Eqn. (1)]{TRPO} for infinite-time discounted MDPs) for discounted finite-time MDPs. We will add a proof for the convenience of the reader.
\begin{proposition}[Performance difference identity]\label{prop:performancedifference}
    For two arbitrary policies $\pi, \tilde{\pi}$,
    \begin{equation*}
        V^{\tilde{\pi}}(\mu) - V^\pi(\mu) = \sum_{t=0}^{T-1}\gamma^t \E^{\tilde{\pi}}_\mu\big[\mathbb A_t^\pi(S_t,A_t)\big].
    \end{equation*}
    In particular, for any two parametrized policies $\pi_\theta,\pi_{\theta_\mathrm{old}}$,
    \begin{equation}\label{eqn:true-gradient}
        \nabla_\theta J(\theta) = \nabla_\theta \left(V^{\pi_\theta}(\mu) - V^{\pi_{\theta_\mathrm{old}}}(\mu)\right) = \nabla_\theta \sum_{t=0}^{T-1}\gamma^t \E^{\pi_\theta}_\mu\big[\mathbb A_t^{\pi_{\theta_\mathrm{old}}}(S_t,A_t)\big].
    \end{equation}
\end{proposition}
\begin{proof}
   First recall that $Q_t^\pi(s,a) = \E^\pi_{S_t=s,A_t=a}\left[R_t + \gamma V_{t+1}^\pi(S_{t+1})\right]$ and $\E_{S\sim d_t^\pi, A \sim \pi(\cdot\;;S)}\left[A_t^\pi(S,A)\right] = 0$. Thus,
    \begin{align*}
        V^{\tilde{\pi}}(\mu) - V^\pi(\mu)
        = & \sum_{t=0}^{T-1} \gamma^t \left(\E^{\tilde{\pi}}[R_t] - \E^{\pi}[R_t]\right)\\
        = & \sum_{t=0}^{T-1} \gamma^t \Big(\E_{S\sim d_t^{\tilde{\pi}}, A\sim \tilde{\pi}(\cdot\;;S)}\big[\underbrace{\E^\pi_{S_t=S,A_t=A}[R_t]}_{\mathclap{=Q_t^\pi(S,A) - \gamma \E_{S'\sim p(\cdot\;;S,A)}[V_{t+1}^\pi(S')]}}\big] - \E_{S\sim d_t^{\pi}, A\sim \pi(\cdot\;;S)}\big[\E^\pi_{S_t=S,A_t=A}[R_t]\big]\Big) \\
        = & \sum_{t=0}^{T-1} \begin{aligned}[t]
                                 \gamma^t \Big( & \E_{S\sim d_t^{\tilde{\pi}}, A\sim \tilde{\pi}(\cdot\;;S)}\big[A_t^\pi(S,A) - \gamma \E_{S'\sim p(\cdot\;;S,A)}[V_{t+1}^\pi(S')] + V_t^\pi(S)\big] \\
                                 -& \E_{S\sim d_t^{\pi}, A\sim \pi(\cdot\;;S)}\big[A_t^\pi(S,A) - \gamma \E_{S'\sim p(\cdot\;;S,A)}[V_{t+1}^\pi(S')] + V_t^\pi(S)\big] \Big)
                             \end{aligned} \\
        = & \sum_{t=0}^{T-1} \gamma^t \E_{S\sim d_t^{\tilde{\pi}}, A\sim \tilde{\pi}(\cdot\;;S)}[A_t^\pi(S,A)] \\
          & +\sum_{t=0}^{T-1} \gamma^t \left(\gamma \E_{S\sim d_t^\pi,A\sim\pi(\cdot\;;S)}\big[\E_{S'\sim p(\cdot\;;S,A)}[V_{t+1}^\pi(S')]\big] - \E_{S\sim d_t^\pi,A\sim\pi(\cdot\;;S)}[V_t^\pi(S)]\right) \\
          & - \sum_{t=0}^{T-1} \gamma^t \left(\gamma \E_{S\sim d_t^{\tilde{\pi}},A\sim\tilde{\pi}(\cdot\;;S)}\big[\E_{S'\sim p(\cdot\;;S,A)}[V_{t+1}^\pi(S')]\big] - \E_{S\sim d_t^{\tilde{\pi}},A\sim\tilde{\pi}(\cdot\;;S)}[V_t^\pi(S)]\right). \\
    \end{align*}
    For the second sum, we calculate, using $V_T^\pi \equiv 0$,
    \begin{align*}
          & \sum_{t=0}^{T-1} \gamma^t \left(\gamma \E_{S\sim d_t^\pi,A\sim\pi(\cdot\;;S)}\big[\E_{S'\sim p(\cdot\;;S,A)}[V_{t+1}^\pi(S')]\big] - \E_{S\sim d_t^\pi,A\sim\pi(\cdot\;;S)}[V_t^\pi(S)]\right) \\
        = & \sum_{t=0}^{T-1} \gamma^t \left(\gamma \E^\pi[V_{t+1}^\pi(S_{t+1})] - \E^\pi[V_t^\pi(S_t)]\right)                                                                                              \\
        = &\, \gamma^T \E^\pi[V_T^\pi(S_T)] - \E^\pi[V_0^\pi(S_0)] + \sum_{t=0}^{T-2} \gamma^{t+1} \E^\pi[V_{t+1}^\pi(S_{t+1})] - \sum_{t=1}^{T-1} \gamma^t \E^\pi[V_t^\pi(S_t)]            \\
        = & - \E^\pi[V_0^\pi(S_0)],
    \end{align*}
    and analogously for the third sum. Because $d_0^\pi = \mu = d_0^{\tilde{\pi}}$, we have $\E^\pi[V_0^\pi(S_0)] = \E^{\tilde{\pi}}[V_0^\pi(S_0)]$, meaning these sums cancel, which finishes the proof.
\end{proof}
Using the performance difference identity allows us to deduce the following lemma on the difference of the true policy gradient and the (non-clipped) surrogate. From now on we use $\theta$ and $\theta_\text{old}$ for arbitrary parameters, as this will be used in the later analysis.
\begin{lemma}\label{lemma:gradient}
    \begin{align*}
  \nabla_\theta J(\theta) - g_{\text{PPO}}(\theta,\theta_\text{old})
  = \sum_{t=0}^{T-1} \gamma^t \,\nabla_\theta \left(\E^{\pi_\theta}[g_t^\theta(S_t)] - \E^{\pi_{\theta_\mathrm{old}}}[g_t^\theta(S_t)]\right).
\end{align*}
with $g_t^\theta(s)\coloneq \E_{A\sim\pi_\theta(\,\cdot\,;\,s)}[\mathbb A_t^{\pi_{\theta_\mathrm{old}}}(s,A)]$ 
\end{lemma}
\begin{proof}
    The decomposition is a consequence of the performance difference lemma and Fubini's theorem that allows us to disintegrate state- and action distributions. First, from performance difference and Fubini
    \begin{align*}
           \nabla_\theta J(\theta)
           = \nabla_\theta \sum_{t=0}^{T-1}\gamma^t \E^{\pi_\theta}_\mu\big[\mathbb A_t^{\pi_{\theta_\mathrm{old}}}(S_t,A_t)\big]
           &= \nabla_\theta \sum_{t=0}^{T-1}\gamma^t \E_{S\sim d_t^{\pi_\text{old}}}\big[\E_{A\sim \pi_\theta(\cdot\,;\,S)}\big[    \mathbb A_t^{\pi_{\theta_\mathrm{old}}}(S_t,A_t)\big]\big]\\
           &=\sum_{t=0}^{T-1} \gamma^t \,\nabla_\theta \E^{\pi_\theta}[g_t^\theta(S_t)].
    \end{align*}
    Next, 
    \begin{align*}
g_{\text{PPO}}(\theta,\theta')
&=
\sum_{t=0}^{T-1} \gamma^t \E^{\pi_{\theta_\mathrm{old}}}_\mu\Big[\frac{\nabla_\theta \pi_\theta(A_t\,;\,S_t)}{\pi_{\theta_\mathrm{old}}(A_t\,;\,S_t)} \mathbb A_t^{\pi_{\theta_\mathrm{old}}}(S_t,A_t)\Big]\\
&=
\sum_{t=0}^{T-1} \gamma^t \nabla_\theta\E^{\pi_{\theta_\mathrm{old}}}\Big[\frac{ \pi_\theta(A_t\,;\,S_t)}{\pi_{\theta_\mathrm{old}}(A_t\,;\,S_t)} \mathbb A_t^{\pi_{\theta_\mathrm{old}}}(S_t,A_t)\Big]\\
&=\sum_{t=0}^{T-1} \gamma^t \nabla_\theta\E_{S\sim d_t^{\pi_{\theta_\mathrm{old}}}}\big[\E_{A\sim\pi_\theta(\,\cdot\,;\,S)}[\mathbb A_t^{\pi_{\theta_\mathrm{old}}}(S,A)]\big],
\end{align*}
where we note that without the importance ratio product, after Fubini disintegration the importance ratio only reweights actions in state $S_t$. Taking differences gives the claim.
\end{proof}

\begin{theorem}[Unclipped Surrogate Gradient Bias]\label{thm:grad-approx-mean-tv}
    Suppose $\pi_{\theta_\mathrm{old}}$ is a behavior policy with strictly positive weights and define the mean TV distance
    \begin{align*}
        \mathrm{M_{ean}TV}(\pi_{\theta_\mathrm{old}},\pi_\theta)\coloneq \frac{1}{T} \sum_{t=0}^{T-1} \E^{\pi_{\theta_\mathrm{old}}}_\mu\big[\mathrm{TV}\left(\pi_{\theta_\mathrm{old}}(\,\cdot\,;\,S_t),\pi_\theta(\,\cdot\,;\,S_t)\right)\big]
    \end{align*}
    and the max TV distance \[\mathrm{M_{ax}TV}(\pi_{\theta_\mathrm{old}},\pi_\theta)\coloneq\max_{t=0,\dots,T-1}\E^{\pi_\mathrm{old}}_\mu\big[\mathrm{TV}\big(\pi_{\theta_\mathrm{old}}(\,\cdot\,;\,S_t),{\pi_\theta}(\,\cdot\,;\,S_t)\big)\big].\]
    Under the Assumptions \ref{ass:boundedrewards}, \ref{ass:score} it holds that
    \begin{align*}
        \big\|\nabla_\theta J(\theta) - g_{\text{PPO}}(\theta,\theta_\text{old})\big\|_\infty 
        \leq \Pi_\ast \, R_\ast\, \min\big\{ c_1 \, \mathrm{M_{ean}TV}(\pi_{\theta_\mathrm{old}},\pi_\theta), c_2 \, \mathrm{M_{ax}TV}(\pi_{\theta_\mathrm{old}},\pi_\theta)\big\},
    \end{align*}
    with
    \begin{align*}
        c_1&\coloneq \begin{cases}
            8 \frac{1-\gamma^T}{1-\gamma} T \big(\frac{\gamma}{(1-\gamma)^2}+\frac{\gamma}{1-\gamma}\big)\leq \frac{16T\gamma}{(1-\gamma)^3} &:0<\gamma<1\\
            4T^4 &:\gamma=1
        \end{cases},\\
        c_2&\coloneq \begin{cases}
            8 \frac{1-\gamma^T}{1-\gamma}\frac{2\gamma-T(T+1)\gamma^{T-1}+2(T^2-1)\gamma^T-T(T-1)\gamma^{T+1}}{(1-\gamma)^3} \leq \frac{16\gamma}{(1-\gamma)^4}\quad &:0<\gamma<1\\
            8T\frac{(T-1)T(T+1)}{3}\leq \frac{8}{3}T^4 &:\gamma=1
        \end{cases}.
    \end{align*}
\end{theorem}
The formulation of the bias bound looks a bit complicated because it combines at once finite-time discounted, finite-time non-discounted, and infinite-time discounted MDP settings. In our PPO analysis we will only work with the finite time horizon, bounding the gradient bias with the mean-TV policy distance. For discounted infinite time-horizon MDPs the reader should work with the max-TV divergence with quartic constant in the effective time-horizon $\frac{1}{1-\gamma}$.
\begin{lemma}\label{lem:g-upper-bound}
    For any $s\in\mathcal S$ with $\big\lVert \mathbb A_t^{\pi_{\theta_\mathrm{old}}}(s,\,\cdot\,)\big\rVert_\infty \leq \mathbb{A}_{\max}$, 
    \[\big\lvert \E_{A\sim\pi_\theta(\cdot\,;\,s)}\big[\mathbb A_t^{\pi_{\theta_\mathrm{old}}}(s,A)\big] \big\rvert \leq 2 \,\mathrm{TV}\big(\pi_\theta(\,\cdot\,;\,s),{\pi_{\theta_\mathrm{old}}}(\,\cdot\,;\,s)\big) \mathbb{A}_{\max}.
    \]
\end{lemma}
\begin{proof}
    Since $\E_{A\sim{\pi_{\theta_\mathrm{old}}}(\,\cdot\,;\,s)}[\mathbb A_t^{\pi_{\theta_\mathrm{old}}}(s,A)] = 0$, the TV distance inequality \citep[Proposition 4.5]{mcmixtimes} implies
    \begin{align*}
    \big|\E_{A\sim\pi_\theta(\,\cdot\,;\,s)}\big[\mathbb A_t^{\pi_{\theta_\mathrm{old}}}(s,A)\big]\big|
        & = \big|\E_{A\sim\pi_\theta(\,\cdot\,;\,s)}\big[\mathbb A_t^{\pi_{\theta_\mathrm{old}}}(s,A)\big] - \E_{A\sim{\pi_{\theta_\mathrm{old}}}(\,\cdot\,;\,s)}\big[\mathbb A_t^{\pi_{\theta_\mathrm{old}}}(s,A)\big]\big| \\
        &\leq 2 \, \mathrm{TV}(\pi_\theta(\,\cdot\,;\,s),{\pi_{\theta_\mathrm{old}}}(\,\cdot\,;\,s)) \mathbb{A}_{\max}.
    \end{align*}
\end{proof}
\begin{proof}[Proof of \cref{thm:grad-approx-mean-tv}]
    We define $g_t^\theta(s)\coloneq \E_{A\sim\pi_\theta(\,\cdot\,;\,s)}[\mathbb A_t^{\pi_{\theta_\mathrm{old}}}(s,A)]$ so that Lemma \ref{lemma:gradient} gives \[\nabla_\theta J(\theta) - g_{\text{PPO}}(\theta,\theta_\text{old})= \sum_{t=0}^{T-1} \gamma^t \,\nabla_\theta \left(\E_{S\sim d_t^{\pi_\theta}}[g_t^\theta(S)] - \E_{S\sim d_t^{\pi_{\theta_\mathrm{old}}}}[g_t^\theta(S)]\right).\] For a single coordinate $\theta_j$ of $\theta$, we first compute
    \begin{align*}
        \partial_{\theta_j} d_t^{\pi_\theta}(s) = \sum_{\tau:s_t=s}\partial_{\theta_j} \P_t^{\pi_\theta}(\tau) = \sum_{\tau:s_t=s} \P_t^{\pi_\theta}(\tau) \,\partial_{\theta_j}\! \log \P_t^{\pi_\theta}(\tau) = \sum_{\tau:s_t=s} \P_t^{\pi_\theta}(\tau) \sum_{i=0}^{t-1} \partial_{\theta_j} \! \log \pi_\theta(a_i;s_i),
    \end{align*}
    where $\tau:s_t=s$ denotes all trajectories of length $t$ ending in $s_t=s$ and the derivatives of the transition probabilities do not appear in the final expression because of their independence of $\theta$. We also have the estimate 
    \begin{align*}
        \left\lvert\partial_{\theta_j} g_t^\theta(s)\right\rvert = \Big\lvert\sum_{a\in\mathcal A} \partial_{\theta_j} \pi_\theta(a\,;\,s) \mathbb A_t^{\pi_{\theta_\mathrm{old}}}(s,a)\Big\rvert = \Big\lvert\sum_{a\in\mathcal A}\pi_\theta(a\,;\,s)\partial_{\theta_j}\!\log\pi_\theta(a;s)\mathbb A_t^{\pi_{\theta_\mathrm{old}}}(s,a)\big]\Big\rvert \leq 2 \Pi_\ast R_\ast \sum_{t=0}^{T-1} \gamma^t,
    \end{align*}
    using that the advantage is bounded above by $2 R_\ast \sum_{t=0}^{T-1} \gamma^t$ by the bounded reward assumption.
   Combining the above with \cref{lem:g-upper-bound} and applying again the TV distance inequality, we have
    \begin{align*}
        &\quad \Big\lvert \partial_{\theta_j} \Big(\E_{S\sim d_t^{\pi_\theta}}[g_t^\theta(S)] - \E_{S\sim d_t^{\pi_{\theta_\mathrm{old}}}}[g_t^\theta(S)]\Big) \Big\rvert \\
        &= \Big\lvert \sum_{s\in\mathcal{S}} \partial_{\theta_j} \!\left(d_t^{\pi_\theta}(s) g_t^\theta(s)\right) - \sum_{s\in\mathcal{S}} d_t^{\pi_{\theta_\mathrm{old}}}(s) \,\partial_{\theta_j} g_t^\theta(s) \Big\rvert \\
        &= \Big\lvert \sum_{s\in\mathcal S} \sum_{\tau:s_t=s} \P_t^{\pi_\theta}(\tau) \sum_{i=0}^{t-1} \partial_{\theta_j} \!\log \pi_\theta(a_i\,;\,s_i) \, g_t^\theta(s) + \sum_{s\in\mathcal S} (d_t^{\pi_\theta}(s) - d_t^{\pi_{\theta_\mathrm{old}}}(s)) \,\partial_{\theta_j} g_t^\theta(s) \Big\rvert \\
        &\leq t \Pi_\ast \E^{\pi_\theta}\big[g_t^\theta(S_t)\big] +4 \Pi_\ast R_\ast \Big(\sum_{i=0}^{T-1} \gamma^i\Big) \mathrm{TV}\big(d_t^{\pi_{\theta_\mathrm{old}}},d_t^{\pi_\theta}\big) \\
        &\leq \underbrace{4 \Pi_\ast R_\ast \Big(\sum_{i=0}^{T-1} \gamma^i\Big)}_{\eqcolon C} \left( t \,\E^{\pi_\theta}\big[\mathrm{TV}\big(\pi_{\theta_\mathrm{old}}(\,\cdot\,;\,S_t),\pi_\theta(\,\cdot\,;\,S_t)\big)\big] + \mathrm{TV}\big(d_t^{\pi_{\theta_\mathrm{old}}},d_t^{\pi_\theta}\big) \right).
    \end{align*}
    Because we want to avoid any expectations w.r.t. $\pi_\theta$, we again use the TV distance inequality to get
    \begin{align*}
        \E^{\pi_\theta} \big[\mathrm{TV} \big(\pi_{\theta_\mathrm{old}}(\,\cdot\,;\,S_t),{\pi_\theta}(\,\cdot\,;\,S_t)\big) \big] &\leq \E^{\pi_{\theta_\mathrm{old}}}\big[\mathrm{TV}\big(\pi_{\theta_\mathrm{old}}(\,\cdot\,;\,S_t),{\pi_\theta}(\,\cdot\,;\,S_t)\big)\big] + 2\,\mathrm{TV}\big(d_t^{\pi_{\theta_\mathrm{old}}},d_t^{\pi_\theta}\big).
    \end{align*}
    This gives
    \begin{align*}
        \Big\lvert \partial_{\theta_j} \Big(\E_{S\sim d_t^{\pi_\theta}}[g_t^\theta(S)] - \E_{S\sim d_t^{\pi_{\theta_\mathrm{old}}}}[g_t^\theta(S)]\Big) \Big\rvert \leq C \Big(&t \, \E^{\pi_{\theta_\mathrm{old}}}\big[\mathrm{TV}\big(\pi_{\theta_\mathrm{old}}(\,\cdot\,;\,S_t),{\pi_\theta}(\,\cdot\,;\,S_t)\big)\big] \\
        &+ \big(2t+ 1\big)\mathrm{TV}\big(d_t^{\pi_{\theta_\mathrm{old}}},d_t^{\pi_\theta}\big)\Big).
    \end{align*}
    Applying \cref{lem:marginal-tv-decomposition} yields
    \begin{equation}\label{eqn:thmc3-key-estimate}
        \begin{aligned}
            &\quad\big\lvert \nabla_\theta J(\theta)_j - g_{\text{PPO}}(\theta,\theta_\text{old})_j\big\rvert\\
            &\leq C \sum_{t=0}^{T-1} \gamma^t \Big(
            t \, \E^{\pi_{\theta_\mathrm{old}}}\big[\mathrm{TV}\big(\pi_{\theta_\mathrm{old}}(\,\cdot\,;\,S_t),{\pi_\theta}(\,\cdot\,;\,S_t)\big)\big] 
            + (2t+ 1)\sum_{i=0}^{t-1} \E^{\pi_{\theta_\mathrm{old}}}\left[\mathrm{TV}\big(\pi_{\theta_\mathrm{old}}(\,\cdot\,;\,S_i),\pi_\theta(\,\cdot\,;\,S_i)\big)\right]\Big).
        \end{aligned}
    \end{equation}
    Now, denoting $d_t\coloneq \E^{\pi_{\theta_\mathrm{old}}}\big[\mathrm{TV}\big(\pi_{\theta_\mathrm{old}}(\,\cdot\,;\,S_t),{\pi_\theta}(\,\cdot\,;\,S_t)\big]$, we can refactor
    \begin{align*}
       \big\lvert \nabla_\theta J(\theta)_j - g_{\text{PPO}}(\theta,\theta_\text{old})_j\big\rvert
        &\leq C \sum_{i=0}^{T-1} d_i\Big(i\gamma^i+\sum_{t=i+1}^{T-1}\gamma^t(2t+1)\Big) \\
        &\eqcolon C\sum_{i=0}^{T-1} d_i w_i(\gamma) = 4 \Pi_\ast R_\ast \sum_{i=0}^{T-1} \gamma^i \sum_{i=0}^{T-1} d_i w_i(\gamma).
    \end{align*}
     We first make the estimate \[\sum_{i=0}^{T-1} d_i w_i(\gamma)\leq \max_{i=0,\dots,T-1}w_i(\gamma)\sum_{i=0}^{T-1}d_i =\max_{i=0,\dots,T-1}w_i(\gamma)T \, \mathrm{M_{ean}TV}(\pi_{\theta_\mathrm{old}},\pi_\theta)\] and now need an upper bound for $\max_{i=0,\dots,T-1}w_i(\gamma)$.
    For $\gamma<1$, $x\mapsto x\gamma^x$ reaches a maximum of $\frac{\exp(-1)}{-\log \gamma} \leq \frac{\exp(-1)}{1-\gamma}$ at $x = (- \log \gamma)^{-1}$ and thus we have $\max_{i=0, \dots, T-1} i \gamma^i \leq \frac{\exp(-1)}{1-\gamma}$. Additionally, using a careful application of the geometric series gives \[\max_{i=0,\dots,T-1} \sum_{t=i+1}^{T-1} \gamma^t (2t+1) = \sum_{t=1}^{T-1} \gamma^t (2t+1) = 2 \frac{\gamma-T\gamma^T+(T-1)\gamma^{T+1}}{(1-\gamma)^2} + \frac{\gamma-\gamma^T}{1-\gamma} \leq \frac{2\gamma}{(1-\gamma)^2} + \frac{\gamma}{1-\gamma}\] and combining these two estimates we find $\max_{i=0,\dots,T-1} w_i(\gamma) \leq \frac{\exp(-1)}{1-\gamma} + \frac{2\gamma}{(1-\gamma)^2} + \frac{\gamma}{1-\gamma}$, which implies the constant $c_1$ from the assertion in the case $\gamma<1$. For $\gamma=1$, $c_1$ is implied by \[\max_{i=0,\dots,T-1} w_i(\gamma) = \max_{i=0,\dots,T-1} i+\sum_{t=i+1}^{T-1} (2t+1) = \max_{i=0,\dots,T-1} T^2 - 1 - i^2 - i \leq T^2.\] 
    Alternatively, for $\gamma<1$, careful application of the formula for geometric series yields
    \[\sum_{t=0}^{T-1}\gamma^t(t^2+t) = \frac{2\gamma-T(T+1)\gamma^{T-1}+2(T^2-1)\gamma^T-T(T-1)\gamma^{T+1}}{(1-\gamma)^3} \leq \frac{2\gamma}{(1-\gamma)^3}\]
    and, for $\gamma=1$, we have $\sum_{t=0}^{T-1} (t^2+t)=\frac{(T-1)T(T+1)}{3}.$
    We can use this together with \cref{eqn:thmc3-key-estimate} to estimate
    \begin{align*}
        \big\lvert \nabla_\theta J(\theta)_j - g_{\text{PPO}}(\theta,\theta_\text{old})_j\big\rvert
        &\leq C \sum_{t=0}^{T-1} \gamma^t \big(t \mathrm{M_{ax}TV}(\pi_{\theta_\mathrm{old}},\pi_\theta) + (2t+1)t\mathrm{M_{ax}TV}(\pi_{\theta_\mathrm{old}},\pi_\theta) \big) \\
        &= 8 \Pi_\ast R_\ast \mathrm{M_{ax}TV}(\pi_{\theta_\mathrm{old},\pi_\theta}) \sum_{t=0}^{T-1} \gamma^t \sum_{t=0}^{T-1} \gamma^t (t^2 + t)\\
        &\leq \Pi_\ast R_\ast c_2 \mathrm{M_{ax}TV}(\pi_{\theta_\mathrm{old}},\pi_\theta)
    \end{align*}
    with $c_2$ from the assertion.
\end{proof}

\subsection{Clipped surrogate gradient bias}
In this section, we establish an upper bound for the difference between the unclipped surrogate gradient and a clipped surrogate gradient that mimics the structure of the clipped loss introduced in the original PPO paper \cite{PPO}. Combining this bound with the upper bound derived in the section before, we can bound the distance between the clipped surrogate gradient and the true policy gradient.
We introduce a clipped surrogate gradient proxy that truncates the contribution of samples whose importance ratio deviates too much from one.
Compared to the original PPO objective \cite{PPO}, the truncation used here is symmetric in the ratio and, therefore, slightly more conservative; see Remark~\ref{rem:ppo-vs-symmetric-truncation} below.

Consider the following surrogate gradient:
\begin{align}\label{eq: clipped surrogate}
        g_{\text{PPO}}^{\text{clip}}(\theta,\theta_{\text{old}})
        :=\sum\limits_{t=0}^{T-1} \gamma^t\E^{\pi_{\theta_\mathrm{old}}}\Big[ \frac{\nabla_{\theta}\pi_\theta(A_t\, ;\, S_t)}{\pi_{\theta_\mathrm{old}}(A_t\, ;\,S_t) }\mathds{1}_{\big |\frac{\pi_\theta(A_t\, ;\, S_t)}{\pi_{\theta_\mathrm{old}}(A_t\, ;\,S_t) }-1 \big|\leq\epsilon}  \mathbb A_t^{\pi_{\theta_\mathrm{old}}}({S_t},{A_t}) \Big]
 \end{align}

\begin{remark}
\label{rem:ppo-vs-symmetric-truncation}
Note that \eqref{eq: clipped surrogate} is a two-sided truncated gradient proxy. In contrast, the original PPO clipped objective \cite{PPO} 
clips asymmetrically depending on the sign of the advantage, while \eqref{eq: clipped surrogate}  truncates both sides regardless of the sign of $\ \mathbb A_t^{\pi_{\theta_\mathrm{old}}}$, which simplifies the analysis at the cost of being more conservative.
\end{remark}

The main result of this section will be the following:

\begin{theorem}\label{thm: clipped grad-approx-mean-tv}
    Suppose $\pi_{\theta_\mathrm{old}}$ is a behavior policy with strictly positive weights.
    Under the Assumptions \ref{ass:boundedrewards}, \ref{ass:score} it holds that
    \begin{align*}
        \big\lVert g_{\text{PPO}}(\theta,\theta_\text{old})-g^{\text{clip}}_{\text{PPO}}(\theta,\theta_\text{old})\big \rVert_\infty 
        \leq C \cdot \mathrm{M_{\text{ean}}TV}(\pi_{\theta_\mathrm{old}},\pi_\theta)
    \end{align*}
   for $C:= 4\Pi_\ast R_\ast\Big(1+\frac{1}{\epsilon}\Big)\,\frac{T}{1-\gamma},$ where $\epsilon$ is the clipping parameter.
\end{theorem}
Using the triangle inequality together with Theorem \ref{thm:grad-approx-mean-tv} and Theorem \ref{thm: clipped grad-approx-mean-tv}, yields 
  \begin{align*}
        \big\lVert {\nabla_\theta J(\theta)}-g^{\text{clip}}_{\text{PPO}}(\theta,\theta_\text{old})\big \rVert_\infty 
        \leq \text{const.} \cdot \mathrm{M_{\text{ean}}TV}(\pi_{\theta_\mathrm{old}},\pi_\theta).
    \end{align*}
Estimating the mean total variation by $\frac{\Pi_\ast}{2}|\theta-\theta_\text{old}|$ via Lemma \ref{lem:TVPOLICY} finally gives the main bias bound from the main text:
\begin{theorem}[Theorem \ref{Thm:1} from the main text]\label{Thm:1'}
    Suppose $\pi_{\theta_\mathrm{old}}$ is a behavior policy with strictly positive weights.
    Under the Assumptions \ref{ass:boundedrewards}, \ref{ass:score} it holds that
     \begin{align*}
        \big\lVert {\nabla_\theta J(\theta)}-g^{\text{clip}}_{\text{PPO}}(\theta,\theta_\text{old})\big \rVert_\infty 
        \leq R \,|\theta-\theta_\text{old}|, 
    \end{align*}
    where $R = \Pi_\ast^2 R_\ast \big( \frac{8 T \gamma}{(1-\gamma)^3} + \frac{2T}{1-\gamma} (1+\frac{1}{\epsilon}) \big)$ is the sum of the constants from Theorems \ref{thm:grad-approx-mean-tv} and \ref{thm: clipped grad-approx-mean-tv} multiplied by $\frac{\Pi_\ast}{2}.$
\end{theorem}

For the proof of \autoref{thm: clipped grad-approx-mean-tv}, we need the following two lemmas. 

\begin{lemma}\label{lem: exp ratios minus 1}
    Let $P$ and $Q$ be two discrete probability distributions on $\mathcal A$. Then,
    \begin{equation*}
        \mathbb{E}_{A \sim P}\left [\left | \frac{Q(A)}{P(A)} -1 \right | \right] =\; 2\,  \mathrm{TV}(P,Q).
    \end{equation*}
\end{lemma}
\begin{proof}
By the definition of the total variation distance 
\begin{equation*}
    \mathrm{TV}(P,Q)
= \frac{1}{2}\sum_{a\in \mathcal{A}} |P(a) - Q(a)|.
\end{equation*}
Thus, we obtain
\begin{align*}
 \mathbb{E}_{A \sim P}\left [\left | \frac{Q(A)}{P(A)} -1 \right | \right]
 = 
 \sum_{a\in \mathcal{A}} P(a)\left|\frac{Q(a)}{P(a)} - 1\right|
= 
\sum_{a\in \mathcal{A}} |Q(a) - P(a)|
= 2\,\mathrm{TV}(P,Q). \qquad \qedhere
    \end{align*}
\end{proof}

\begin{lemma}\label{lem: tv bound for ratio tail exp}
 Let $P$ and $Q$ be two discrete probability distributions on $\mathcal A$.
 Then,
 \begin{align*}
      \mathbb{E}_{A \sim P}\left [\left | \frac{Q(A)}{P(A)}  \right |\mathds{1}_{\left | \frac{Q(A)}{P(A)} -1 \right |>\epsilon } \right]
       \leq \left (2+\frac{2}{\epsilon}\right) \mathrm{TV}(P,Q). 
 \end{align*}
\end{lemma}
 \begin{proof}
     Simply applying the triangle inequality $| \frac{Q(A)}{P(A)}| = | \frac{Q(A)}{P(A)}-1 +1| \leq | \frac{Q(A)}{P(A)}-1| + 1$ inside the expectation, yields
\begin{align*}
      \mathbb{E}_{A \sim P}\left [\Big| \frac{Q(A)}{P(A)}  \Big|\mathds{1}_{\big| \frac{Q(A)}{P(A)} -1 \big|>\epsilon } \right]
&\leq 
      \mathbb{E}_{A \sim P}\left [\Big| \frac{Q(A)}{P(A)}  -1 \Big|\mathds{1}_{\big| \frac{Q(A)}{P(A)} -1 \big |>\epsilon } \right]
      + \mathbb{E}_{A \sim P}\big [\mathds{1}_{\big | \frac{Q(A)}{P(A)} -1 \big |>\epsilon } \big] \\
&\leq   
          \mathbb{E}_{A \sim P}\left [\Big | \frac{Q(A)}{P(A)}  -1 \Big|  \right]
      + \mathbb{E}_{A \sim P}\big [\mathds{1}_{\big | \frac{Q(A)}{P(A)} -1 \big |>\epsilon } \big]
      \\
&=    \mathbb{E}_{A \sim P}\left [\Big| \frac{Q(A)}{P(A)}  -1 \Big|  \right]
      + \mathbb{P}_{A \sim P}\Big(\Big | \frac{Q(A)}{P(A)} -1 \Big |>\epsilon  \Big).
\end{align*} 
Applying Markov's inequality,
\begin{align*}
    \mathbb{P}_{A \sim P}\left (\Big| \frac{Q(A)}{P(A)} -1 \Big|>\epsilon  \right) \leq \frac{\mathbb{E}_{A \sim P}\left [\Big| \frac{Q(A)}{P(A)} -1 \Big| \right]}{\epsilon},
\end{align*}
together with Lemma \ref{lem: exp ratios minus 1}, yields the final result
\begin{align*}
  \mathbb{E}_{A \sim P}\left [\Big | \frac{Q(A)}{P(A)}  \Big |\mathds{1}_{\left | \frac{Q(A)}{P(A)} -1 \right |>\epsilon } \right]   
&\leq 
   2 \, \mathrm{TV}(P,Q) + \frac{2\,  \mathrm{TV}(P,Q)}{\epsilon}. \qquad\qedhere
 \end{align*}    
 \end{proof}

 Now we have all ingredients for the proof of \autoref{thm: clipped grad-approx-mean-tv}.
 \begin{proof}[Proof of \autoref{thm: clipped grad-approx-mean-tv}]
Again, we compute the differences of a single coordinate $j \in \{1, \dots, d\}$:    
\begin{align*}
     &\big\lvert g_{\text{PPO}}(\theta,\theta_\text{old})_j- g_{\text{PPO}}^{\text{clip}}(\theta,\theta_\text{old})_j\big\rvert \\
    &=
    \Big\lvert 
    \sum\limits_{t=0}^{T-1} \gamma^t \;
    \E^{\pi_{\theta_\mathrm{old}}}\Big[ \Big(\frac{\partial_{\theta_j}\pi_\theta(A_t\, ;\, S_t)}{\pi_{\theta_\mathrm{old}}(A_t\, ;\,S_t) } 
    -  
    \frac{\partial_{\theta_j}\pi_\theta(A_t\, ;\, S_t)}{\pi_{\theta_\mathrm{old}}(A_t\, ;\,S_t) }  \mathds{1}_{\big |\frac{\pi_\theta(A_t\, ;\, S_t)}{\pi_{\theta_\mathrm{old}}(A_t\, ;\,S_t) }-1 \big|\leq\epsilon}\Big ) \mathbb A_t^{\pi_{\theta_\mathrm{old}}}({S_t},{A_t}) \Big]
    \Big\rvert
    \\
&=
    \Big\lvert 
    \sum\limits_{t=0}^{T-1} \gamma^t \;
    \E^{\pi_{\theta_\mathrm{old}}}\Big[ \frac{\partial_{\theta_j}\pi_\theta(A_t\, ;\, S_t)}{\pi_{\theta_\mathrm{old}}(A_t\, ;\,S_t) } 
    \mathds{1}_{\big |\frac{\pi_\theta(A_t\, ;\, S_t)}{\pi_{\theta_\mathrm{old}}(A_t\, ;\,S_t) }-1 \big|>\epsilon}  \mathbb A_t^{\pi_{\theta_\mathrm{old}}}({S_t},{A_t}) \Big]
    \Big\rvert
     \\
&=
    \Big\lvert 
    \sum\limits_{t=0}^{T-1} \gamma^t \;
    \E^{\pi_{\theta_\mathrm{old}}}\Big[ 
    \frac{\pi_\theta(A_t\, ;\, S_t)}{\pi_{\theta_\mathrm{old}}(A_t\, ;\,S_t) }  
    \mathds{1}_{\big |\frac{\pi_\theta(A_t\, ;\, S_t)}{\pi_{\theta_\mathrm{old}}(A_t\, ;\,S_t) }-1 \big|>\epsilon}  \; (\partial_{\theta_j} \log \pi_\theta(A_t\, ;\, S_t)) \, \mathbb A_t^{\pi_{\theta_\mathrm{old}}}({S_t},{A_t}) \Big]
    \Big\rvert   \\
&\leq 2 \Pi_\ast R_\ast \frac{1-\gamma ^T}{1-\gamma}
    \sum\limits_{t=0}^{T-1} \gamma^t \; 
    \E^{\pi_{\theta_\mathrm{old}}}\Big[ 
        \Big\lvert 
            \frac{\pi_\theta(A_t\, ;\, S_t)}{\pi_{\theta_\mathrm{old}}(A_t\, ;\,S_t) } 
        \Big\rvert  
        \mathds{1}_{\big |\frac{\pi_\theta(A_t\, ;\, S_t)}{\pi_{\theta_\mathrm{old}}(A_t\, ;\,S_t) }-1 \big|>\epsilon}  
        \; 
    \Big]   ,
\end{align*}
where we have used Assumptions \ref{ass:boundedrewards}, \ref{ass:score} to bound $\big|\partial_{\theta_j} \log \pi_\theta(A_t\, ;\, S_t)\big| \leq \Pi_\ast $ and $ \big|\mathbb A_t^{\pi_{\theta_\mathrm{old}}}({S_t},{A_t})\big|\leq 2R_\ast\frac{1-\gamma^T}{1-\gamma}$.
Now, by Lemma~\ref{lem: tv bound for ratio tail exp}, we obtain for fixed $s \in \mathcal{S}$ that
\begin{equation*}
    \E_{A_t \sim \pi_{\theta_\mathrm{old}}(\,\cdot\,;\, s)}\Big[ 
        \Big\lvert 
            \frac{\pi_\theta(A_t\, ;\, s)}{\pi_{\theta_\mathrm{old}}(A_t\, ;\,s) } 
        \Big\rvert  
        \mathds{1}_{\big |\frac{\pi_\theta(A_t\, ;\, s)}{\pi_{\theta_\mathrm{old}}(A_t\, ;\,s) }-1 \big|>\epsilon}  
        \; 
    \Big] 
    \leq \Big (2+\frac{2}{\epsilon}\Big) \mathrm{TV}\left (\pi_{\theta_\mathrm{old}}(\cdot \, ;\, s),\pi_\theta(\cdot\, ;\, s) \right ). 
\end{equation*}
Integrating out the distribution of $S_t$ with the use of Fubini's Theorem, yields 
\begin{align*}
\big\lvert g_{\text{PPO}}(\theta,\theta_\text{old})_j- g_{\text{PPO}}^{\text{clip}}(\theta,\theta_\text{old})_j\big\rvert
    &=2  \Pi_\ast R_\ast\frac{1-\gamma^T}{1-\gamma}
    \sum\limits_{t=0}^{T-1} \gamma^t \; 
    \E^{\pi_{\theta_\mathrm{old}}}\Big[ 
        \Big\lvert 
            \frac{\pi_\theta(A_t\, ;\, S_t)}{\pi_{\theta_\mathrm{old}}(A_t\, ;\,S_t) } 
        \Big\rvert  
        \mathds{1}_{\big |\frac{\pi_\theta(A_t\, ;\, S_t)}{\pi_{\theta_\mathrm{old}}(A_t\, ;\,S_t) }-1 \big|>\epsilon}  
        \; 
    \Big] \\
&\leq
2\Pi_\ast R_\ast\frac{1-\gamma^T}{1-\gamma} \left (2+\frac{2}{\epsilon}\right)
    \sum\limits_{t=0}^{T-1} \gamma^t \; 
    \E^{\pi_{\theta_\mathrm{old}}}\left[ 
\mathrm{TV}\left (\pi_{\theta_\mathrm{old}}(\,\cdot \, ;\, S_t),\pi_\theta(\,\cdot\, ;\, S_t) \right )
    \right] \\
&\leq
    2\Pi_\ast R_\ast\frac{1-\gamma^T}{1-\gamma} \left (2+\frac{2}{\epsilon}\right) 
        \frac{T}{T}\sum\limits_{t=0}^{T-1}  
        \E^{\pi_{\theta_\mathrm{old}}}\left[ 
    \mathrm{TV}\left (\pi_{\theta_\mathrm{old}}(\,\cdot \, ;\, S_t),\pi_\theta(\,\cdot\, ;\, S_t) \right )
        \right] \\
&= 
    2\Pi_\ast R_\ast\frac{1-\gamma^T}{1-\gamma} \left (2+\frac{2}{\epsilon}\right) T \cdot 
         \mathrm{M_{\text{ean}}TV}(\pi_{\theta_\mathrm{old}},\pi_\theta)  .       \qquad \qedhere
\end{align*}
 \end{proof}

\subsection{Surrogate gradients are bounded}
Next, we show that the surrogate gradients are uniformly bounded.
\begin{proposition}[Surrogate gradient bounds]\label{prop:surrogategradbound}
    Under Assumptions \ref{ass:boundedrewards}, \ref{ass:score}, we have 
    \begin{align*}
    \big\lVert g^{\text{clip}}_{\text{PPO}}(\theta,\theta_\text{old})\big\rVert_\infty 
    \leq G:=2\,\Pi_\ast R_\ast\left(\frac{1-\gamma^T}{1-\gamma}\right)^2.
    \end{align*}
\end{proposition}
\begin{proof}
First, recall that for bounded rewards, the true advantage $\mathbb A^\pi_t$ is bounded by $2R_\ast\frac{1-\gamma^T}{1-\gamma}$.
Using the bounded score function assumption, i.e.\ $\|\nabla_\theta\log\pi_\theta(a\,;\, s)\|_\infty\le \Pi_\ast$ we obtain
\begin{align*}
    \big\lVert g^\text{clip}_{\text{PPO}}(\theta,\theta_\text{old})\big\rVert_\infty
    &=
    \Big\lVert \sum_{t=0}^{T-1}\gamma^t
    \E^{\pi_{\theta_\mathrm{old}}}\Big[
        \frac{\pi_{\theta}(A_t\,;\, S_t)}{\pi_{\theta_\mathrm{old}}(A_t\,;\,S_t)}\, \mathds{1}_{\big |\frac{\pi_\theta(A_t\, ;\, S_t)}{\pi_{\theta_\mathrm{old}}(A_t\, ;\,S_t) }-1 \big|\leq\epsilon}\,\nabla_\theta \log\pi_\theta(A_t\,;\, S_t)\,
        \mathbb A_{t}^{\pi_{\theta_\mathrm{old}}}(S_t,A_t)
    \Big]\Big\rVert_\infty
    \\
    &\le
    \sum_{t=0}^{T-1}\gamma^t
    \E^{\pi_{\theta_\mathrm{old}}}\Big[
        \frac{\pi_{\theta}(A_t\,;\, S_t)}{\pi_{\theta_\mathrm{old}}(A_t\,;\,S_t)}\,
        \|\nabla_\theta \log\pi_\theta(A_t\,;\, S_t)\|_\infty\,
        \big|\mathbb A_{t}^{\pi_{\theta_\mathrm{old}}}(S_t,A_t)\big|
    \Big]
    \\
    &\le
    \sum_{t=0}^{T-1}\gamma^t
    \Pi_\ast 2R_\ast\frac{1-\gamma^T}{1-\gamma}\;
    \E^{\pi_{\theta_\mathrm{old}}}\Big[\frac{\pi_{\theta}(A_t\,;\, S_t)}{\pi_{\theta_\mathrm{old}}(A_t\,;\,S_t)}\Big].
\end{align*}
Moreover,
\begin{align}\label{equation:sum1}
\E_{A_t\sim {\pi_{\theta_\mathrm{old}}}(\,\cdot \,;\,s)}\Big[\frac{\pi_{\theta}(A_t\,;\, s)}{\pi_{\theta_\mathrm{old}}(A_t\,;\,s)}\Big] 
= \sum_{a \in \mathcal A}\pi_{\theta_\mathrm{old}}(a\,;\,s)  \frac{\pi_{\theta}(a\,;\,s)}{\pi_{\theta_\mathrm{old}}(a\,;\,s)}  
= \sum_{a \in \mathcal A} \pi_{\theta}(a\,;\,s)= 1.
\end{align}
Hence, conditioning upon $S_t$, and then integrating out, yields 
\[
\big\lVert g^\text{clip}_{\text{PPO}}(\theta,\theta_\text{old})\big\rVert_\infty
\le
2\,\Pi_\ast R_\ast\frac{1-\gamma^T}{1-\gamma}\sum_{t=0}^{T-1}\gamma^t
=
2\,\Pi_\ast R_\ast\left(\frac{1-\gamma^T}{1-\gamma}\right)^2.
\]
\end{proof}
Note that the clipped surrogate gradient norm could be estimated more carefully, bounding the non-clipping probability $\mathbb{P}^{\pi_{\theta_{\mathrm{old}}}} ( | \frac{\pi_\theta(A \mid S)}{\pi_{\theta_{\mathrm{old}}}(A \mid S)} - 1 | \le \epsilon )$. Under strong policy assumptions one can use anti-concentration inequalities to show the clipped gradient norm goes to zero as $\theta$ moves away from $\theta_{\text{old}}$. Since clipping probabilities do not vanish in practice, we work with the coarse bound.


\section{Convergence Proofs}\label{sec:PPOconvergence}\label{app: stoch convergence}
We now come to the convergence proof which is built on the preliminary work of the previous sections. We follow the proof strategy presented in \cite{Mishenko}, where RR for SGD was analyzed in the supervised learning setting and push their ideas into the reinforcement learning framework. To fix suitable notation for the analysis, we slightly reformulate the policy update mechanism of PPO. Recall that we do not focus on the actor-critic aspect of PPO, i.e., for the stochastic setting, we assume access to bounded and biased advantage estimators (c.f. Assumption \ref{ass:unbiased-bounded-adv}). 

At a high level, the PPO algorithm can be described as follows.
\begin{itemize}
    \item PPO samples $n$ new rollouts of length $T$ at the beginning of a cycle $C$, samples are flattened into a state-action transition buffer of length $N:= nT$. The buffer also stores the biased advantage estimates and the corresponding time of the transition, since we include discounting. 
    \item Within a cycle, PPO proceeds with one A2C step, followed by a number of clipped importance sampling steps.
    \item The gradient steps in each cycle are partitioned into epochs. An epoch consists of $m=\frac{N}{B}$ gradient steps, where every gradient step uses $B$ transitions (without replacement) drawn from the transition buffer. Before starting an epoch the transition buffer is reshuffled.
\end{itemize}

\subsection{PPO formalism}
More formally, we consider the following algorithm. Fix 
\begin{itemize}
    \item number of cycles $C$,
    \item number of epochs $K$ per cyle,
    \item number of rollouts $n$,
    \item transition batch size $N=nT$,
    \item mini-batch size $B$ such that $m:=N/B\in\N$ is the number of gradient steps per epoch epoch
    \item constant learning rate $\eta>0$.
\end{itemize}

For each cycle $c=0,1,\dots,C-1$:
\begin{enumerate}
\item Sample a fresh dataset of $n$ rollouts of length $T$ from $\pi_{\theta_{c,0,0}}$ (this is $\pi_\text{old}$) and use these rollouts to compute (possibly biased) advantage estimates (e.g. via GAE under true value function). Flatten the resulting data into transition buffer $\{(s_c^i,a_c^i,r_c^i,\hat{\mathbb A}_c^i,t_c^i)\}_{i=0}^{N-1}$ of size $N$, where $(s_c^i,a_c^i,r_c^i,t_c^i)$ ranges over all state-action-reward-time quartets encountered in the rollouts and $\hat{\mathbb A}_c^i$ denotes the (biased by assumption) advantage estimate for $\mathbb A_{t_{c}^{i}}^{\pi_{\theta_{c,0,0}}}( s_c^{i}, a_c^{i}) $.
\item For each epoch $e=0,1,\dots,K-1$:
\begin{enumerate} 
\item Draw a fresh random permutation $\sigma_{c,e} = (\sigma_{c,e,0},\dots,\sigma_{c,e,N-1})$ of $\{0,\dots,N-1\}$, i.e. reshuffle the transition buffer.
Split it into consecutive disjoint mini-batches
\[
\mathcal B_{c,e,k}:=\{\sigma_{c,e,kB},\dots,\sigma_{c,e,(k+1)B-1}\},\qquad k=0,\dots,m-1.
 \]
\item For each step in the epoch $k=0,\dots,m-1$, compute the mini-batch gradient estimator 
\begin{align*}
\hat g_{c,e,k}:=\hat g^{\text{clip}}(\theta_{c,e,k},\theta_{c,0,0}; \mathcal B_{c,e,k})&:= \frac{T }{B}\sum_{i\in\mathcal B_{c,e,k}} 
\gamma^{t_c^{i}} 
        \frac{\nabla \pi_{\theta_{c,e,k}}( a_c^{i}\,;\, s_c^{i})}{\pi_{\theta_{c,0,0}}( a_c^{i}\,;\, s_c^{i}))}
         \mathds{1}_{\big|\frac{\pi_{\theta_{c,e,k}}( a_c^{i}\,;\, s_c^{i}))}{\pi_{c,0,0}( a_c^{i}\,;\, s_c^{i})}-1\big|\le\epsilon}
        \hat \A_c^i
\end{align*}
and update
\[
\theta_{c,e,t+1}=\theta_{c,e,t}+\eta\, \hat g_{c,e,k}.
\]
\item Set $\theta_{c,e+1,0}:=\theta_{c,e,m}$.
\end{enumerate}
\item Set $\theta_{c+1,0,0}:=\theta_{c,K,0}$.
\end{enumerate}
For the following analysis, we use the clipped surrogate
\begin{align*}
         g^{\text{clip}}_{\text{PPO}}(\theta,\theta_{c,0,0})
         &:=\sum_{t=0}^{T-1} \gamma^t \E^{\pi_{c,0,0}}\Big[
        \frac{\nabla \pi_{\theta}(A_t\,;\,S_t))}{\pi_{\theta_{c,0,0}}(A_t\,;\,S_t))}
    \mathds{1}_{|\frac{\pi_{\theta}(A_t\,;\,S_t))}{\pi_{c,0,0}(A_t\,;\,S_t)}-1|\le\epsilon}
        \mathbb A_{t}^{\pi_{\theta_{c,0,0}}}(S_t,A_t)\Big].    
\end{align*}
and unclipped surrogate
\begin{align*}
         g_{\text{PPO}}(\theta,\theta_{c,0,0})
         &:=\sum_{t=0}^{T-1} \gamma^t \E^{\pi_{c,0,0}}\Big[
        \frac{\nabla \pi_{\theta}(A_t\,;\,S_t))}{\pi_{\theta_{c,0,0}}(A_t\,;\,S_t))} 
        \mathbb A_{t}^{\pi_{\theta_{c,0,0}}}(S_t,A_t)\Big].    
\end{align*}
To link the notation to the previous sections, just recall that $\pi_{\theta_{c,0,0}}=:\pi_{\theta_\text{old}}$. 
Within each cycle, we define the clipped surrogate per-transition contribution
\begin{align*}
     g_{\text{PPO}}^{(i),\text{clip}}(\theta,\theta_{c,0,0}) &:=  T \gamma^{t_c^{i}} 
         \frac{\nabla \pi_{\theta}( a_c^{i}\,;\, s_c^{i})}{\pi_{\theta_{c,0,0}}( a_c^{i}\,;\, s_c^{i}))} 
\mathds{1}_{|\frac{\pi_{\theta}( a_c^{i}\,;\, s_c^{i}))}{\pi_{c,0,0}( a_c^{i}\,;\, s_c^{i})}-1|\le\epsilon}
         \hat \A_c^i 
\end{align*}
and the unclipped surrogate per-transition contribution
\begin{align*}
     g_{\text{PPO}}^{(i)}(\theta,\theta_{c,0,0}) &:=  T \gamma^{t_c^{i}} 
         \frac{\nabla \pi_{\theta}( a_c^{i}\,;\, s_c^{i})}{\pi_{\theta_{c,0,0}}( a_c^{i}\,;\, s_c^{i}))}
        \hat \A_c^i
\end{align*}
for $i=0, \dots, N-1$.
Note that in the first step of each cycle one has
$g_{\text{PPO}}^{(i),\text{clip}}(\theta_{c,0,0},\theta_{c,0,0}) = g_{\text{PPO}}^{(i)}(\theta_{c,0,0},\theta_{c,0,0})$.
\subsection{Proof of the deterministic case, Theorem \ref{thm:convdeterministc}}\label{app:proofdeterministic}
We start by analyzing the deterministic setting. Here, we assume that we have direct access to the clipped surrogate $g_{\text{PPO}}^{\text{clip}}$. 
For each cycle $c=0,...,C-1$ of length $K$, we consider the iterates
\begin{align} \begin{split} \label{eq:detPPOapp}
    \theta_{c,e+1}&=\theta_{c,e}+\eta\, g_\text{PPO}^\text{clip}(\theta_{c,e},\theta_{c,0}), \quad  e = 0, \dots, K-1 \\
    \theta_{c+1,0}&=\theta_{c,K}
    \end{split}
\end{align}
Thus, one surrogate gradient step corresponds to an epoch of mini-batch sample surrogate gradient steps in the stochastic setting.

In the deterministic case, we can directly invoke the bias estimate from Theorem~\ref{Thm:1}. This yields a sharper error bound and allows us to demonstrate an advantage of PPO over standard gradient ascent in many realistic settings. By contrast, in the stochastic case considered below, we must instead develop a pathwise bias bound. This will be carried out in the following subsections.

\begin{proof}[Proof of Theorem~\ref{thm:convdeterministc}]
    In the following, we interpret the clipped surrogate as biased gradient approximation of the exact gradient $\nabla J$, i.e., we define
    \[
    b_{c,e}:= g_{\text{PPO}}^{\text{clip}}(\theta_{c,e},\theta_{c,0})- \nabla J(\theta_{c,e}).
    \]
    we write the updates \eqref{eq:detPPOapp} as approximate gradient ascent scheme
    \[
        \theta_{c,e+1} = \theta_{c,e} + \eta (\nabla J(\theta_{c,e}) + b_{c,e}).
    \]
    
    By $L$-smoothness of $J$, assuming that $\eta \le \frac 1L$, we have for $e=0, \dots, K-1$
    \begin{align*}
        J(\theta_{c,e+1}) &\ge J(\theta_{c,e}) + \langle \nabla J(\theta_{c,e}),\theta_{c.e+1}-\theta_{c,e}\rangle - \frac{L}2|\theta_{c,e+1}-\theta_c|^2\\
        &=J(\theta_{c,e}) + \eta|\nabla J(\theta_{c,e})|^2 + \eta \langle \nabla J(\theta_{c,e}),b_{c,e}\rangle - \frac{L}{2}\eta^2|\nabla J(\theta_{c,e})+b_{c,e}|^2\\
        &=J(\theta_{c,e}) + (1-\frac{L\eta}{2})\eta |\nabla J(\theta_{c,e})|^2-(1-L\eta)\eta\langle \nabla J(\theta_{c,e}),b_{c,e}\rangle - \frac{L}2\eta^2|b_{c,e}|^2\\
        &\ge J(\theta_{c,e}) + \Big(1-\frac{1}{2\delta}-\frac{L}{2}\Big(1-\frac{1}{\delta}\Big)\eta\Big)\eta|\nabla J(\theta_{c,e})|^2 - \Big((1-L\eta)\eta \frac{\delta}{2} + \frac{L}2 \eta^2\Big) |b_{c,e}|^2,
    \end{align*}
    where the last inequality holds for all $\delta>0$ due to Young's inequality. In particular, for $\delta=1$ we deduce 
    \begin{align*}
        J(\theta_{c,e+1}) &\ge J(\theta_{c,e}) + \frac {\eta}{2}|\nabla J(\theta_{c,e})|^2 - \frac{\eta}{2}|b_{c,e}|^2.
    \end{align*}
    Due to Theorem~\ref{Thm:1}, we can control the bias term by
    \[|b_{c,e}| \le R |\theta_{c,e}-\theta_{c,0}|  \]
    for some $R>0$. Moreover, by Proposition~\ref{prop:surrogategradbound} there exists $G >0$ such that 
    \[|g_{\text{PPO}}^{\text{clip}}(\theta_{c,e},\theta_{c,0})| \le G \]
    for any $c=0, \dots, C-1$ and $e=0, \dots, K-1$. This implies that
    \begin{align*}
        |b_{c,e}| \le R |\theta_{c,e}-\theta_{c,0}| \le R\sum_{e'=0}^{e-1} |\theta_{c,e'+1}-\theta_{c,e'}| \le \eta e R G
    \end{align*}
    Thus,
    \[
        J(\theta_{c,e+1}) \ge J(\theta_{c,e}) + \frac {\eta}{2}|\nabla J(\theta_{c,e})|^2 - \frac{\eta^3}{2} e^2R^2 G^2.
    \]
    Rearranging this inequality and taking the sum over all $c=0,\dots,C-1$, $e=0,\dots, K-1$ we obtain
    \begin{align*}
        \frac{\eta}2\sum_{c=0}^{C-1}\sum_{e=0}^{K-1} |\nabla J(\theta_{c,e}|^2 &\le \sum_{c=0}^{C-1} \sum_{e=0}^{K-1}\big(J(\theta_{c,e+1})-J(\theta_{c,e})\big) + \sum_{c=0}^{C-1} \sum_{e=0}^{K-1}\frac{\eta^3}2e^2R^2 G^2\\
        &\le \Delta_0 + \frac{\eta^3}{12}C (K-1)K(2K-1) R^2G^2
    \end{align*}
    where we have applied the telescoping sum $\sum_{c=0}^{C-1} \sum_{e=0}^{K-1}\big(J(\theta_{c,e+1})-J(\theta_{c,e})\big)=J(\theta_{C,K})-J(\theta_{0,0})$ and $\Delta_0:=J_\ast-J(\theta_{0,0})$ denotes the initial optimality gap. Dividing both sides by $\frac{\eta}{2CK}$ yields
    \begin{align*}
        \min_{0\le c\le C-1,\ 0\le e\le K-1} |\nabla J(\theta_{c,e})|^2&\le \frac{1}{CK} \sum_{c=0}^{C-1}\sum_{e=0}^{K-1} |\nabla J(\theta_{c,e}|^2\\ &\le \frac{2\Delta_0}{\eta CK} + \frac 16 \eta^2(K-1)(2K-1) R^2G^2\,.
    \end{align*}
    
\end{proof}
In particular, when optimizing the upper bound 
    \[
    \frac{2\Delta_0}{\eta CK} + \frac 16 \eta^2(K-1)(2K-1) R^2G^2
    \]
    with respect to $\eta$ we get
        \[
            \eta^*=\min\Big( \frac 1L, \Big(\frac{6\,\Delta_0}{C K (K-1)(2K-1)\, R^2 G^2}\Big)^{\frac{1}{3}} \Big)
\]
 which, in the case $\eta^* < \frac 1L$, gives the associated upper bound
    \begin{align*}
        \min_{0\le c\le C-1,\ 0\le e\le K-1} |\nabla J(\theta_{c,e})|^2\le \Big( \frac{9}{2}\;
\frac{\Delta_0^2
(K-1)(2K-1)\,R^2 G^2}{(C K)^2} \Big)^{1/3}\,.
    \end{align*}   
To simplify the constants, we also optimize the weaker upper bound
    $$
        f(\eta):= \frac{2\Delta_0}{\eta CK} + \frac 13 \eta^2K^2 R^2G^2
    $$
    which gives
    $
        \eta^* = \min ( \frac 1L, \frac cK ),
    $
    with
    \begin{align} \label{eq:constantdet}
        c = \Big(\frac{3\,\Delta_0}{C R^2 G^2}\Big)^{\frac{1}{3}}.
    \end{align}
    For $\eta^ \ast < \frac 1L$ we get
    $$
        f(\eta^*) = \left(\frac{3\Delta_0 RG}{C}\right)^{\frac23}. 
    $$
    Similarly, let us optimize the upper bound with respect to the cycle length $K$. First, note that $K=1$ recovers the classical gradient ascent rate 
    $$
    \min_{0\le c\le C-1,\ 0\le e\le K-1} |\nabla J(\theta_{c,e})|^2\le \frac{2\Delta_0}{\eta CK} .
    $$
    However, there are scenarios in which $K > 1 $ outperforms the gradient ascent rate. Assume, for the moment, that $K \in (0,\infty)$ is a continuous variable and, again, optimize the simplified (weaker) upper bound
    $$
        f(K):= \frac{2\Delta_0}{\eta CK} + \frac 13 \eta^2K^2 R^2G^2
    $$
    with respect to $K$. Then, the optimal cycle length is given by 
    $ K^* = \frac c \eta,
    $
    where $c$ is given by \eqref{eq:constantdet}.
    Plugging this back into the convergence rate again yields
    $$
       f(K^*)= \left(\frac{3\Delta_0 RG}{C}\right)^{\frac23} .
    $$
    Thus, for all  $\eta\le \frac1L$ we get
    \begin{align*}
        \min_{0\le c\le C-1,\ 0\le e\le K-1} |\nabla J(\theta_{c,e})|^2\le \min \Big( \frac{2\Delta_0}{\eta CK}, f(\lfloor K^* \rfloor ), f(\lceil K^* \lceil ) \Big)\,.
    \end{align*}
    In conclusion, relative to standard gradient ascent, incorporating multiple biased gradient steps per cycle (i.e., selecting $K > 1$) yields faster convergence in regimes where $\Delta_0$ is large and the parameters $\eta$, $R$, $G$, and $C$ are small. 
    
\subsection{Important properties for the stochastic case}
Let $(\mathcal F_c)_{c=0, \dots, C}$ be the canonical filtration generated by the iterates $\theta$ \underline{before} the current cycle, i.e., 
$$
    \mathcal F_c = \sigma (\theta_{c',e,k} : c' = 0, \dots, c-1, e=0, \dots, K-1, k=0, \dots, m) \quad c=0, \dots, C.
$$
Note that $\theta_{c,0,0}=\theta_{c-1,K-1,m}$ is $\mathcal F_c$-measurable for all $c=1, \dots, C$ and recall that we have 
\begin{align*}
\frac{1}{N}\sum_{i=0}^{N-1}g_{\text{PPO}}^{(i),\text{clip}}(\theta_{c,0,0},\theta_{c,0,0})  =\frac{1}{N}\sum_{i=0}^{N-1} g_{\text{PPO}}^{(i)}(\theta_{c,0,0},\theta_{c,0,0}) 
\end{align*}
since there is no clipping for the first cycle step.

\begin{lemma}[Full Batch Variance]\label{lem:variance}
    Under Assumptions~\ref{ass:unbiased-bounded-adv} and~\ref{ass:score}, there exists $\sigma>0$ such that 
    \[
    \E \Big[\Big|\frac{1}{N}\sum_{i=0}^{N-1}g_{\text{PPO}}^{(i)}(\theta_{c,0,0},\theta_{c,0,0})-g_{\text{PPO}}(\theta_{c,0,0},\theta_{c,0,0})\Big|^2  \,\Big|\, \mathcal F_c\Big] \le 2\frac{\sigma^2}{N}+2T^2\Pi_\ast^2\delta^2\,
    \]
    with $\sigma^2=\frac{1}{T} \Pi_\ast^2  A_\ast^2  \Big(\frac{1-\gamma^T}{1-\gamma}\Big)^2$.
\end{lemma}
\begin{proof}
    In the full batch setting the situation is simple. By definition of the transition buffer the sum of all transition estimators is equal to the sum of $n$ (independent) rollouts. For clarity, let us first give the argument for the unbiased case ($\delta=0$), where we have 
    \begin{equation} \label{eq: unbiasdness full batch surrogate} 
    \E \Big[\frac{1}{N}\sum_{i=0}^{N-1}g_{\text{PPO}}^{(i),\text{clip}}(\theta_{c,0,0},\theta_{c,0,0})  \,\Big|\, \mathcal F_c\Big]  = \frac{1}{N}\sum_{i=0}^{N-1} \E\Big[g_{\text{PPO}}^{(i)}(\theta_{c,0,0},\theta_{c,0,0})  \,\Big|\, \mathcal F_c\Big] = g_{\text{PPO}}(\theta_{c,0,0},\theta_{c,0,0})\,. 
    \end{equation}
    By the Markov property, at the cycle start we can write
    \begin{align*}
    &\E \Big[\Big|\frac{1}{N}\sum_{i=0}^{N-1}g_{\text{PPO}}^{(i)}(\theta_{c,0,0},\theta_{c,0,0})-g_{\text{PPO}}(\theta_{c,0,0},\theta_{c,0,0})\Big|^2  \,\Big|\, \mathcal F_c\Big] \\
    &  = \frac{1}{N^2}
    \Var^{\pi_{\theta_{c,0,0}}}\Big[\sum_{j=0}^{n-1} \sum_{t=0}^{T-1} 
    T \gamma^{t} 
         \nabla \log \pi_{\theta_{c,0,0}}( A_t^{j}\,;\, S_t^{j})\, 
         \hat \A_t^j 
    \Big],
\end{align*}
where $(S^1,A^1),..., (S^n,A^n)$ are $n$ iid copies of the MDP under $\pi_{\theta_{c,0,0}}$ with advantages estimates $\hat \A^1, \dots, \hat \A^n$. Using independence and $N = n\cdot T$, followed by the bounds on the score function and advantage estimators, one gets
\begin{align*}
&\quad \frac{1}{N^2}
    \sum_{j=0}^{n-1} \Var^{\pi_{\theta_{c,0,0}}}\Big[ \sum_{t=0}^{T-1} 
     \gamma^{t} 
         \nabla \log \pi_{\theta_{c,0,0}}( A_t^{j}\,;\, S_t^{j})\, 
         \hat \A_{t}^j
    \Big]   
\\
    &= \frac{n}{N^2}
    \Var^{\pi_{\theta_{c,0,0}}} \Big[ \sum_{t=0}^{T-1} 
     \gamma^{t} 
         \nabla \log \pi_{\theta_{c,0,0}}( A_t^{1}\,;\, S_t^{1})\, 
         \hat \A_{t}^1
    \Big]    
\\
    &\leq \frac{n}{N^2}
    \E^{\pi_{\theta_{c,0,0}}} \Big[ \Big(\sum_{t=0}^{T-1} 
     \gamma^{t} 
         \nabla \log \pi_{\theta_{c,0,0}}( A_t^{1}\,;\, S_t^{1})\, 
         \hat \A_{t}^1
         \Big)^2
    \Big] 
\\
    &\leq \frac{n}{N^2} (\Pi_\ast A_\ast )^2  \Big( \sum_{t=0}^{T-1} 
     \gamma^{t} \Big)^2
     \\
         &= \frac{1}{TN} \Pi_\ast^2  A_\ast^2  \Big(\frac{1-\gamma^T}{1-\gamma}\Big)^2.
\end{align*}
Finally, by Assumption~\ref{ass:unbiased-bounded-adv} and the Markov property we have
\begin{align}\label{advantagecondition}
    \E[|\E[\hat \A_t^j\mid\mathcal F_c,A_t^{j},S_t^{j}]-\A_t^{\pi_{\theta_{c,0,0}}}(A_t^j,S_t^j)|^2\mid\mathcal F_c] = \E^{\pi_{\theta_{c,0,0}}}[|\E^{\pi_{\theta_{c,0,0}}}[\hat \A_t^j\mid A_t^{j},S_t^{j}]-\A_t^{\pi_{\theta_{c,0,0}}}(A_t^j,S_t^j)|^2]\le \delta^2
\end{align}
so that
    \begin{align*}
&\E[|T \gamma^{t} 
         \nabla\log \pi_{\theta_{c,0,0}}( A_t^{j}\,;\, S_t^{j}) 
         (\E[\hat \A_t^j\mid\mathcal F_c,A_t^{j},S_t^j]-\A_t^{\pi_{\theta_{c,0,0}}}(A_t^j,S_t^j) )|^2 \mid \mathcal F_c]\\ &\le T^2\Pi_\ast^2 \E[|\E[\hat \A_t^j\mid\mathcal F_c,A_t^j,S_t^j]-\A_t^{\pi_{\theta_{c,0,0}}}(A_t^j,S_t^j)|^2\mid\mathcal F_c]\\ &\le T^2\Pi_\ast^2 \delta^2
\end{align*}
and therefore, using $(a+b)^2 \le 2a^2+2b^2$, the claim holds with
\[
    \E \Big[\Big|\frac{1}{N}\sum_{i=0}^{N-1}g_{\text{PPO}}^{(i)}(\theta_{c,0,0},\theta_{c,0,0})-g_{\text{PPO}}(\theta_{c,0,0},\theta_{c,0,0})\Big|^2  \,\Big|\, \mathcal F_c\Big] \le 2\frac{\sigma^2}{N} +2T^2\Pi_\ast^2\delta^2\,. \qedhere
    \]
\end{proof}

\begin{lemma}[Bounded drift]\label{lem:bdd_drift}
Under Assumptions~\ref{ass:unbiased-bounded-adv} and~\ref{ass:score}, one has for all $(c,e,k)$
    \[
    \big|\hat g^{\text{clip}}(\theta_{c,e,k},\theta_{c,0,0})\big| \le G, \quad \text{ almost surely.}
    \]
    with $G = \,T \,(1+\epsilon) \,\Pi_\ast \,A_\ast $.
\end{lemma}
\begin{proof}
By definition,
\begin{align*}
\big|\hat g^{\text{clip}}(\theta_{c,e,k},\theta_{c,0,0})\big|
&=
\Big|
\frac{T}{B}\sum_{i\in\mathcal B_{c,e,k}} 
\gamma^{t_c^{i}} 
        \frac{\nabla \pi_{\theta_{c,e,k}}( a_c^{i}\,;\, s_c^{i})}{\pi_{\theta_{c,0,0}}( a_c^{i}\,;\, s_c^{i})}
        \mathds{1}_{\big|\frac{\pi_{\theta_{c,e,k}}( a_c^{i}\,;\, s_c^{i})}{\pi_{\theta_{c,0,0}}( a_c^{i}\,;\, s_c^{i})}-1\big|\le\epsilon}
        \hat \A_c^i 
\Big| \\
&=
\Big|
\frac{T}{B}\sum_{i\in\mathcal B_{c,e,k}} 
\gamma^{t_c^{i}} 
        \frac{ \pi_{\theta_{c,e,k}}( a_c^{i}\,;\, s_c^{i})}{\pi_{\theta_{c,0,0}}( a_c^{i}\,;\, s_c^{i})}
        \mathds{1}_{\big|\frac{\pi_{\theta_{c,e,k}}( a_c^{i}\,;\, s_c^{i})}{\pi_{\theta_{c,0,0}}( a_c^{i}\,;\, s_c^{i})}-1\big|\le\epsilon} \nabla\log \pi_{\theta_{c,e,k}}( a_c^{i}\,;\, s_c^{i})
       \hat \A_c^i 
\Big| \\
&\leq
\frac{T}{B}\sum_{i\in\mathcal B_{c,e,k}} 
\gamma^{t_c^{i}} 
        \frac{ \pi_{\theta_{c,e,k}}( a_c^{i}\,;\, s_c^{i})}{\pi_{\theta_{c,0,0}}( a_c^{i}\,;\, s_c^{i})}
        \mathds{1}_{\big|\frac{\pi_{\theta_{c,e,k}}( a_c^{i}\,;\, s_c^{i})}{\pi_{\theta_{c,0,0}}( a_c^{i}\,;\, s_c^{i})}-1\big|\le\epsilon} \big|\nabla\log \pi_{\theta_{c,e,k}}( a_c^{i}\,;\, s_c^{i})\big |
        \big|\hat \A_c^i
        \big| .
\end{align*}
Moreover, the clipping indicator implies
\(
\frac{\pi_{\theta_{c,e,k}}( a_c^{i}\,;\, s_c^{i})}{\pi_{\theta_{c,0,0}}( a_c^{i}\,;\, s_c^{i})}\le (1+\epsilon)
\) a.s.
This together with the assumed bounds on score (c.f. Assumption \ref{ass:score}) and advantage estimates (c.f. Assumption \ref{ass:unbiased-bounded-adv}), i.e.,
$
|\nabla\log\pi_{\theta_{c,e,k}}( a_c^{i}\,;\, s_c^{i})|\le \Pi_\ast,
$ and $
|\hat \A_c^i|
\le A_\ast,
$
implies that almost surely
\[
\big|\hat g^{\text{clip}}(\theta_{c,e,k},\theta_{c,0,0})\big|
\le
\frac{T}{B}\sum_{i\in\mathcal B_{c,e,k}} \gamma^{t_c^{i}}\, 
(1+\epsilon)\,\Pi_\ast\, A_\ast
=
T(1+\epsilon)\,\Pi_\ast\,A_\ast.
\]
\end{proof}

\begin{lemma}[Path-level bias decomposition]\label{lem:stability}
    For $t=0,\dots, T-1$, $s\in\mathcal S$, $a\in\mathcal A$ and any scalar $\A$ with $|\A|\le A_\ast$, we define
    \[ \hat g_{\text{PPO}}^{\text{clip}}(\theta,\theta_{\text{old}},t,a,s, \A) := T\gamma^t
        \frac{\nabla \pi_{\theta}(a\,;\,s)}{\pi_{\theta_{\text{old}}}(a\,;\,s)}\\
 \mathds{1}_{|\frac{\pi_{\theta}(a\,;\,s))}{\pi_{\text{old}}(a\,;\,s)}-1|\le\epsilon}
        \A \, . 
        \]
    Then, under \Cref{ass:score} and \Cref{ass:Lipschitz}, for any $t\le T-1,\ s\in\mathcal S,\ a\in\mathcal A$ it holds that
    \begin{align*}
    \big|\hat g_{\text{PPO}}^{\text{clip}}(&\theta,\theta_{\text{old}},t,a,s, \A) - \hat g_{\text{PPO}}^{\text{clip}}(\theta_{\text{old}},\theta_{\text{old}},t,a,s, \A)\big| \\
    &\le B_1 |\theta-\theta_{\text{old}}| + B_2 \min(\epsilon,|r_{\theta,\theta_{\text{old}}}(s,a)-1|) + B_3 \mathds{1}_{|r_{\theta,\theta_{\text{old}}}(s,a)-1|>\epsilon}\,, 
    \end{align*}
    where $r_{\theta,\theta_{\text{old}}}(s,a) :=\frac{\pi_{\theta}(a\,;\,s)}{\pi_{\text{old}}(a\,;\,s)}$, $B_1 = TA_\ast L_s$, and $B_2 = B_3 =TA_\ast\Pi_\ast$.
\end{lemma}

\begin{proof}
    By definition of $\hat g_{\text{PPO}}^{\text{clip}}(\theta,\theta_{\text{old}},t,a,s, \A)$ and the fact $|\mathbb A|\le A_\ast$, 
    we have
\begin{align*} 
     &\quad \big|\hat g_{\text{PPO}}^{\text{clip}}(\theta,\theta_{\text{old}},t,a,s, \A) - \hat g_{\text{PPO}}^{\text{clip}}(\theta_{\text{old}},\theta_{\text{old}},t,a,s, \A)\big| \\ 
    &\le TA_\ast\gamma^t \, |r_{\theta,\theta_{\text{old}}}(s,a) \nabla \log\pi_{\theta}(s\,;\,a) \mathds{1}_{|r_{\theta,\theta_{\text{old}}(s,a)}-1|\le\epsilon}- \nabla \log\pi_{\theta_{\text{old}}}(s\,;\,a)|\\
    &\le TA_\ast | r_{\theta_{\text{old}},\theta_{\text{old}}}(s,a) (\nabla\log\pi_{\theta_{\text{old}}}(s\,;\,a) - \nabla\log\pi_{\theta}(s\,;\,a))|\big)\\
        &\quad+ TA_\ast|\nabla \log\pi_{\theta}(s\,;\, a)(r_{\theta,\theta_{\text{old}}}(s,a)\mathds{1}_{|r_{\theta,\theta_{\text{old}}(s,a)}-1|\le\epsilon}-r_{\theta_{\text{old}},\theta_{\text{old}}}(s,a))|\\
        &= TA_\ast| \nabla\log\pi_{\theta_{\text{old}}}(s\,;\,a) - \nabla\log\pi_{\theta}(s\,;\,a)|\\
        &\quad+ TA_\ast|\nabla \log\pi_{\theta}(s\,;\, a)(r_{\theta,\theta_{\text{old}}}(s,a)\mathds{1}_{|r_{\theta,\theta_{\text{old}}(s,a)}-1|\le\epsilon}-1)|\\
        &\le TA_\ast L_s|\theta-\theta_{\text{old}}| + TA_\ast\Pi_\ast |r_{\theta,\theta_{\text{old}}}(s,a)\mathds{1}_{|r_{\theta,\theta_{\text{old}}(s,a)}-1|\le\epsilon}-1|,
    \end{align*}
    where we have used \Cref{ass:score} and \Cref{ass:Lipschitz}. Finally, we note that
    \[ \big|r_{\theta,\theta_{\text{old}}}(s,a)\mathds{1}_{|r_{\theta,\theta_{\text{old}}(s,a)}-1| \le\epsilon}-1\big| \le \begin{cases}
        \min(\epsilon,|r_{\theta,\theta_{\text{old}}}(s,a)-1|)&: r_{\theta,\theta_{\text{old}}}(s,a)\in [1-\epsilon,1+\epsilon]\\
        1&: r_{\theta,\theta_{\text{old}}}(s,a)\notin [1-\epsilon,1+\epsilon]
    \end{cases}\,,\]
    which finishes the proof with $B_1 = TA_\ast L_s$, $B_2 = B_3 =TA_\ast\Pi_\ast$.
\end{proof}

\begin{remark}
    We will make use of $\min(\epsilon,r_{\theta,\theta_{\text{old}}}(s,a))\le \epsilon^{1-p} \cdot |r_{\theta,\theta_{\text{old}}}(s,a)-1|^p$ for arbitrary $p\in(0,1)$.
\end{remark}

We will now look more closely at the upper bound from the path level bias decomposition. 
\begin{lemma}[Lipschitz policies]\label{lemma:lipschitz}
    Under Assumption~\ref{ass:Lipschitz} we have that $\theta\mapsto\pi_\theta$ is uniformly Lipschitz continuous in the sense that 
    \[\big|\pi_\theta(a\,;\,s)-\pi_{\theta'}(a\,;\, s)\big| \le \Pi_\ast\big|\theta-\theta'\big|,\quad \forall s\in \mathcal S, a\in \mathcal A, \theta,\theta'\in \R^d.\]
\end{lemma}
\begin{proof}
By the chain rule, $\nabla \pi_\theta(a; s)
= \pi_\theta(a; s)\,\nabla \log \pi_\theta(a; s)$. 
Note that $0\le \pi_\theta(a; s)\le 1$ and $|\nabla \log(\pi_\theta(a;s))|\le \Pi_\ast$. Thus, $|\nabla \log \pi_\theta(a \,;\, s)|\le \Pi_\ast$ and the Lipschitz continuity follows from the mean-value theorem.
\end{proof}

We estimate the clipping probability, which appears in the upper bound in the path-level bias decomposition within the cycles (Lemma~\ref{lemma:lipschitz}).
\begin{lemma}[Bounded weights]\label{lem:bdd_weights}
Under Assumption~\ref{ass:Lipschitz}, one has for all $q \in (0,1)$
    \begin{enumerate} 
    \item[(i)] For any $(c,e)$ it holds that
    \[ \frac{1}{m}\sum_{k=0}^{m-1}  \mathbb P\big(\frac{1}{B}\sum_{i\in\mathcal B_{c,e,k}}\big  |r_{\theta_{c,e,k},\theta_{c,0,0}}( a_{c}^{i}, s_{c}^{i}) -1\big|>\epsilon\,\big|\, \mathcal F_c\big) \le \frac{|\mathcal A|^q \, \Pi_\ast^q\, \sum_{k=0}^{m-1}\E[| \theta_{c,e,k}-\theta_{c,0,0}|^{\frac{q}{1-q}}\mid \mathcal F_c]^{1-q} }{\epsilon^q}\,.\]
    Similarly, for any $(c,e,k)$ it holds that
    \[ \frac{1}{m}\sum_{k=0}^{m-1}  \E\big[\big(\frac{1}{B}\sum_{i\in\mathcal B_{c,e,k}}\big|r_{\theta_{c,e,k},\theta_{c,0,0}}( a_{c}^{i}, s_{c}^{i}) -1|\big)^q\,\big|\,\mathcal F_c\big] \le |\mathcal A|^{q} \Pi_\ast^q  \sum_{k=0}^{m-1} \E\big[\big| \theta_{c,e,k}-\theta_{c,0,0}|^{\frac{q}{1-q}}\big|\, \mathcal F_c\big]^{1-q}\,.\]
    \item[(ii)] For any $(c,e)$ it holds that
    \[ \frac{1}{N}\sum_{i=0}^{N-1} \mathbb P\big(\big  |r_{\theta_{c,e,k},\theta_{c,0,0}}( a_{c}^{i}, s_{c}^{i}) -1\big|>\epsilon\,\big|\, \mathcal F_c\big) \le \frac{|\mathcal A|^q \, \Pi_\ast^q\, \E[| \theta_{c,e,k}-\theta_{c,0,0}|^{\frac{q}{1-q}}\mid \mathcal F_c]^{1-q} }{\epsilon^q}\,.\]
    Similarly, for any $(c,e,k)$ it holds that
    \[ \frac{1}{N}\sum_{i=0}^{N-1}  \E\big[\big|r_{\theta_{c,e,k},\theta_{c,0,0}}( a_{c}^{i}, s_{c}^{i}) -1|^q\,\big|\,\mathcal F_c\big] \le |\mathcal A|^{q}  \Pi_\ast^q \E\big[\big| \theta_{c,e,k}-\theta_{c,0,0}|^{\frac{q}{1-q}}\big|\, \mathcal F_c\big]^{1-q}\,.\]
    \end{enumerate}
\end{lemma}
\begin{proof}
    Fix $q \in (0,1)$ and let $k\in\{0,\dots,m-1\}$ be arbitrary. First, we apply Markov's inequality
    \begin{align*} 
    \mathbb P\Big( \frac{1}{B}\sum_{i\in\mathcal B_{c,e,k}}|r_{\theta_{c,e,k},\theta_{c,0,0} }( a_{c}^{i}, s_{c}^{i})-1|>\epsilon\mid \mathcal F_c\Big) &\le \frac{\E[\big(\frac{1}{B}\sum_{i\in\mathcal B_{c,e,k}}|r_{\theta_{c,e,k},\theta_{c,0,0}}( a_{c}^{i}, s_{c}^{i}) -1|\big)^q\mid \mathcal F_c]}{\epsilon^q}\\ 
    &= \frac{\E\Big[\big(\frac{1}{B}\sum_{i\in\mathcal B_{c,e,k}}\frac{| \pi_{\theta_{c,e,k}}( a_{c}^{i}\,;\, s_{c}^{i})-\pi_{\theta_{c,0,0}}( a_{c}^{i}\,;\, s_{c}^{i})|}{\pi_{\theta_{c,0,0}}( a_{c}^{i}\,;\, s_{c}^{i})}\big)^q\mid \mathcal F_c\Big]}{\epsilon^q}.
    \end{align*}
    Using the policy Lipschitz property from \ref{lemma:lipschitz} we have 
    \begin{align*}
       &\quad \frac{\E\Big[\big(\frac{1}{B}\sum_{i\in\mathcal B_{c,e,k}}\frac{| \pi_{\theta_{c,e,k}}( a_{c}^{i}\,;\, s_{c}^{i})-\pi_{\theta_{c,0,0}}( a_{c}^{i}\,;\, s_{c}^{i})|}{\pi_{\theta_{c,0,0}}( a_{c}^{i}\,;\, s_{c}^{i}))}\big)^q\,\Big|\, \mathcal F_c\Big]}{\epsilon^q}\\
       &\le\frac{\E\Big[\big(\frac{1}{B}\sum_{i\in\mathcal B_{c,e,k}}\frac{\Pi_\ast| \theta_{c,e,k}-\theta_{c,0,0}|}{\pi_{\theta_{c,0,0}}( a_{c}^{i}\,;\, s_{c}^{i})}\big)^q\,\Big|\, \mathcal F_c\Big]}{\epsilon^q}\\
        &=\Pi_\ast^q\frac{\E\Big[| \theta_{c,e,k}-\theta_{c,0,0}|^q\big(\frac{1}{B}\sum_{i\in\mathcal B_{c,e,k}}\frac{1}{\pi_{\theta_{c,0,0}}( a_{c}^{i}\,;\, s_{c}^{i})}\big)^q\,\Big|\, \mathcal F_c\Big]}{\epsilon^q}\,,
    \end{align*}
    where we used that $|\theta_{c,e,k}-\theta_{c,0,0}|$ is independent of $B_{c,e,k}$.
    
    Next, we apply (conditional) Hölder's inequality to deduce
    \begin{align*}
        &\quad \E\Big[| \theta_{c,e,k}-\theta_{c,0,0}|^q\big(\frac{1}{B}\sum_{i\in\mathcal B_{c,e,k}}\frac{1}{\pi_{\theta_{c,0,0}}( a_{c}^{i}\,;\, s_{c}^{i})}\big)^q \mid \mathcal F_c\Big] \\ &\le \E[| \theta_{c,e,k}-\theta_{c,0,0}|^{qs}\mid\mathcal F_c]^{1/s} \, \E[\big(\frac{1}{B}\sum_{i\in\mathcal B_{c,e,k}}\frac{1}{\pi_{\theta_{c,0,0}}( a_{c}^{i}\,;\, s_{c}^{i})}\big)^{\frac{qs}{s-1}}\mid\mathcal F_c]^{\frac{s-1}{s}},
    \end{align*}
    where $s = 1+\frac{q}{1-q}>1$,  which gives the relation $qs=\frac{q}{1-q}$, $\frac{1}{s} = 1-q$ and $\frac{s-1}{s} = q$. Hence, 
    \begin{align*}
        &\quad \E\Big[| \theta_{c,e,k}-\theta_{c,0,0}|^q\big(\frac{1}{B}\sum_{i\in\mathcal B_{c,e,k}}\frac{1}{\pi_{\theta_{c,0,0}}( a_{c}^{i}\,;\, s_{c}^{i})}\big)^q \mid \mathcal F_c\Big] \\ &\le \E[| \theta_{c,e,k}-\theta_{c,0,0}|^{\frac{q}{1-q}}\mid\mathcal F_c]^{1-q} \, \E\Big[\frac{1}{B}\sum_{i\in\mathcal B_{c,e,k}}\frac{1}{\pi_{\theta_{c,0,0}}( a_{c}^{i}\,;\, s_{c}^{i})}\mid\mathcal F_c\Big]^{q},
    \end{align*}
    Finally, we use Jensen's inequality, to get
    \begin{align*} 
    \frac 1m \sum_{k=1}^{m-1} \E\Big[\frac{1}{B}\sum_{i\in\mathcal B_{c,e,k}}\frac{1}{\pi_{\theta_{c,0,0}}( a_{c}^{i}\,;\, s_{c}^{i})}\mid\mathcal F_c\Big]^q 
    &\le \E\Big[\frac{1}{m}\sum_{k=1}^{m-1}  \frac{1}{B}\sum_{i\in\mathcal B_{c,e,k}}\frac{1}{\pi_{\theta_{c,0,0}}( a_{c}^{i}\,;\, s_{c}^{i})}\,\Big|\,\mathcal F_c\Big]^q\\ 
    &=  \E\Big[\frac{1}{N}\sum_{i=1}^{N-1}\frac{1}{\pi_{\theta_{c,0,0}}( a_{c}^{i}\,;\, s_{c}^{i})}\,\Big|\, \mathcal F_c\Big]^q .
    \end{align*}
    Since, conditioned on $\mathcal F_c$, $(s_c^i,a_c^i)_{i=1, \dots, N-1}$ are $n$ independent runs of the MDP using the policy $\pi_{\theta_{c,0,0}}$, we get
    $$
        \E\Big[\frac{1}{N}\sum_{i=1}^{N-1}\frac{1}{\pi_{\theta_{c,0,0}}( a_{c}^{i}\,;\, s_{c}^{i})}\,\Big|\, \mathcal F_c\Big] = \frac{1}{T} \sum_{t=0}^{T-1} \sum_{s \in \mathcal S} \E^{\pi_{\theta_{c,0,0}}}\Big[\1_{S_t=s} \frac{1}{\pi_{\theta_{c,0,0}}(A_t\, ; \, S_t)}\Big] = |\mathcal A|.
    $$
    
    The remaining three claims follow by similar arguments.
    \end{proof}

\begin{lemma}[$L^2$-Accumulated drift control for $e$th epochs]\label{lem:bddincrements}
    Under Assumptions~\ref{ass:unbiased-bounded-adv} and~\ref{ass:score},
    one has for all $(c,e,k)$ and $p>0$ that
    \[ \E\big[|\theta_{c,e,0}-\theta_{c,0,0}|^p\,\big|\, \mathcal F_c\big] \le \eta^p e^pm^p G^p \]
    and 
    \[\E\Big[\frac{1}{m}\sum_{k=0}^{m-1} |\theta_{c,e,k}-\theta_{c,0,0}|^p \,\Big|\,\mathcal F_c\Big]\le \eta^p (e+1)^p m^p G^p\,,\]
    where $G = \,T \,(1+\epsilon) \,\Pi_\ast \,A_\ast $.
 \end{lemma}
\begin{proof}
By definition of the PPO iteration with constant learning rate (summing gradients of $e-1$ completed epochs and the partial current epoch) we have
    \begin{align*} 
    \E\Big[\frac{1}{m}\sum_{k=0}^{m-1}|\theta_{c,e,k}-\theta_{c,0,0}|^p\,\Big|\,\mathcal F_c\Big]
    &=\frac{\eta^p}{m}\sum_{k=0}^{m-1}\E\Big[\Big|\sum_{e'=0}^{e-1} \sum_{k'=0}^{m-1}\hat g_{c,e',k'} + \sum_{k'=0}^{k-1}\hat g_{c,e,k'}\Big|^p\,\Big|\,\mathcal F_c \Big]\\& \le \eta^p(e+1)^pm^pG^p,
    \end{align*}
    where we have used Lemma~\ref{lem:bdd_drift}. The first claim follows from only considering the first summand in the latter sum.
\end{proof}

\subsection{Proof of the stochastic case (PPO), Theorem \ref{Thm:3}}\label{app:PPOconv}
We start by proving an ascent property within a fixed cycle. A crucial ingredient is the $L$-smoothness of $J$ shown in Proposition~\ref{lem:smooth}. This part of the proof is inspired by the SGD setting studied in \cite{Mishenko}, and we study the ascent effect of all iterations in an epoch combined.
\begin{lemma}[Per-epoch ascent property]\label{lem:perepochdescent}
Let $\eta\le \frac{1}{Lm}$, then for each cycle $c = 0, \dots, C-1$ and each epoch $e=0,\dots,K-1$ it holds almost surely that
\[
J(\theta_{c,e+1,0})
\ge
J(\theta_{c,e,0})
+\frac{\eta m}{2}|\nabla J(\theta_{c,e,0})|^2
-\frac{\eta m}{2}\,\Big|\nabla J(\theta_{c,e,0})-\frac{1}{m}\sum_{k=0}^{m-1}\hat g_{c,e,k}\Big|^2.
\]
\end{lemma}

\begin{proof}
    By the ascent lemma, under the $L$-smoothness of $J$ (Proposition \ref{lem:smooth} we have
    \begin{align*}
        &\quad J(\theta_{c,e+1,0})\\
        &\ge J(\theta_{c,e,0}) - \langle \nabla J(\theta_{c,e,0}),\theta_{c,e+1,0}-\theta_{c,e,0}\rangle - \frac{L}2|\theta_{c,e+1,0}-\theta_{c,e,0}|^2\\
        &=J(\theta_{c,e,0}) + \eta m \langle \nabla J(\theta_{c,e,0}),\frac{1}{m}\sum_{k=0}^{m-1} \hat g_{c,e,k}\rangle - \frac{\eta^2 m^2 L}{2} \Big|\frac{1}{m}\sum_{k=0}^{m-1} \hat g_{c,e,k}\Big|^2 \\
        &= J(\theta_{c,e,0}) + \frac{\eta m}{2}\Big(|\nabla J(\theta_{c,e,0})|^2 + |\frac{1}{m}\sum_{k=0}^{m-1} \hat g_{c,e,k}|^2 - |\nabla J(\theta_{c,e,0}) - \frac{1}{m}\sum_{k=0}^{m-1} \hat g_{c,e,k}|^2\Big) - \frac{\eta^2 m^2 L}{2} \Big|\frac{1}{m}\sum_{k=0}^{m-1} \hat g_{c,e,k}\Big|^2 \\
        &=J(\theta_{c,e,0}) + \frac{\eta m}{2}|\nabla J(\theta_{c,e,0})|^2 +\frac{\eta m}2 (1-L\eta m) \Big|\frac{1}{m}\sum_{k=0}^{m-1} \hat g_{c,e,k}\Big|^2 -\frac{\eta m}{2} \Big|\nabla J(\theta_{c,e,0}) -\frac{1}{m}\sum_{k=0}^{m-1} \hat g_{c,e,k}\Big|^2\\
        &\ge J(\theta_{c,e,0}) + \frac{\eta m}{2}\big|\nabla J(\theta_{c,e,0})\big|^2 -\frac{\eta m}{2} \Big|\nabla J(\theta_{c,e,0}) - \frac{1}{m}\sum_{k=0}^{m-1} \hat g_{c,e,k}\Big|^2,
    \end{align*}
    where we have used $1-L\eta m\ge 0$ by the assumption on $\eta$.
\end{proof}

In order to derive a convergence rate for PPO, we are left to upper bound 
\[\E\Big[\Big|\nabla J(\theta_{c,e,0}) - \frac{1}{m}\sum_{k=0}^{m-1} \hat g_{c,e,k}\Big|^2\,\Big|\,\mathcal F_c\Big]\,.\]
For this, we decompose
\begin{equation} \label{eq:decomposition}
\begin{split}
\Big|\nabla J(\theta_{c,e,0}) - \frac{1}{m}\sum_{k=0}^{m-1} \hat g_{c,e,k}\Big|^2
&\le 2\Big|\nabla J(\theta_{c,e,0})-\frac{1}{N}\sum_{i=0}^{N-1} g_{\text{PPO}}^{(i),\text{clip}}(\theta_{c,e,0},\theta_{c,0,0})\Big|^2\\
&\quad+2\Big|\frac{1}{N}\sum_{i=0}^{N-1} g_{\text{PPO}}^{(i),\text{clip}}(\theta_{c,e,0},\theta_{c,0,0})-\frac{1}{m}\sum_{k=0}^{m-1} \hat g_{c,e,k}\Big|^2
\end{split}
\end{equation}
and consider both terms separately. 

\begin{lemma}\label{lem:bound1}
Under Assumptions~\ref{ass:unbiased-bounded-adv}, \ref{ass:score}, and~\ref{ass:Lipschitz}, one has for $p,q\in(0,1)$ and any $(c,e)$ that
    \begin{align*} 
    &\E\Big[\Big|\nabla J(\theta_{c,e,0})-\frac{1}{N}\sum_{i=1}^{N-1} g_{\text{PPO}}^{(i),\text{clip}}(\theta_{c,e,0},\theta_{c,0,0})\Big|^2\,\Big|\,\mathcal F_c\Big] \\
    &\le \frac{3\sigma^2}{N} + 3L^2 \E\big[\big|\theta_{c,e,0}-\theta_{c,0,0}|^2\,\big|\,\mathcal F_c\big]\\
    &\quad + 9B_1^2 \E[|\theta_{c,e,0}-\theta_{c,0,0}|^2\mid\mathcal F_c\big] + 9B_2^2 \epsilon^{2(1-p)}\frac{1}{N-1}\sum_{i=0}^{N-1}\E\big[|r_{\theta_{c,e,0},\theta_{c,0,0}}(a_{c}^{i}, s_{c}^{i}) -1|^p\big|\,\mathcal F_c\big]\\
    &\quad+ 9B_3^2\frac{1}{N-1}\sum_{i=0}^{N-1}\mathbb P\big(|r_{\theta_{c,e,0},\theta_{c,0,0}}( a_{c}^{i}, s_{c}^{i}) -1|>\epsilon\,\big|\, \mathcal F_c\big) \,. 
    \end{align*}
    In particular, we have
    \begin{align*} 
    &\E\Big[\Big|\nabla J(\theta_{c,e,0})-\frac{1}{N}\sum_{i=0}^{N-1} g_{\text{PPO}}^{(i),\text{clip}}(\theta_{c,e,0},\theta_{c,0,0})\Big|^2\,\Big|\,\mathcal F_c\Big] \\
    \le &3\eta^2(3B_1^2+L^2)K^2m^2G^2 + 9\eta^{2p}B_2^2\epsilon^{2(1-p)}|\mathcal A|^{2p}\Pi_\ast^{2p} K^{2p} m^{2p}G^{2p} \\
    &+ \frac{9\eta^q B_3^2|\mathcal A|^q \Pi_\ast^q K^q  m^qG^q }{\epsilon^q}+ \frac{6\sigma^2}{N}+6T^2\Pi_\ast^2 \delta^2,
    \end{align*}
    with constants defined in the lemmas above.
\end{lemma}
\begin{proof}
 We further decompose
    \begin{align*}
        \Big|\nabla J(\theta_{c,e,0})-\frac{1}{N}\sum_{i=0}^{N-1} g_{\text{PPO}}^{(i),\text{clip}}(\theta_{c,e,0},\theta_{c,0,0})\Big|^2
        &\le 3\Big|\frac{1}{N}\sum_{i=0}^{N-1} g_{\text{PPO}}^{(i),\text{clip}}(\theta_{c,e,0},\theta_{c,0,0})- \frac{1}{N}\sum_{i=0}^{N-1} g_{\text{PPO}}^{(i),\text{clip}}(\theta_{c,0,0},\theta_{c,0,0})\Big|^2\\
        &\quad+3 \Big|\frac{1}{N}\sum_{i=0}^{N-1} g_{\text{PPO}}^{(i),\text{clip}}(\theta_{c,0,0},\theta_{c,0,0})-\nabla J(\theta_{c,0,0})\Big|^2\\
        &\quad+3 \Big|\nabla J(\theta_{c,0,0})-\nabla J(\theta_{c,e,0})\Big|^2\,.
    \end{align*}
    For the second term, note that, at the beginning of the cycle the clipping probability is $0$, i.e. $g_{\text{PPO}}^{(i),\text{clip}}(\theta_{c,0,0},\theta_{c,0,0})= g_{\text{PPO}}^{(i)} (\theta_{c,0,0},\theta_{c,0,0})$ for all $i=0, \dots, N-1$.
    Therefore, we apply Lemma~\ref{lem:variance}  to bound
    \[3 \E\Big[\Big|\frac{1}{N}\sum_{i=0}^{N-1} g_{\text{PPO}}^{(i),\text{clip}}(\theta_{c,0,0},\theta_{c,0,0})-\nabla J(\theta_{c,0,0})\Big|^2\,\Big|\,\mathcal F_c\Big] \le \frac{6\sigma^2}{N} +6T^2\Pi_\ast^2 \delta^2\,.\]
    For the third term, we use $L$-smoothness of $J$,
    \[3\E\big[ \big|\nabla J(\theta_{c,0,0})-\nabla J(\theta_{c,e,0})\big|^2\mid \mathcal F_c\big] \le 3L^2\E[|\theta_{c,e,0}-\theta_{c,0,0}|^2 \mid \mathcal F_c]\,.\]
    By Lemma~\ref{lem:bddincrements}, we can further bound this expression by
    \[3\E\big[ \big|\nabla J(\theta_{c,0,0})-\nabla J(\theta_{c,e,0})\big|^2\,\big|\, \mathcal F_c\big] \le 3\eta^2L^2K^2m^2G^2\,.\]
    For the first term, we use Lemma~\ref{lem:stability} with the abstract $\A$ given by the estimate $\hat \A_c^i$ together with Assumption \ref{ass:unbiased-bounded-adv} and the fact that $\min(x,y)\le x^py^{1-p}$ for arbitrary $p\in(0,1)$,
    \begin{align*} 
    & 3\Big|\frac{1}{N}\sum_{i=0}^{N-1} g_{\text{PPO}}^{(i),\text{clip}}(\theta_{c,e,0},\theta_{c,0,0})- \frac{1}{N}\sum_{i=0}^{N-1} g_{\text{PPO}}^{(i),\text{clip}}(\theta_{c,0,0},\theta_{c,0,0})\Big|^2\\
    \le &9B_1^2 \frac{1}{N}\sum_{i=0}^{N-1}|\theta_{c,e,0}-\theta_{c,0,0}|^2 
    +9B_2^2\epsilon^{2(1-p)} \frac{1}{N}\sum_{i=0}^{N-1}|r_{\theta_{c,e,0},\theta_{c,0,0}}(s_c^{i},a_c^i)-1|^{2p} \\
    &+ 9B_3^2 \frac{1}{N}\sum_{i=0}^{N-1}\mathds{1}_{|r_{\theta_{c,e,0},\theta_{c,0,0}}(s_c^i,a_c^i)-1|>\epsilon}\,.
    \end{align*}
    Taking conditional expectation with respect to $\mathcal F_c$ we apply Lemma~\ref{lem:bdd_weights} to deduce
    \begin{align*}
        &\E\Big[3\Big|\frac{1}{N}\sum_{i=0}^{N-1} g_{\text{PPO}}^{(i),\text{clip}}(\theta_{c,e,0},\theta_{c,0,0})- \frac{1}{N}\sum_{i=0}^{N-1} g_{\text{PPO}}^{(i),\text{clip}}(\theta_{c,0,0},\theta_{c,0,0})\Big|^2\,\Big|\, \mathcal F_c\Big]\\
        &\le 9B_1^2 \E\big[\big|\theta_{c,e,0}-\theta_{c,0,0}|^2\,\big|\,\mathcal F_c\big]
         +9B_2^2\epsilon^{2(1-p)} |\mathcal A|^{2p} \Pi_\ast^{2p} E[|\theta_{c,e,0}-\theta_{c,0,0}|^{\frac{2p}{1-2p}}\mid\mathcal F_c]^{1-2p}\\
        &\quad+ 9B_3^2 \frac{|\mathcal A|^q \, \Pi_\ast^q\, \E[| \theta_{c,e,0}-\theta_{c,0,0}|^{\frac{q}{1-q}}\mid \mathcal F_c]^{1-q} }{\epsilon^q}\,.
    \end{align*}
    Now, we can apply Lemma~\ref{lem:bddincrements} to get 
    \begin{align*}
        &\quad \E\Big[\Big|\frac{1}{N}\sum_{i=0}^{N-1} g_{\text{PPO}}^{(i),\text{clip}}(\theta_{c,e,0},\theta_{c,0,0})- \frac{1}{N}\sum_{i=0}^{N-1} g_{\text{PPO}}^{(i),\text{clip}}(\theta_{c,0,0},\theta_{c,0,0})\Big|^2\,\Big|\, \mathcal F_c\Big]\\
        &\le 9\eta^2B_1^2K^2m^2G^2 + 9\eta^{2p}B_2^2\epsilon^{2(1-p)}|\mathcal A|^{2p}\Pi_\ast^{2p} K^{2p}m^{2p}G^{2p}+ \frac{9\eta^q B_3^2|\mathcal A|^q \Pi_\ast^q K^q m^qG^q }{\epsilon^q}\,.
    \end{align*}
\end{proof}

For the second term in \eqref{eq:decomposition} we prove the following upper bound.
\begin{lemma}\label{lem:bound2}
    Under Assumptions~\ref{ass:unbiased-bounded-adv}, \ref{ass:score}, and~\ref{ass:Lipschitz}, one has for $p,q\in(0,1)$ and any $(c,e)$ that
    \begin{align*} 
    &\quad \E\Big[\Big|\frac{1}{N}\sum_{i=0}^{N-1} g_{\text{PPO}}^{(i),\text{clip}}(\theta_{c,e,0},\theta_{c,0,0})-\frac{1}{m}\sum_{k=0}^{m-1} \hat g_{c,e,k}\Big|^2\,\Big|\,\mathcal F_c\Big]\\ 
    &\le 12\eta^2B_1^2K^2m^2G^2 + 12\eta^{2p}B_2^2\epsilon^{2(1-p)}|\mathcal A|^{2p}\Pi_\ast^{2p} K^{2p} m^{2p}G^{2p}+ \frac{12 \eta^q B_3^2|\mathcal A|^q \Pi_\ast^q K^q m^qG^q }{\epsilon^q},
    \end{align*}
    with constants defined in the lemmas above.
\end{lemma}
\begin{proof}
We begin by applying Jensen's inequality 
\begin{align*}
    &\quad\Big|\frac{1}{N}\sum_{i=0}^{N-1} g_{\text{PPO}}^{(i),\text{clip}}(\theta_{c,e,0},\theta_{c,0,0})-\frac{1}{m}\sum_{k=0}^{m-1} \hat g_{c,e,k}\Big|^2\\ 
    &= \Big|\frac{1}{N}\sum_{i=0}^{N-1} g_{\text{PPO}}^{(i),\text{clip}}(\theta_{c,e,0},\theta_{c,0,0})-\frac{1}{m}\sum_{k=0}^{m-1} \frac{1}{B}\sum_{i\in\mathcal B_{c,e,k}} g_{\text{PPO}}^{(i),\text{clip}}(\theta_{c,e,k},\theta_{c,0,0})\Big|^2\\
    &\le 2 \Big|\frac{1}{N}\sum_{i=0}^{N-1} g_{\text{PPO}}^{(i),\text{clip}}(\theta_{c,e,0},\theta_{c,0,0})-\frac{1}{N}\sum_{i=0}^{N-1} g_{\text{PPO}}^{(i),\text{clip}}(\theta_{c,0,0},\theta_{c,0,0})\Big|^2\\
    &\quad+2\Big|\frac{1}{m}\sum_{k=1}^{m-1} \frac{1}{B}\sum_{i\in\mathcal B_{c,e,k}} g_{\text{PPO}}^{(i),\text{clip}}(\theta_{c,e,k},\theta_{c,0,0})-\frac{1}{N}\sum_{i=0}^{N-1} g_{\text{PPO}}^{(i),\text{clip}}(\theta_{c,0,0},\theta_{c,0,0})\Big|^2\\
    &\le \frac{2}{N}\sum_{i=0}^{N-1} \Big| g_{\text{PPO}}^{(i),\text{clip}}(\theta_{c,e,0},\theta_{c,0,0})-g_{\text{PPO}}^{(i),\text{clip}}(\theta_{c,0,0},\theta_{c,0,0})\Big|^2\\
    &\quad+\frac{2}{m}\sum_{k=0}^{m-1} \frac{1}{B}\sum_{i\in\mathcal B_{c,e,k}} \Big| g_{\text{PPO}}^{(i),\text{clip}}(\theta_{c,e,k},\theta_{c,0,0})-g_{\text{PPO}}^{(i),\text{clip}}(\theta_{c,0,0},\theta_{c,0,0})\Big|^2\\
\end{align*}
Next, we apply Lemma~\ref{lem:stability} with $\A$ given by the estimate $\hat \A_c^i$ together with Assumption \ref{ass:unbiased-bounded-adv}, take conditional expectation with respect to $\mathcal F_c$ and apply Lemma~\ref{lem:bdd_weights} to derive
\begin{align*}
    &\frac{2}{N}\sum_{i=0}^{N-1}  \E[\big| g_{\text{PPO}}^{(i),\text{clip}}(\theta_{c,e,0},\theta_{c,0,0})-g_{\text{PPO}}^{(i),\text{clip}}(\theta_{c,0,0},\theta_{c,0,0})\big|^2\,|\,\mathcal F_c\big]\\ &\le 6B_1^2 \E\big[\big|\theta_{c,e,0}-\theta_{c,0,0}\big|^2\,\big|\,\mathcal F_c\big]
        + 6B_2^2\epsilon^{2(1-p)} |\mathcal A|^{2p} \Pi_\ast^{2p} E\big[\big|\theta_{c,e,0}-\theta_{c,0,0}\big|^{\frac{2p}{1-2p}}\,\big|\,\mathcal F_c\big]^{1-2p}\\
        &\quad+ 6B_3^2 \frac{|\mathcal A|^q \, \Pi_\ast^q\, \E\big[\big| \theta_{c,e,0}-\theta_{c,0,0}\big|^{\frac{q}{1-q}}\,\big|\, \mathcal F_c\big]^{1-q} }{\epsilon^q}\,\\
        &\le 6\eta^2B_1^2K^2m^2G^2 + 6\eta^{2p}B_2^2\epsilon^{2(1-p)}|\mathcal A|^{2p}\Pi_\ast^{2p} K^{2p} m^{2p}G^{2p}+ \frac{6\eta^q B_3^2|\mathcal A|^q \Pi_\ast^q K^q m^qG^q }{\epsilon^q}\,.
\end{align*}
Similarly, we can apply Lemma~\ref{lem:bdd_weights} to bound 
\begin{align*}
    &\quad \E\big[\frac{2}{m}\sum_{k=0}^{m-1}\sum_{i\in\mathcal B_{c,e,k}}\big| g_{\text{PPO}}^{(i),\text{clip}}(\theta_{c,e,k},\theta_{c,0,0})-g_{\text{PPO}}^{(i),\text{clip}}(\theta_{c,0,0},\theta_{c,0,0})\big|^2\,\big|\,\mathcal F_c\big]\\ &\le 6B_1^2 \frac{1}{m}\sum_{k=0}^{m-1}\E\big[\big|\theta_{c,e,k}-\theta_{c,0,0}|^2\,\big|\,\mathcal F_c\big]
        + 6B_2^2\epsilon^{2(1-p)} |\mathcal A|^{2p} \Pi_\ast^{2p} \frac{1}{m}\sum_{k=0}^{m-1}\E\big[\big|\theta_{c,e,k}-\theta_{c,0,0}\big|^{\frac{2p}{1-2p}}\,\big|\,\mathcal F_c\big]^{1-2p}\\
        &\quad+ 6B_3^2 \frac{1}{m}\sum_{k=0}^{m-1} \frac{|\mathcal A|^q \, \Pi_\ast^q\, \E[| \theta_{c,e,k}-\theta_{c,0,0}|^{\frac{q}{1-q}}\mid \mathcal F_c]^{1-q} }{\epsilon^q}\,\\
        &\le 6\eta^2B_1^2K^2m^2G^2 + 6\eta^{2p}B_2^2\epsilon^{2(1-p)}|\mathcal A|^{2p}\Pi_\ast^{2p} K^{2p}m^{2p}G^{2p}+ \frac{6\eta^q B_3^2|\mathcal A|^q \Pi_\ast^q K^q m^{q}G^q }{\epsilon^q}\,.
\end{align*}
\end{proof}
We are now ready to state and prove our main result, on $L^2$-gradient norms at parameters chosen uniformly at the beginning of epochs. The choice seems arbitrary, but it is an upper bound for the minimum of $L^2$-gradient norms over the learning process, which is often studied in SGD under weak assumptions.
\begin{theorem}
    Assume Assumptions \ref{ass:boundedrewards}, \ref{ass:unbiased-bounded-adv}, \ref{ass:score}, \ref{ass:Lipschitz} and suppose the learning rate $\eta$ is smaller than $\frac{1}{Lm}$. Sample $\tilde\theta$ uniformly from $\{\theta_{c,e,0}\mid c=0,\dots, C-1,\ e=0,\dots, K-1\}$.
Then for arbitrary $p,q \in (0,1)$ it holds that
\begin{align*} 
&\quad \min_{c=0,\dots, C-1,\ e=0,\dots, K-1} \E\big[|\nabla J(\theta_{c,e,0})|^2\big]\\
&\le \E\big[|\nabla J(\tilde \theta)\big|^2]\\
&\le \frac{2\Delta_0}{\eta CKm}+ 6\eta^2(7B_1^2+L^2)K^2m^2G^2 + 42 \eta^{2p}B_2^2\epsilon^{2(1-p)}|\mathcal A|^{2p}\Pi_\ast^{2p} K^{2p}m^{2p}  G^{2p}\\ &\quad + \frac{42\eta^q B_3^2|\mathcal A|^q \Pi_\ast^q K^q m^q  G^q }{\epsilon^q}+ \frac{12\sigma^2}{N} + 12T^2\Pi_\ast^2\delta^2\,,
\end{align*}
with constants defined in the lemmas above and $\Delta_0:=J_\ast-J(\theta_{0,0,0})$. 
\end{theorem}
\begin{proof}
    From Lemma~\ref{lem:perepochdescent} we have, since $\eta \le \frac{1}{Lm}$, that
    \begin{align*} 
    \E\big[J(\theta_{c,e+1,0})\big] \ge \E\big[J(\theta_{c,e,0})\big] + \frac{\eta m}2 \E\big[\big|\nabla J(\theta_{c,e,0})\big|^2\big] -\frac{\eta m}{2}\,\E\Big[\Big|\nabla J(\theta_{c,e,0})-\frac{1}{m}\sum_{k=0}^{m-1}\hat g_{c,e,k}\Big|^2\Big].
    \end{align*}
    By Lemma~\ref{lem:bound1} and Lemma~\ref{lem:bound2} we have 
    \begin{align*} 
    &\quad\E\Big[\Big|\nabla J(\theta_{c,e,0})-\frac{1}{m}\sum_{k=0}^{m-1}\hat g_{c,e,k}\Big|^2\Big]\\ &\le 6\eta^2(7B_1^2+L^2)K^2m^2G^2 + 42 \eta^{2p}B_2^2\epsilon^{2(1-p)}|\mathcal A|^{2p}\Pi_\ast^{2p} K^{2p}m^{2p}  G^{2p}\\ &\quad + \frac{42\eta^q B_3^2|\mathcal A|^q \Pi_\ast^q K^q m^q  G^q }{\epsilon^q}+ \frac{12\sigma^2}{N}+12T^2\Pi_\ast^2\delta^2\,.
    \end{align*}
    We take the average over $c=0,\dots, C-1$ and $e=0,\dots, K-1$ such that for $\tilde \theta$ uniformly sampled from the beginning of epoch $\{\theta_{c,e,0}\mid c=0,\dots, C-1,\ e=0,\dots, K-1\}$ we have
\begin{align*} 
&\quad\min_{c=0,\dots, C-1,\ e=0,\dots, K-1} \E\big[|\nabla J(\theta_{c,e,0})|^2\big]\\
&\le \frac{1}{CK}\sum_{c=0}^{C-1}\sum_{e=0}^{K-1}\E\big[|\nabla J(\theta_{c,e,0})|^2\big]\\
&\le \frac{2\Delta_0}{\eta CKm}+ 6\eta^2(7B_1^2+L^2)K^2m^2G^2 + 42 \eta^{2p}B_2^2\epsilon^{2(1-p)}|\mathcal A|^{2p}\Pi_\ast^{2p} K^{2p}m^{2p}  G^{2p}\\ &\quad + \frac{42\eta^q B_3^2|\mathcal A|^q \Pi_\ast^q K^q m^q  G^q }{\epsilon^q}+ \frac{12\sigma^2}{N}+12T^2\Pi_\ast^2\delta^2\,,
\end{align*}
where we divided both sides by $\frac{\eta m}{2}$ and used the telescoping sum 
\begin{align*}
\sum_{c=0}^{C-1} \sum_{e=0}^{K-1} \big(\E[J(\theta_{c,e+1,0}]-\E[J(\theta_{c,e,0}]\big)
=\E[J(\theta_{C,K,0})-J(\theta_{0,0,0})] \le J_\ast-\E[J(\theta_{0,0,0})] 
=\Delta_0.    
\end{align*}     
\end{proof}

\newpage 

\section{Finite-time GAE}
\label{sec: finite-time-gae}

\subsection{TD Errors, $k$-Step Advantage Estimators, and Standard GAE}
For the reader non-familiar with GAE (for infinite-time MDPs) this section collects the most important definitions. To construct estimators of the advantage function in the actor-critic framework, GAE relies on the notion of temporal-difference (TD) errors \cite{GAE}.
Given a value function approximation $V$ (typically from a value network),
the one-step TD-error at time~$t$ is defined as 
\[
\delta_t := R_t + \gamma  V(S_{t+1}) - V(S_t).
\]
If the value function approximation is the true value function the TD error is an unbiased estimator of the advantage: 
\begin{equation}\label{eq:td-unbiased}
\mathbb E[\delta_t \mid S_t, A_t]
=
\mathbb E\!\left[R_t + \gamma V^\pi(S_{t+1}) - V^\pi(S_t) \,\middle|\, S_t, A_t\right]
=
Q^\pi(S_t,A_t) - V^\pi(S_t)
=
\A^\pi(S_t,A_t),
\end{equation}
due to the Markov property and the  Bellmann-equation. Using TD errors \cite{GAE} define $k$-step advantage estimators that accumulate information from $k$ future steps before bootstrapping with~$V$: 
\begin{equation}\label{eq:kstep-adv-def}
\hat \A_t^{(k)}
:= \sum_{\ell=0}^{k}\gamma^\ell \delta_{t+\ell}
=
\sum_{\ell=0}^{k}\gamma^\ell R_{t+\ell}
+ \gamma^{k+1}  V(S_{t+k+1}) -  V(S_t)
\end{equation}
The second equality follows from a telescopic sum cancellation. Larger $k$ lead to more variance from the stochastic return and less value function approximation bias from the bootstrapping, with $k\approx \infty$ corresponding to the Monte Carlo advantage approximation. Conversely, small $k$ corresponds to less variance but more function approximation bias.

The generalized advantage estimator is an exponential mixture of all $k$-step advantage estimators. Using the geometric weights $(1-\lambda)\lambda^k$, for $\lambda\in(0,1)$ the original GAE estimator is defined as
\begin{equation}\label{eq:gae-infty-mixture}
\hat \A_t^\infty:= (1-\lambda)\sum_{k=0}^{\infty}\lambda^k\,\hat \A_t^{(k)}.
\end{equation}
The prefactor $(1-\lambda)$ normalizes the geometric weights so that $\sum_{k\ge 0}(1-\lambda)\lambda^k=1$.
Hence, \eqref{eq:gae-infty-mixture} is a convex combination of $k$-step estimators, with longer horizons
downweighted exponentially. The hyperparameter $\lambda$ is a continuous parameter that can interpolate between the large variance and large bias regimes. The mixture \eqref{eq:gae-infty-mixture} admits an equivalent compact
representation as a discounted sum of TD errors.
Indeed, inserting $\hat \A_t^{(k)}=\sum_{\ell=0}^{k}\gamma^\ell\delta_{t+\ell}$ and exchanging the order of summation yields
\begin{align}\label{eq:gae-infty-tdsum}
\hat \A_t^\infty
=
\sum_{\ell=0}^{\infty}(\gamma\lambda)^\ell\,\delta_{t+\ell}.
\end{align}
\begin{remark}[Indexing convention vs.\ \cite{GAE}]
The original GAE paper \cite{GAE} defines the $k$-step advantage with bootstrapping at time $t+k$.
In contrast, our definition \eqref{eq:kstep-adv-def} bootstraps at time $t+k+1$.
Equivalently, our $k$-step estimator corresponds to the $(k{+}1)$-step estimator in the indexing used in \cite{GAE}.
This is purely a notational shift chosen so that geometric mixtures take the form $\sum_{k\ge 0}\lambda^k \hat \A_t^{(k)}$.
\end{remark}

\subsection{Tail-Mass Collapse of GAE}
The sequences defined by \eqref{eq:gae-infty-mixture} and \eqref{eq:gae-infty-tdsum} are intrinsically related to infinite-horizon MDPs. They implicitly rely on the fact that $\delta_{t+\ell}$, and hence the MDP, is  defined for all future times. However, the GAE estimator sequence is used in practice for finite-time MDP settings such as PPO implementations. In this section we will point towards a finite-time side effect that we call tail-mass collapse. In the next sections we discuss alternatives to GAE in finite-time that avoid tail-mass collapse. 

Let as assume $T$ is a finite-time horizon and additionally $\tilde \tau := \inf\{\, t \ge 0 \;|\; S_t \text{ is terminal} \,\}$ is a termination time (such as landing in Lunar Lander). We denote by $\tau=\tilde \tau\wedge T$ the minimum of termination and $T$, the effective end of an episode. For instance, in PPO one collects rollouts until the end $\tau$ and then uses a backtracking recursion to compute advantage estimators. Without further justification, PPO in practice takes \eqref{eq:gae-infty-tdsum} and cancels TD-errors after termination:
\begin{align}\label{eq:gaeinf}
\hat \A_t
:=
\sum_{\ell=0}^{\tau-t-1}(\gamma\lambda)^\ell\,\delta_{t+\ell},\quad t\leq \tau.
\end{align}
In accordance with \cite{PPO}, we call this estimator truncated GAE. 
The form of \eqref{eq:gaeinf} is particularly useful as it gives
\begin{align*}
\hat \A_t= \delta_t+ \gamma \lambda \sum_{\ell=1}^{\tau-t-1}(\gamma\lambda)^{\ell-1}\,\delta_{t+\ell}
= \delta_t+ \gamma \lambda \sum_{\ell=0}^{\tau-t-2}(\gamma\lambda)^{\ell}\,\delta_{t+1+\ell}
= \delta_t+  \gamma \lambda \hat \A_{t+1},
\end{align*}
which results in an iterative computation scheme backwards in time. For a collected rollout up to $\tau$, one can directly  backtrack $\hat{\mathbb A}_t=\delta_t+\gamma\lambda \hat{\mathbb A}_{t+1}$ using the terminal condition  $\hat{\mathbb A}^\infty_\tau:=0$.

We now come to the tail-mass collapse of finite-time truncation of the GAE sequences. The scaling does not meet the original purpose. The geometric weights were originally distributed on $\N$, now they are restricted to $\{0,...,\tau\}$. The entire weights past $\tau$ collapse on the longest $k$-step estimator available, the one that is closest to Monte Carlo. It follows that the direct application of a truncated GAE sequence on rollouts of finite length has more variance/less bias then originally intended. Here is a formal proposition.
\begin{proposition}[Tail-mass collapse of GAE]\label{prop:app:tail-collapse}
Fix $t\in\{0,\dots,\tau-1\}$ and assume that the GAE estimator sequence $\hat \A_t ^\infty$ is given by \eqref{eq:gaeinf}. Then, 
\[
\hat \A_t
=
\sum_{k=0}^{\tau-t-2}\underbrace{(1-\lambda)\lambda^k}_{\text{GAE weights}}\,\hat \A_t^{(k)}
\;+\;
\underbrace{\lambda^{\tau-t-1}}_{\text{tail-mass collapse weight}}\,\hat \A_t^{(\tau-t-1)},
\]
with the convention that an empty sum equals zero.
\end{proposition}
Since an infinite number of weighs collapse into one, we call this feature of GAE applied to finite-time settings GAE tail mass collapse. Figure \ref{fig:gae-tail-collapse} of the main text visualizes the weights on different $k$-step estimators for four choices of $t$. The large blue atoms reflect the tail-mass collapse.
\begin{proof}
Fix $t\in\{0,\dots,\tau-1\}$. Let $\tilde \A_t := (1-\lambda)\sum_{k=0}^{\tau-t-2}\lambda^k\,\hat \A_t^{(k)}
\;+\;
\lambda^{\tau-t-1}\,\hat \A_t^{(\tau-t-1)}$.
We prove that the identity for $\tilde \A_t$ exactly yields \eqref{eq:gaeinf}.
Since $\hat \A_t^{(k)}=\sum_{\ell=0}^{k}\gamma^\ell \delta_{t+\ell}$, we obtain
\begin{align*}
\tilde \A_t 
&=(1-\lambda)\sum_{k=0}^{\tau-t-2}\lambda^k \sum_{\ell=0}^{k}\gamma^\ell \delta_{t+\ell}
\;+\;
\lambda^{\tau-t-1}\sum_{\ell=0}^{\tau-t-1}\gamma^\ell \delta_{t+\ell}.
\end{align*}
For the first term we swap the order of summation and use standard formulas for 
\begin{align*}
(1-\lambda)\sum_{k=0}^{\tau-t-2}\lambda^k \sum_{\ell=0}^{k}\gamma^\ell 
&=
(1-\lambda)\sum_{k=0}^{\tau-t-2}\lambda^k \sum_{\ell=0}^{\tau -t-2} \mathds{1}_{\{\ell \le k\}} \gamma^\ell \delta_{t+\ell}\gamma^\ell \delta_{t+\ell} 
= (1-\lambda)\sum_{\ell=0}^{\tau-t-2}\gamma^\ell \delta_{t+\ell}\sum_{k=\ell}^{\tau-t-2}\lambda^k \\
&=(1-\lambda)\sum_{\ell=0}^{\tau-t-2}\gamma^\ell \delta_{t+\ell}\;
\lambda^\ell \sum_{j=0}^{\tau-t-\ell-2}\lambda^j 
=\sum_{\ell=0}^{\tau-t-2}\gamma^\ell \delta_{t+\ell}\;\lambda^\ell\bigl(1-\lambda^{\tau-t-\ell-1}\bigr) \\
&=\sum_{\ell=0}^{\tau-t-2}(\gamma\lambda)^\ell \delta_{t+\ell}
\;-\;
\lambda^{\tau-t-1}\sum_{\ell=0}^{\tau-t-2}\gamma^\ell \delta_{t+\ell}.
\end{align*}
Plugging this back into the formula for $\tilde \A_t$, yields
\begin{align*}
\tilde \A_t
&=\sum_{\ell=0}^{\tau-t-2}(\gamma\lambda)^\ell \delta_{t+\ell}
-\lambda^{\tau-t-1}\sum_{\ell=0}^{\tau-t-2}\gamma^\ell \delta_{t+\ell}
+\lambda^{\tau-t-1}\sum_{\ell=0}^{\tau-t-1}\gamma^\ell \delta_{t+\ell}\\
&=\sum_{\ell=0}^{\tau-t-2}(\gamma\lambda)^\ell \delta_{t+\ell}
+\lambda^{\tau-t-1}\gamma^{\tau-t-1}\delta_{t+\tau-t-1}
=\sum_{\ell=0}^{\tau-t-1}(\gamma\lambda)^\ell \delta_{t+\ell}.
\end{align*}
This is exactly \eqref{eq:gaeinf}.
\end{proof}
Using the standard GAE mixture in finite horizon (or terminating) MDPs thus induces a pronounced weight collapse onto the final non-trivial estimator. At the same time, the original motivation of GAE is to perform a geometric TD$(\lambda)$-style averaging over $k$-step estimators \cite{GAE}. 

To mitigate tail mass collapse we suggest to take the rollout length into consideration when normalizing the geometric weights. We do not normalize with $(1-\lambda)$ but adaptively with $\frac{1-\lambda}{1-\lambda^{T-t}}$ or $\frac{1-\lambda}{1-\lambda^{\tau-t}}$ that gives the desired exponential weight to each summand. The resulting backwards induction will be identical to GAE except different scaling factor. 

\subsection{Fixed-Time GAE}\label{sec:finitetimegae}
We first consider the effect of deterministic truncation at the trajectory horizon $T$.
Even in the absence of early termination, standard GAE implicitly mixes $k$-step estimators over an infinite range of $k$, while only the estimators with $k\le T-t-1$ are supported by the data collected after time $t$. A natural finite-time analogue is therefore obtained by restricting the geometric mixture to the available range and renormalizing the weights to sum to one.
\begin{definition}[Fixed-time GAE]\label{def:fixed-time-gae}
Fix $\lambda\in(0,1)$ and a horizon $T\in\mathbb N$.
For any $t\in\{0,\dots,T-1\}$, the fixed-time GAE estimator is defined as
\begin{equation}\label{eq:fixed-time-gae-def}
\hat \A_t^{T}
:=
\frac{1-\lambda}{1-\lambda^{T-t}}
\sum_{k=0}^{T-t-1}\lambda^k\,\hat \A_t^{(k)}.
\end{equation}
\end{definition}
The normalization factor $\tfrac{1-\lambda}{1-\lambda^{T-t}}$ ensures that the geometric weights sum up to one, making $\hat \A_t^{T}$ a convex combination of the generally observable $k$-step estimators. This formulation yields a consistent fixed-horizon analogue of GAE that aligns with the data available from truncated trajectories.  

Similarly to GAE, this estimator admits a compact TD-sum representation and results in a recursion formula that can be used in practical implementations.

\begin{proposition}[Backward Recursion for fixed-time estimator]\label{prop:fixed-time recursion}
For $t=T-1,\dots,0$,
we have 
\begin{align*}
  \hat \A_t^T=   \sum_{\ell=0}^{T-t-1} 
(\gamma\lambda)^\ell\frac{1-\lambda^{T-t-\ell}}{1-\lambda^{T-t}}\delta_{t+\ell}.
\end{align*}
Moreover, if we set $\hat \A_T^T = 0$ the estimator admits the following recursion formula:
\begin{align*}
 \hat \A_t^{T}
\;=\;
\delta_t
\;+\;
\gamma\lambda\,
\underbrace{\frac{1-\lambda^{\,T-t-1}}{1-\lambda^{\,T-t}}}_{\text{new additional factor}}\;\hat \A_{t+1}^{T}, 
\qquad t = T-1, \dots, 0.
\end{align*}
\end{proposition}

\begin{proof}
Fix $t \in \{0, \dots , T-1\}.$
\begin{align*}
\hat{A}_t^T :&= \frac{1-\lambda}{1-\lambda^{T-t}} \sum_{k=0}^{T-t-1} \lambda^k \hat{A}_{t}^{(k)}
= \frac{1-\lambda}{1-\lambda^{T-t}}
\sum_{k=0}^{T-t-1} \lambda^k \sum_{\ell=0}^{T-t-1} \mathds{1}_{\{\ell \le k\}} \gamma^\ell \delta_{t+\ell} \\
&= \frac{1-\lambda}{1-\lambda^{T-t}}
\sum_{\ell=0}^{T-t-1} \gamma^\ell \delta_{t+\ell}
\sum_{k=\ell}^{T-t-1} \lambda^k
= \frac{1-\lambda}{1-\lambda^{T-t}}
\sum_{\ell=0}^{T-t-1} \gamma^\ell \delta_{t+\ell} \; \lambda ^l 
\sum_{k=0}^{T-t-l-1} \lambda^k \\
&= \frac{1-\lambda}{1-\lambda^{T-t}}
\sum_{\ell=0}^{T-t-1} \gamma^\ell \delta_{t+\ell}
\; \lambda ^l \frac{1-\lambda^{T-t-l}}{1-\lambda} 
= \frac{1-\lambda}{1-\lambda^{T-t}}
\sum_{\ell=0}^{T-t-1} \gamma^\ell \delta_{t+\ell}
\;  \frac{\lambda ^l-\lambda^{T-t}}{1-\lambda} 
\\
&= \sum_{\ell=0}^{T-t-1} 
\frac{\lambda^\ell - \lambda^{T-t}}{1-\lambda^{T-t}} \gamma^\ell \delta_{t+\ell}
= \sum_{\ell=0}^{T-t-1} 
\frac{1-\lambda^{T-t-\ell}}{1-\lambda^{T-t}}\delta_{t+\ell} (\gamma\lambda)^\ell.
\end{align*}

Thus, we can use above equation to obtain the recursive formula via induction in $t$. The base case follows easily with $\hat{\A}_{T}^T = 0$
as for $t = T-1,$ we have $\hat{\A}_{T-1}^T = \delta_{T-1}$.  
For general $t \in \{0, \dots , T-2\}$, above formula yields
\begin{align*}
\hat{\A}_t^T
= \delta_t
+ \sum_{\ell=1}^{T-t-1} \frac{\lambda^\ell - \lambda^{T-t}}{1-\lambda^{T-t}} \gamma^\ell \delta_{t+\ell}
= \delta_t
+ \sum_{\ell=0}^{T-t-2} \frac{\lambda^{\ell+1} - \lambda^{T-t}}{1-\lambda^{T-t}} \gamma^{\ell+1} \delta_{t+1+\ell}
= \delta_t
+ \frac{1-\lambda^{T-t-1}}{1-\lambda^{T-t}}\, \lambda \gamma \hat{\A}_{t+1}^T. \qquad \qedhere
\end{align*}
\end{proof}

Similar to infinite time-horizon GAE, in the idealized case where the true value function~$V^\pi_t$ of the policy~$\pi$ is used to compute the temporal-difference errors the estimator $\hat \A_t^{T}$ remains unbiased for the time-dependent advantage $\A_t^\pi(S_t,A_t)$.
 
\begin{proposition}
Suppose that the true value function~$V^\pi_t$ of $\pi$ is used in the TD-errors:
\[
    \delta_t := R_t + \gamma V^\pi_t(S_{t+1}) - V^\pi_t(S_t),
    \qquad \gamma \in (0,1).
\]
Then, for any $t \in \{0, \dots T-1\}$:
\[
    \mathbb E\!\left[\hat \A_t^{T} \mid S_t, A_t\right]
    \;=\;
    \A^\pi_t(S_t, A_t)
\]
\end{proposition}

\begin{proof}
For a fixed starting time~$t$, recall that the fixed-time estimator is defined as the geometrically weighted average of the $k$-step estimators. 
For any $k \le T-t-1$, we have due to the Bellman equation
\begin{align*}
    \mathbb E\left[\hat \A_t^{(k)} \mid S_t, A_t\right] 
     &= 
    \mathbb E\Big[\sum_{\ell=0}^{k}\gamma^\ell R_{t+\ell}
    + \gamma^{k+1}  V^\pi_t(S_{t+k+1})  -  V^\pi_t(S_t) \mid S_t, A_t \Big]= \A_t^\pi(S_t, A_t).
\end{align*}
As the fixed-time estimator is a normalized, geometrically weighted average of 
the $k$-step estimators using aboves identiy, yields
\begin{align*}
\mathbb E[\hat \A_t^{T} \mid S_t, A_t]
&= 
\frac{1-\lambda}{1-\lambda^{T-t}}
\sum_{k=0}^{T-t-1} \lambda^k
\, \mathbb E[\hat \A_t^{(k)} \mid S_t, A_t]
=
\frac{1-\lambda}{1-\lambda^{T-t}}
\sum_{k=0}^{T-t-1} \lambda^k
\, \A^\pi_t(S_t, A_t) 
= \A^\pi_t(S_t, A_t),
\end{align*}
since the geometric weights sum to one.
\end{proof}

So far, the fixed-time estimator $\hat \A_t^{T}$ was introduced as a principled way to account for truncation at the trajectory horizon $T$.
We now analyze what happens when the episode terminates before $T$. 
In this case, we have $\tilde\tau < T$ and thus $\tau < T$ and adopting the usual terminal state convention, this implies that all $k$-step estimators that would extend beyond the termination time coincide with the last nontrivial one, so that the fixed-time geometric mixture implicitly reallocates its remaining tail mass onto that final estimate.
The following proposition makes this weight collapse precise.

\begin{proposition}[Weight collapse of fixed-time GAE]
\label{prop:fixed-time-soft-collapse}
Assume that for all $\tau \leq u \leq T$ we have
\[
S_u = S_{\tilde\tau}, \qquad R_u = 0 \quad \text{and} \quad V(S_{\tilde\tau})= 0 .
\]
Then, the fixed time GAE estimator \eqref{eq:fixed-time-gae-def} admits the decomposition
\begin{equation}\label{eq:fixed-time-soft-collapse}
\hat \A_t^{T}
=
\frac{1-\lambda}{1-\lambda^{T-t}}
\sum_{k=0}^{\tau-t-2}\lambda^k \hat \A_t^{(k)}
\;+\;
\frac{\lambda^{\tau-t-1}-\lambda^{T-t}}{1-\lambda^{T-t}}  \,
\hat \A_t^{(\tau-t-1)}
\end{equation}
for any $t\in\{0,\dots,\tau-1\}$.
\end{proposition}
\begin{proof}
Fix $t\in\{0,\dots,\tau-1\}$.
By assumption, for any $u\ge \tau$ we have $R_u=0$ and $S_u=S_\tau$ and $ V_u(S_\tau)=0$.
Hence, for all $u\ge \tau$,
\[
\delta_u = R_u + \gamma V_{u+1}(S_{u+1})- V_u(S_u)=0+\gamma\cdot 0-0=0.
\]
Therefore, for any $\tau-t\leq k \leq T-t$,
\[
\hat \A_t^{(k)}
=\sum_{\ell=0}^{k}\gamma^\ell \delta_{t+\ell}
=\sum_{\ell=0}^{\tau-t-1}\gamma^\ell \delta_{t+\ell}
=\hat \A_t^{(\tau-t-1)},
\]
which again implies that all $k$-step advantage estimators that would require information beyond $\tau$ coincide with the last nontrivial one estimator $\hat \A_t^{(\tau-t-1)}.$
Thus, splitting the geometric sum at the last observable index $\tau-t-1$ yields
\begin{align*}
\sum_{k=0}^{T-t-1}\lambda^k \hat \A_t^{(k)}
&=
\sum_{k=0}^{\tau-t-1}\lambda^k \hat \A_t^{(k)}
\;+\;
\sum_{k=\tau-t}^{T-t-1}\lambda^k \hat \A_t^{(k)} \\
&=
\sum_{k=0}^{\tau-t-1}\lambda^k \hat \A_t^{(k)}
\;+\;
\hat \A_t^{(\tau-t-1)}\sum_{k=\tau-t}^{T-t-1}\lambda^k.
\end{align*}
The tail sum is a geometric series:
\[
\sum_{k=\tau-t}^{T-t-1}\lambda^k
=
\lambda^{\tau-t}\sum_{j=0}^{T-\tau-1}\lambda^j
=
\lambda^{\tau-t}\frac{1-\lambda^{T-\tau}}{1-\lambda}.
\]
Multiplying by the prefactor $(1-\lambda)/(1-\lambda^{T-t})$ yields
\begin{align*}
\hat \A_t^{T}
&=
\frac{1-\lambda}{1-\lambda^{T-t}}
\sum_{k=0}^{\tau-t-1}\lambda^k \hat \A_t^{(k)}
\;+\;
\frac{1-\lambda}{1-\lambda^{T-t}}
\lambda^{\tau-t}\frac{1-\lambda^{T-\tau}}{1-\lambda} \hat \A_t^{(\tau-t-1)}
\\
&= 
\frac{1-\lambda}{1-\lambda^{T-t}}
\sum_{k=0}^{\tau-t-1}\lambda^k \hat \A_t^{(k)}
\;+\;
\frac{\lambda^{\tau-t}-\lambda^{T-t}}{1-\lambda^{T-t}}\,
\hat \A_t^{(\tau-t-1)}
\\&=
\frac{1-\lambda}{1-\lambda^{T-t}}
\sum_{k=0}^{\tau-t-2}\lambda^k \hat \A_t^{(k)}
\;+\;
\Big(\frac{1-\lambda}{1-\lambda^{T-t}} \lambda^{\tau-t-1} + \frac{\lambda^{\tau-t}-\lambda^{T-t}}{1-\lambda^{T-t}} \Big) \,
\hat \A_t^{(\tau-t-1)}
\\&=
\frac{1-\lambda}{1-\lambda^{T-t}}
\sum_{k=0}^{\tau-t-2}\lambda^k \hat \A_t^{(k)}
\;+\;
\frac{\lambda^{\tau-t-1}-\lambda^{T-t}}{1-\lambda^{T-t}}  \,
\hat \A_t^{(\tau-t-1)},
\end{align*}
which is exactly \eqref{eq:fixed-time-soft-collapse}.
\end{proof}
If termination occurs before the end of an episode, i.e. $\tau<T$, equation \eqref{eq:fixed-time-soft-collapse} shows that the fixed-time estimator no longer performs a purely geometric averaging over genuinely distinct $k$-step estimators. Instead, the geometric tail mass that would be assigned to unobservable indices $\tau-t\leq k\leq T-t$ is effectively reallocated to the last nontrivial estimate $\hat \A_t^{(\tau-t-1)}$, again causing a similar weight collapse effect onto this term, though in a weaker form than when considering the standard estimator. The earlier the termination (i.e., the smaller $\tau-t$) and the larger $\lambda$, the larger the corresponding tail coefficient becomes, and hence the more $\hat \A_t^{T}$ concentrates on $\hat \A_t^{(\tau-t-1)}$ rather than distributing weight across the observed range of $k$. 

However, this motivates a termination-adaptive variant in which the geometric mixture is truncated at the effective end $\tau$ and renormalized accordingly, so that the estimator depends only on rewards and TD errors observed up to time $\tau-1$.

\subsection{Termination-Time GAE}\label{sec:stoppingtimegae}
As mentioned above we restrict the geometric averaging to the range of steps actually available before termination.
This leads to an estimator that depends on a random horizon, given by the episode’s termination-time.
For any $t\in\{0,\dots,\tau-1\}$, only the $k$-step estimators with $k\le \tau-t-1$ are fully supported by the observed rollout segment.
We therefore define the following renormalized geometric mixture.

\begin{definition}[Termination-time GAE]\label{def:termination-time-gae}
For any $t\in\{0,\dots,\tau-1\}$, the termination-time GAE estimator is defined as
\begin{equation}\label{eq:termination-time-gae-def}
\hat \A_t^{\tau}
:=
\frac{1-\lambda}{1-\lambda^{\tau-t}}
\sum_{k=0}^{\tau-t-1}\lambda^k\,\hat \A_t^{(k)}.
\end{equation}
\end{definition}

By construction, $\hat \A_t^{\tau}$ uses only information up to the effective end $\tau$. It depends solely on the rewards $\{R_{t},\dots,R_{\tau-1}\}$
and value-function evaluations along the states $\{S_t,\dots,S_\tau\}$.
When $\tau=T$, the estimator coincides with the fixed-time estimator $\hat \A_t^{T}$.
When $\tau<T$, it automatically adapts to the shorter available trajectory length and avoids mass collapse from the indices $k\geq \tau-t$ to $k=\tau-t-1$.
\begin{proposition}[Backward recursion for termination-time GAE]\label{prop:stoptime-recursion}
For any $t\in\{0,\dots,\tau-1\}$, the termination-time estimator admits the TD-sum representation
\begin{equation}\label{eq:stoptime-tdsum}
\hat \A_t^{\tau}
=
\sum_{\ell=0}^{\tau-t-1}
(\gamma\lambda)^\ell\,
\frac{1-\lambda^{\tau-t-\ell}}{1-\lambda^{\tau-t}}\,
\delta_{t+\ell}.
\end{equation}
Moreover, if we set $\hat \A_\tau^{\tau}=0$, then $\hat \A_t^{\tau}$ satisfies the backward recursion
\begin{equation}\label{eq:stoptime-backward-recursion}
\hat \A_t^{\tau}
=
\delta_t
+
\gamma\lambda\,
\frac{1-\lambda^{\tau-t-1}}{1-\lambda^{\tau-t}}\,
\hat \A_{t+1}^{\tau},
\qquad t=\tau-1,\dots,0.
\end{equation}
\end{proposition}

\begin{proof}
The proof is analogous to the proof of proposition \ref{prop:fixed-time recursion} by replacing the deterministic horizon $T$ with the (random) effective horizon $\tau$ and proceeding path-wise.
\end{proof}
Algorithm \ref{alg:gae} gives pseudocode for the termination-time GAE.

\subsection{Relation of the Estimators and Bias-Variance Tradeoff}
We now use Propositions~\ref{prop:app:tail-collapse} and~\ref{prop:fixed-time-soft-collapse} to relate the three estimators and then discuss some heuristics regarding their bias-variance tradeoff.
\begin{proposition}[Relations between standard, fixed-time, and termination-time GAE]
\label{prop:relations-three-gae}
Under the same assumption as in  \ref{prop:app:tail-collapse} and \ref{prop:fixed-time-soft-collapse}, the following identities hold for  $t\in\{0,\dots,\tau-1\}$:
\begin{align}
\hat \A_t
&=
(1-\lambda^{\tau-t})\,\hat \A_t^{\tau}
\;+\;
\lambda^{\tau-t}\,\hat \A_t^{(\tau-t-1)},
\label{eq:relation-infty-tau}
\\[3pt]
\hat \A_t^{T}
&=
\frac{1-\lambda^{\tau-t}}{1-\lambda^{T-t}}\,\hat \A_t^{\tau}
\;+\;
\frac{\lambda^{\tau-t}-\lambda^{T-t}}{1-\lambda^{T-t}}\,\hat \A_t^{(\tau-t-1)}.
\label{eq:relation-T-tau}
\end{align}
 In particular, for $\lambda\in(0,1)$ and $\tau\le T$, both $\hat \A_t$ and $\hat \A_t^T$ are convex
combinations of $\hat \A_t^\tau$ and the last nontrivial $k$-step estimator $\hat \A_t^{(\tau-t-1)}$.
\end{proposition}

\begin{proof}
We start with the relationship of the standard estimator $\hat \A_t$ and the termination-time-estimator $\hat \A_t^\tau$. 
By Proposition \ref{prop:tail-collapse-kstep} we have
\begin{equation}\label{eq: weight collaps std}
\hat \A_t
=
(1-\lambda)\sum_{k=0}^{\tau-t-2}\lambda^k\,\hat \A_t^{(k)}
\;+\;
\lambda^{\tau-t-1}\,\hat \A_t^{(\tau-t-1)}.
\end{equation}

By the definition of $\hat \A_t^\tau$ we can rewrite the partial geometric sum up to $\tau-t-2$ as
\begin{align*}
\hat \A_t^{\tau}
&=
\frac{1-\lambda}{1-\lambda^{\tau-t}}
\left(\sum_{k=0}^{\tau-t-2}\lambda^k\,\hat \A_t^{(k)}+\lambda^{\tau-t-1}\hat \A_t^{(\tau-t-1)}\right),
\end{align*}
and thus \begin{align}
(1-\lambda)\sum_{k=0}^{\tau-t-2}\lambda^k\,\hat \A_t^{(k)}
&=
(1-\lambda^{\tau-t})\,\hat \A_t^{\tau}
\;-\;
(1-\lambda)\lambda^{\tau-t-1}\,\hat \A_t^{(\tau-t-1)}.
\label{eq:sum-elim}
\end{align}

Substituting \eqref{eq:sum-elim} into \eqref{eq: weight collaps std} yields
\begin{align*}
\hat \A_t
&=
\Big((1-\lambda^{\tau-t})\,\hat \A_t^{\tau}-(1-\lambda)\lambda^{\tau-t-1}\hat \A_t^{(\tau-t-1)}\Big)
\;+\;
\lambda^{\tau-t-1}\hat \A_t^{(\tau-t-1)}
\\
&=
(1-\lambda^{\tau-t})\,\hat \A_t^{\tau}
\;+\;
\lambda^{\tau-t-1}\bigl(1-(1-\lambda)\bigr)\hat \A_t^{(\tau-t-1)}
\\
&=
(1-\lambda^{\tau-t})\,\hat \A_t^{\tau}
\;+\;
\lambda^{\tau-t}\hat A_t^{(\tau-t-1)},
\end{align*}
which proves \eqref{eq:relation-infty-tau}.
Similarly, for the relationship of the fixed-time estimator $\hat A_t^T$ and the termination-time estimator $\hat \A_t^\tau$, we have by proposition \ref{prop:fixed-time-soft-collapse}
\[
\hat \A_t^{T}
=
\frac{1-\lambda}{1-\lambda^{T-t}}
\sum_{k=0}^{\tau-t-2}\lambda^k\,\hat \A_t^{(k)}
\;+\;
\frac{\lambda^{\tau-t-1}-\lambda^{T-t}}{1-\lambda^{T-t}}\,\hat \A_t^{(\tau-t-1)}.
\]
Multiplying \eqref{eq:sum-elim} by $(1-\lambda^{T-t})^{-1}$ and substituting, yields
\begin{align*}
\hat \A_t^{T}
&=
\frac{1}{1-\lambda^{T-t}}
\Big((1-\lambda^{\tau-t})\,\hat \A_t^{\tau}-(1-\lambda)\lambda^{\tau-t-1}\hat \A_t^{(\tau-t-1)}\Big)
\;+\;
\frac{\lambda^{\tau-t-1}-\lambda^{T-t}}{1-\lambda^{T-t}}\,\hat \A_t^{(\tau-t-1)}
\\
&=
\frac{1-\lambda^{\tau-t}}{1-\lambda^{T-t}}\,\hat \A_t^{\tau}
\;+\;
\frac{-(1-\lambda)\lambda^{\tau-t-1}+\lambda^{\tau-t-1}-\lambda^{T-t}}{1-\lambda^{T-t}}\,\hat \A_t^{(\tau-t-1)}
\\
&=
\frac{1-\lambda^{\tau-t}}{1-\lambda^{T-t}}\,\hat \A_t^{\tau}
\;+\;
\frac{\lambda^{\tau-t}-\lambda^{T-t}}{1-\lambda^{T-t}}\,\hat \A_t^{(\tau-t-1)}. \qedhere
\end{align*}
\end{proof}

Proposition~\ref{prop:relations-three-gae} makes explicit that, once early termination occurs $\{\tau<T$\},
both $\hat \A_t$ and $\hat \A_t^T$ can be viewed as reweighted extensions of the termination-time mixture $\hat \A_t^\tau$.
The second component in \eqref{eq:relation-infty-tau} and \eqref{eq:relation-T-tau} is always the largest nontrivial $k$-step estimator
$\hat \A_t^{(\tau-t-1)}$, i.e.\ the estimator that uses the longest available lookahead before bootstrapping.
This term typically exhibits the smallest bootstrap bias (since it relies least on $\hat V$), but also the largest variance, as it aggregates the longest (discounted) sum of TD errors.

The termination-time estimator $\hat \A_t^\tau$ avoids assigning any additional mass to the tail beyond the observable range:
it averages only over the genuinely distinct $k$-step estimators supported by the data up to the effective end $\tau$.
For this reason, it is the most conservative choice from a variance perspective, and one should expect it to exhibit the smallest variance among the three
(holding $\gamma,\lambda$ fixed).
On the event $\{\tau=T\}$ (no termination within the rollout), the termination-time and fixed-time estimators coincide by definition, $\hat \A_t^\tau=\hat \A_t^T$.

When $\tau\ll T$, the fixed-time estimator $\hat \A_t^T$ still allocates additional geometric mass to the last nontrivial term
through the coefficient $\frac{\lambda^{\tau-t}-\lambda^{T-t}}{1-\lambda^{T-t}}$ in \eqref{eq:relation-T-tau}.
Compared to $\hat \A_t^\tau$, this increases emphasis on $\hat \A_t^{(\tau-t-1)}$, which heuristically decreases bias but increases variance.
The standard estimator $\hat \A_t$ exhibits the strongest form of this effect:
it assigns the full tail mass $\lambda^{\tau-t}$ to $\hat \A_t^{(\tau-t-1)}$ in \eqref{eq:relation-infty-tau},
and therefore should be expected to have the smallest bootstrap bias but the largest variance.

Finally, the differences between the estimators become most pronounced when $\tau-t$ is small,
i.e.\ when the effective trajectory suffix available after time $t$ is short (either due to very short episodes, or because $t$ lies close to $\tau$).
In this regime, even moderate values of $\lambda$ lead to substantial relative tail weights, and the convex combinations in
\eqref{eq:relation-infty-tau}-\eqref{eq:relation-T-tau} can differ significantly.
We further quantify this effect explicitly under toy assumptions on the TD-errors in \autoref{sec: toy gae}.

\begin{figure}[t]
    \centering
    \includegraphics[width=\linewidth]{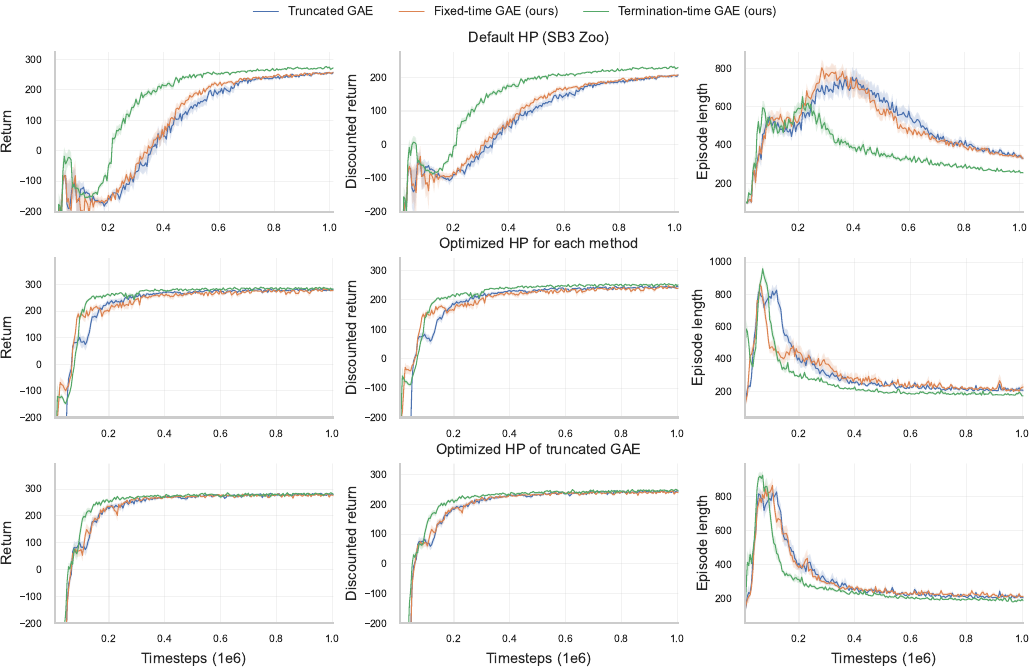}
    \caption{
    LunarLander-v3 evaluation learning curves under three hyperparameter (HP) regimes. Columns report evaluation return (sum of rewards), discounted evaluation return (sum of discounted rewards), and evaluation episode length. Curves are averaged over 20 seeds with standard errors of seeds as the shaded region. Rows correspond to hyperparameter regimes: Top: default PPO hyperparameters from the Stable-Baselines3 Zoo. Middle: best hyperparameters found separately for each method via hyperparameter optimization (100 trials, 3 seeds per trial). Bottom: all methods evaluated using the hyperparameters obtained from the truncated-GAE hyperparameter optimization (same hyperparameters for all methods).
    }
    \label{fig:gae_eval_overview}
\end{figure}

\begin{figure}[t]
    \centering
    \includegraphics[width=\linewidth]{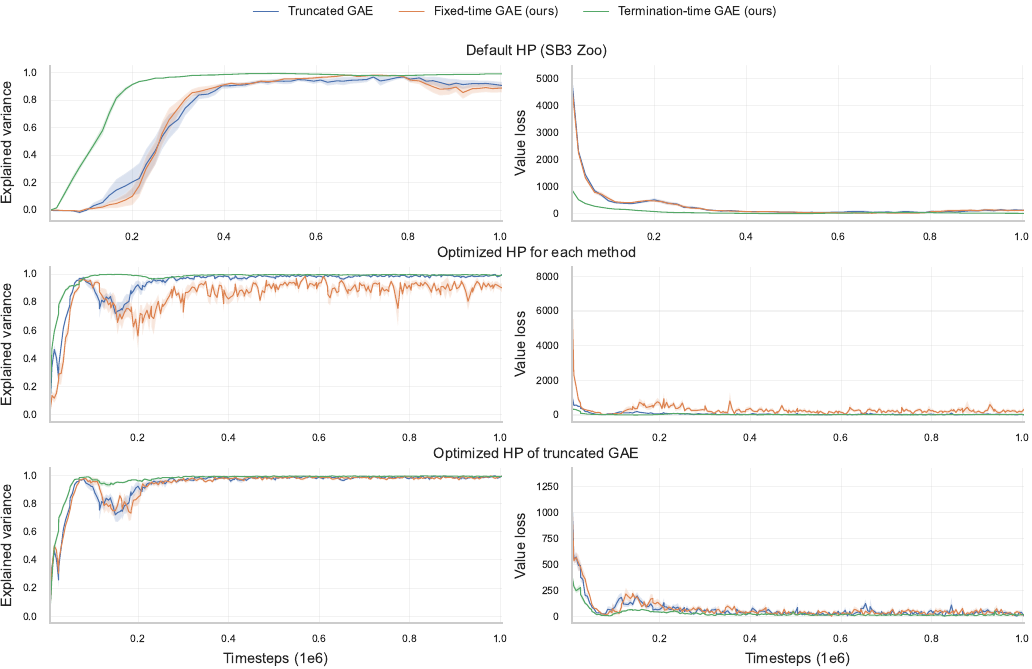}
    \caption{
    Training diagnostics on LunarLander-v3 for PPO with truncted GAE and our modified GAE variants. Rows correspond to hyperparameter (HP) regimes: Top: default PPO hyperparameters from the Stable-Baselines3 Zoo. Middle: best hyperparameters found separately for each method via hyperparameter optimization (100 trials, 3 seeds per trial). Bottom: all methods evaluated using the hyperparameters obtained from the truncated-GAE hyperparameter optimization (same hyperparameters for all methods). Columns show explained variance of the value function fit (left) and value loss (right). Curves are averaged over 20 seeds with standard errors of seeds as the shaded regions. }
    \label{fig:gae_train_overview}
\end{figure}
\subsection{Implementation Notes}
Even though this article has a strong focus on the mathematical foundations of PPO, we performed some experiments to highlight the usefulness of clean formulations, in particular for the finite-time use of infinite-time GAE estimator. We performed experiments on LunarLander-v3, working with Stable-Baselines3 \cite{stablebaslines3}.

Our theoretical definitions treat $(S_t,A_t,R_t)_{t\ge 0}$ as a stochastic process and define advantage estimators as random variables.
In an implementation, however, we only have access to finite realizations of this process, i.e.\ a collection of ordered transitions sampled under the current policy.
A rollout buffer therefore stores a finite set of ordered realizations, which we index by a single global index $n\in\{0,\dots,N-1\}$, even if the buffer contains multiple episodes:
\[
\mathcal B  
=
\bigl\{(s_n,a_n,r_n,v_n, t_n,d_n)_{n=0}^{N-1}\bigr\},
\]
where $n$ is a global buffer index (spanning multiple rollouts).
Here $s_n\in\mathcal S$ denotes state, $a_n\in\mathcal A$ the action,
$r_n$ the observed reward, $v_n = \hat V(s_n)$ the stored time-independent value estimate,
and $t_n\in\{0,\dots, T\}$ the within-episode time stamp of transition $n$. 
Furthermore, $d_n\in\{0,1\}$ is a done-mask indicating whether the next buffer entry belongs to the same episode, i.e. $d_n=1$ if transition $n+1$ continues the episode of transition $n$ and $d_n=0$ if an episode boundary occurs between $n$ and $n{+}1$ (either because the episode terminates at $n$, or because the rollout is truncated and a new episode starts at $n{+}1$).

Thus, the only additional information required beyond a standard PPO buffer is the within-episode time stamp $t_n$ for each transition. This is necessary because the recursion weights of Propositions~\ref{prop:fixed-time recursion} and~\ref{prop:stoptime-recursion} depend on the remaining distance to $T$ for the fixed-time estimator and $\tau$ for the termination-time estimator.

Building upon this propositions, both estimators can be computed with a single backward sweep over the buffer, $n=N-1,\dots,0$.
Define the one-step TD residual on the buffer by
\[
\delta_n := r_n + \gamma\, d_n\, v_{n+1} - v_n,
\]
where $v_{n+1}$ is understood as the bootstrap value for the next state.
Then, algorithmically, we iterate the buffer backwards, $n=N-1,\dots,0$, and maintain an effective end time $\tau_{\mathrm{eff}}$ for the episode segment to which the current transition belongs. Then in order to estimate advantages $\hat \A_n$  we use the following update scheme:
\begin{equation}\label{eq:impl-backward-update}
\hat \A_n
=
\delta_n
+
\gamma\lambda\,
\frac{1-\lambda^{\tau_{\mathrm{eff}} - t_n -1}}{1-\lambda^{\tau_{\mathrm{eff}} - t_n}}\; d_n\;\hat \A_{n+1}.
\end{equation}
In the fixed-time case, $\tau_{\mathrm{eff}}\equiv T$ is constant.
In the termination-time case, $\tau_{\mathrm{eff}}$ is updated whenever the backward sweep crosses an episode boundary, i.e.\ whenever $d_n=0$:
then $\tau_{\mathrm{eff}}$ is set to the effective end of the episode segment, which in buffer time corresponds to $\tau_{\mathrm{eff}}=t_n+1$ for the last transition of that segment.
The mask $d_n$ guarantees that the recursion resets across episode boundaries (since $d_n=0$ implies $\hat \A_n=\delta_n$),
so no information is propagated between different episodes stored in $\mathcal B$.
Finally, return targets used for value-function regression are obtained pointwise as $G_n = v_n + \hat \A_n$.

\subsection{Lunar Lander experiment}
We empirically compare the Stable-Baselines3 (SB3) PPO implementation \cite{stablebaslines3} with standard truncated GAE against our finite-time variants of GAE on LunarLander-v3, using a fixed discount factor $\gamma=0.999$. Throughout, we keep the PPO algorithm and architecture unchanged and modify only the advantage estimation (and thus the induced value targets), so that differences can be attributed to the estimator.

All comparisons are reported under three hyperparameter (HP) regimes.
First, we run each method with the SB3-Zoo default PPO hyperparameters.
Second, for each GAE method separately (trucnated, fixed-time, termination-time), we performed an hyperparameter optimization (HPO) with $100$ trials using a TPE sampler and no pruning. Each trial is evaluated on $3$ random seeds, and the objective is the final discounted evaluation return, aggregated across the $3$ seeds. The best HP configuration found for each method is then used for the learning curves.
Third, to test robustness and to rule out that gains are purely due to improved tuning, we additionally evaluate all methods using the HP configuration obtained from the truncated-GAE HPO (i.e., a single shared HP set for all estimators).
This yields a controlled comparison under default out-of-the-box settings, best achievable tuning per method, and a shared baseline-tuned configuration. The hyperparameter search spaces and the best configurations found by HPO for each advantage estimator are summarized in \autoref{tab:hpo_best_hp}.

\begin{table}[h]
\centering
\small
\setlength{\tabcolsep}{4pt}
\begin{tabular}{llccc}
\toprule
\textbf{Hyperparameter} & \textbf{Search space for HPOs} &
\textbf{Standard} & \textbf{Fixed-time} & \textbf{Termination-time}\\
\midrule
\multicolumn{5}{l}{\textbf{Rollout}} \\
n\_steps &
$2^{p}$, $p \sim \mathrm{Unif}\{5,\dots,12\}$ 
& 256 & 256 & 256 \\
\midrule
\multicolumn{5}{l}{\textbf{GAE}} \\
gae\_lambda &
$\lambda = 1-\epsilon$, $\epsilon \sim \mathrm{LogUnif}[10^{-4},10^{-1}]$
& 0.94748 & 0.99989 & 0.95827 \\
\midrule
\multicolumn{5}{l}{\textbf{Optimization}} \\
batch\_size &
$2^{p}$, $p \sim \mathrm{Unif}\{4,\dots,10\}$ 
& 128 & 256 & 16 \\
learning\_rate &
$\mathrm{LogUnif}[10^{-5},\,2\cdot 10^{-3}]$
& $1.037\!\times\!10^{-4}$ & $5.967\!\times\!10^{-4}$ & $1.687\!\times\!10^{-4}$ \\
n\_epochs &
$\mathrm{Cat}\{1,5,10,20\}$
& 20 & 20 & 10 \\
max\_grad\_norm &
$\mathrm{Unif}[0.3,\,2.0]$
& 0.3589 & 1.3625 & 1.8185 \\
\midrule
\multicolumn{5}{l}{\textbf{PPO objective / regularization}} \\
clip\_range &
$\mathrm{Cat}\{0.1,0.2,0.3,0.4\}$
& 0.3 & 0.2 & 0.1 \\
ent\_coef &
$\mathrm{LogUnif}[10^{-8},\,10^{-1}]$
& $1.323\!\times\!10^{-5}$ & $2.324\!\times\!10^{-4}$ & $5.482\!\times\!10^{-3}$ \\
target\_kl &
$\mathrm{LogUnif}[10^{-3},\,5.0]$
& 3.06 & $3.86\!\times\!10^{-2}$ & 1.75 \\
\midrule
\multicolumn{5}{l}{\textbf{Network}} \\
net\_arch &
$\mathrm{Cat}\{\texttt{[64]},\texttt{[64,64]},\texttt{[256,256]}\}$
& \texttt{[256,256]} & \texttt{[256,256]} & \texttt{[256,256]} \\
activation\_fn &
$\mathrm{Cat}\{\tanh,\mathrm{ReLU}\}$
& ReLU & ReLU & ReLU \\
\bottomrule
\end{tabular}
\caption{Hyperparameter search space used for TPE-based HPO (100 trials, 3 seeds per trial) and best configurations found for each estimator. The discount factor is fixed to $\gamma=0.999$. 
$\mathrm{Cat}\{\cdot\}$ denotes a categorical sampling (uniform over the listed values),
$\mathrm{Unif}[a,b]$ denotes continuous uniform sampling on $[a,b]$,
$\mathrm{Unif}\{a,\dots,b\}$ denotes uniform uniform sampling on $\{a,\dots,b\}$,
and $\mathrm{LogUnif}[a,b]$ log-uniform sampling on $[a,b]$.
}
\label{tab:hpo_best_hp}
\end{table}

During training, we interrupt learning every $5000$ environment steps and evaluate the current policy over $5$ episodes.
Evaluation metrics are reported in \autoref{fig:gae_eval_overview}, and training diagnostics in \autoref{fig:gae_train_overview}.
For each method and each HP regime, curves are averaged over $20$ independent training seeds. Shaded regions indicate standard errors across seeds.

Overall, the termination-time estimator consistently yields the fastest learning dynamics and the shortest episode lengths, indicating faster and more reliable landings. The performance gaps between estimators are most pronounced under the SB3-Zoo default PPO hyperparameters.
In this regime, the termination-time estimator learns substantially faster: it reaches high returns earlier and achieves shorter episode lengths throughout training.
The fixed-time estimator sometimes improves early learning compared to truncated GAE, but typically does not match the termination-time variant in either speed of return improvement or sustained reduction in episode length.

After optimizing HPs separately for each estimator, the qualitative ranking remains similar.
The termination-time estimator still shows the fastest increase in evaluation returns and achieves the smallest episode lengths.
The fixed-time estimator exhibits a steep initial improvement, but its learning curve later becomes similar to the truncated GAE variant. 

When all methods are evaluated using the hyperparameters obtained from optimizing truncated GAE, the termination-time estimator continues to learn faster.
In particular, both returns and landing speed (episode length) improve earlier than for truncated GAE and fixed-time GAE.
This suggests that the observed gains are not solely an artifact of per-method tuning, but reflect a more robust learning behavior induced by the termination-adaptive renormalization.

Figure~\autoref{fig:gae_train_overview} reports value-function explained variance and value loss during training, diagnostics for crtitic estimation.
Under default HPs, the termination-time estimator achieves an explained variance close to $1$ substantially earlier than the other estimators and maintains a markedly smaller value loss.
A plausible interpretation is that termination-time renormalization reduces the variance of the advantage labels and, consequently, the variance of the value targets
$G_n = v_n + \hat \A_n$ used for critic regression.
In this sense, the critic faces a better-conditioned supervised learning problem with less label noise, which allows faster stabilization of the value fit and, in turn, provides more reliable advantage estimates for the policy update.

Using the truncated-optimized HP configuration, the explained variance for all methods typically increases during early training and and then decreases for a short period before rising again.
Notably, this decrease occurs around the same time that evaluation episode lengths drop sharply, suggesting a training regime change in which the policy transitions from coarse control to consistently successful landings.
Such a transition can induce a pronounced shift in the visited state distribution and in the structure of returns, temporarily degrading the critic fit.
The termination-time estimator exhibits a substantially weaker drop in explained variance and maintains a smaller value loss during this phase, consistent with improved stability of the regression targets.

With per-method optimized HPs, termination-time and truncated GAE exhibit broadly similar critic diagnostics, whereas fixed-time can show a noticeably lower explained variance.
A likely contributing factor is that the HPO for fixed-time has selected a values of $\lambda$ closer to $1$, which increase emphasis on long-horizon components and may inflate the variance of both advantages and value targets, thereby making the critic fit more difficult even if the policy initially improves quickly.

\paragraph{Summary.}
Across all hyperparameter regimes, our termination-time GAE improves learning speed and landing efficiency on LunarLander-v3. The training diagnostics indicate that these gains coincide with faster and more stable critic learning (higher explained variance and lower value loss), which is consistent with the hypothesis that termination-adaptive renormalization reduces variance induced by finite-horizon truncation and early termination.

\subsection{Continuous Control Experiment}\label{sec: continuous control}
As our empirical evaluations on Lunar Lander indicate that the proposed termination-time GAE weight correction can yield substantial gains in environments with pronounced terminal effects, we further assess whether these improvements extend to continuous-control tasks. Since fixed-time GAE differs only marginally from standard truncated GAE when the effective horizon is large, we focus on the variant with the strongest empirical impact and compare termination-time GAE against standard truncated GAE on the MuJoCo \cite{MuJoCo} benchmark Ant-v4. 

Results are reported in Figure \ref{fig:ant_eval}. All runs use the default PPO hyperparameters from SB3-Zoo \cite{stablebaslines3}. The learning curves show that termination-time GAE remains competitive and can improve learning speed relative to truncated GAE.
We emphasize that this Ant-v4 study is only a minimal continuous-control check under default hyperparameters and short training time. A more comprehensive evaluation across MuJoCo tasks and tuning regimes is deferred to future work.

\begin{figure}[h]
    \centering
    \includegraphics[width=0.95\textwidth]{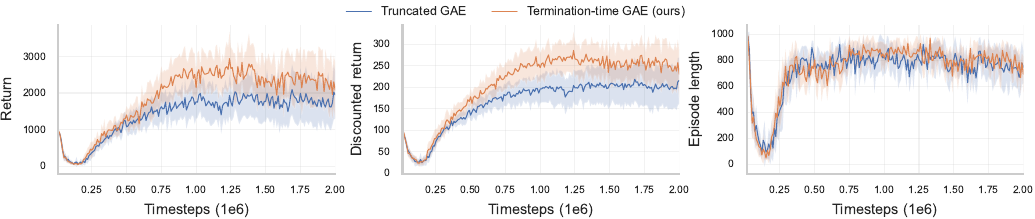}
    \caption{
    Ant-v4 evaluation learning curves comparing truncated GAE and ourtermination-time GAE under the default PPO hyperparameters from the Stable-Baselines3 Zoo. Columns report evaluation return (sum of rewards), discounted evaluation return (sum of discounted rewards), and evaluation episode length. Curves are averaged over $10$ seeds with standard errors across seeds shown as the shaded region.
  }
    \label{fig:ant_eval}
\end{figure}
\subsection{A Toy Model: Covariance Structure under iid TD-Errors}
\label{sec: toy gae}
This section introduces a deliberately simplified toy model designed to isolate the variance mechanism induced by the exponentially decaying TD-error aggregation of GAE.
We assume centered iid \ TD errors and derive closed-form expressions for the covariance structure of the resulting advantage sequence.
Within this controlled setting, we compare truncated GAE to our finite-time estimator by contrasting their covariance functions and, consequently, the variance patterns they induce across time. The assumptions are intentionally strong so that all quantities admit closed-form expressions and can be plotted.

Fix a finite horizon $T\in\mathbb N$ and parameters $\gamma,\lambda\in(0,1)$.
We model the temporal-difference errors $(\delta_t)_{t=0}^{T-1}$ as the random input driving GAE, and study the covariance structure induced on the resulting advantage estimates across time.
We use the following independence model. 
\begin{assumption}[Centered iid\ TD errors]
\label{ass:td-iid-centered}
The sequence $(\delta_t)_{t=0}^{T-1}$ is iid with
\[
\mathbb E[\delta_t]=0,
\qquad
\mathrm{Var}(\delta_t)=\sigma_\delta^2<\infty.
\]
\end{assumption}

We consider the advantage sequence produced by standard (finite-horizon truncated) GAE,
\begin{equation}
\label{eq: toy gae_std_def}
\hat \A_t
=
\sum_{j=0}^{T-t-1}(\gamma\lambda)^j\,\delta_{t+j}, \quad t\in\{0,\dots,T-1\}.
\end{equation}
and compare it to our finite-time renormalized variant (here in the fixed-horizon case $\tau=T$),
\begin{equation}
\label{eq:toy gae_fix_def}
\hat \A_t^{T}=
\sum_{j=0}^{T-t-1}(\gamma\lambda)^j\,
\frac{1-\lambda^{T-t-j}}{1-\lambda^{T-t}}\,\delta_{t+j}, , \quad t\in\{0,\dots,T-1\}.
\end{equation}
Our goal is to compare the temporal dependence induced by these two estimators.
To this end, under Assumption~\ref{ass:td-iid-centered} we derive closed-form expressions  for their covariance functions
and visualize the resulting covariance matrices and their differences via heatmaps.

The key step is an overlap decomposition. As the TD errors are independent, only shared TD-error terms contribute to the covariance. This yields closed-form formulas and a simple dominance argument for finite-time vs.\ truncated GAE.
\begin{lemma}[Covariance Function of truncated GAE]\label{lem: cov_std}
Under assumption \ref{ass:td-iid-centered} the sequence $(\hat \A_t)_{t=0,\dots,T-1}$ given by \eqref{eq: toy gae_std_def} satisfies for any $t  \in \{0, \dots, T-1\}$ and $k \in \{ 0, \dots ,T-t-1\}$:
    \begin{equation}\label{eq: pairwise covariance standard gae}
        \mathrm{Cov}[\hat \A_t, \hat \A_{t+k}] = \sigma^2_\delta (\gamma \lambda) ^ {k} \frac{1- (\gamma \lambda)^{2(T-t-k)}}{1- (\gamma \lambda)^2}.
    \end{equation}
\end{lemma}
\begin{proof}
Since the the TD-erros are assumed to be centered by assumption \ref{ass:td-iid-centered}, we have
\begin{align*}
  \mathrm{Cov}[\hat \A_t, \hat \A_{t+k}]
  &=
  \mathrm{Cov}\left [\sum_{l=0}^{T-t-1}(\gamma\lambda)^l\,\delta_{t+l}, \sum_{m=0}^{T-t-k-1}(\gamma\lambda)^m\,\delta_{t+k+m}
    \right]
\\&= 
\mathbb E \left[\left(\sum_{l=0}^{T-t-1} (\gamma \lambda)^l\,\delta_{t+l} \right) \left(\sum_{m=0}^{T-t-k-1} (\gamma \lambda)^m\,\delta_{t+k+m}\right)\right].    
\end{align*}
Moreover, as $(\delta_t)_{t \in \{0, \dots, T-1\}}$ are independent, $\mathbb E \left[\delta_{t+l} \,\delta_{t+k+m} \right] \neq 0$, only if $l = k+m$. Thus, 
\begin{align*}
\mathrm{Cov}[\hat \A_t, \hat \A_{t+k}]
&= 
 \sum_{l=0}^{T-t-1} \sum_{m=0}^{T-t-k-1} (\gamma \lambda)^{l+m}\;  \mathbb E \left[\delta_{t+l} \,\delta_{t+k+m} \right]
\\
&=
 \sum_{m=0}^{T-t-k-1} (\gamma \lambda)^{k+m+m}\;  \sigma^2_\delta 
\\
&=
\sigma^2_\delta (\gamma \lambda)^k \sum_{m=0}^{T-t-k-1} (\gamma \lambda)^{2m}\;  
\\
&= \sigma^2_\delta (\gamma \lambda) ^ {k} \frac{1- (\gamma \lambda)^{2(T-t-k)}}{1-(\gamma \lambda)^2}. \qedhere
\end{align*}
\end{proof}

Next, we compute the formula for the pairwise covariances estimators given by the finite-time GAE.
\begin{lemma}[Covariance Function of Finite-Time GAE]
\label{lem:cov_fixedtime}
Under assumption \ref{ass:td-iid-centered} the sequence $(\hat \A_t^T)_{t=0,\dots,T-1}$ given by \eqref{eq:toy gae_fix_def} satisfies for any $t  \in \{0, \dots, T-1\}$ and $k \in \{ 0, \dots ,T-t-1\}$:
\begin{align}
\label{eq:pairwise_cov_fixedtime_gae}
 \mathrm{Cov}[\hat \A_t^T, \hat \A_{t+k}^T]
&=
\frac{\sigma_\delta^2 (\gamma \lambda)^{k}}{(1-\lambda^{T-t})(1-\lambda^{T-t-k})}
\Bigg[
\frac{1-(\gamma\lambda)^{2(T-t-k)}}{1-(\gamma\lambda)^2}
- 2\lambda^{T-t-k}\frac{1-(\gamma^2\lambda)^{T-t-k}}{1-\gamma^2\lambda}
\\
&\hspace{6.4em}
+ \lambda^{2(T-t-k)}\frac{1-\gamma^{2(T-t-k)}}{1-\gamma^2}
\Bigg].
\nonumber
\end{align}
\end{lemma}

\begin{proof}
By \eqref{eq:toy gae_fix_def}, we have the TD-error representation
\[
\hat \A_t^{T}
=
\sum_{\ell=0}^{T-t-1} w_\ell^{\,t}\,\delta_{t+\ell},
\qquad
w_\ell^{\,t}:=(\gamma\lambda)^\ell\frac{1-\lambda^{T-t-\ell}}{1-\lambda^{T-t}}.
\]
As the td-erros are centered by Assumption \ref{ass:td-iid-centered}, we obtain
\begin{align*}
\Cov[\hat \A_t^{T},\hat \A_{t+k}^{T}]
&=
\Cov\!\left[\sum_{\ell=0}^{T-t-1} w_\ell^{\,t}\delta_{t+\ell},\;
           \sum_{m=0}^{T-t-k-1} w_m^{\,t+k}\delta_{t+k+m}\right] \\
&= \E\!\left[\left(\sum_{\ell=0}^{T-t-1} w_\ell^{\,t}\delta_{t+\ell} \right)\left(
           \sum_{m=0}^{T-t-k-1} w_m^{\,t+k}\delta_{t+k+m}\right)\right]
\end{align*}

and again $\E[\delta_{t+\ell}\,\delta_{t+k+m}]=0$ unless $\ell=k+m$, in which case it equals $\sigma_\delta^2$.
Therefore,
\begin{align*}
\Cov[\hat \A_t^{T},\hat \A_{t+k}^{T}]
&=
\sum_{\ell=0}^{T-t-1}\sum_{m=0}^{T-t-k-1} w_\ell^{\,t}\,w_m^{\,t+k}\,\E[\delta_{t+\ell}\,\delta_{t+k+m}]
\\
&=
\sigma_\delta^2\sum_{m=0}^{T-t-k-1} w_{k+m}^{\,t}\,w_m^{\,t+k}.
\end{align*}
Plugging in the explicit weights yields
\begin{align*}
\Cov[\hat \A_t^{T},\hat \A_{t+k}^{T}]
&=  \sigma_\delta^2\sum_{m=0}^{T-t-k-1}  
    (\gamma \lambda)^{k+m} \,  \frac{1-\lambda^{T-t-k-m}} {1-\lambda ^{T-t}} \,
    (\gamma \lambda)^{m} \,  \frac{1-\lambda^{T-t-k-m}} {1-\lambda ^{T-t-k}}  \\
 &= \frac{\sigma_\delta^2 (\gamma \lambda)^k }{(1-\lambda ^{T-t})(1-\lambda ^{T-t-k})} \,
    \sum_{m=0}^{T-t-k-1}  (\gamma \lambda)^{2m} \,  (1-\lambda^{T-t-k-m})^2 
\\
&= \frac{\sigma_\delta^2 (\gamma \lambda)^k }{(1-\lambda ^{T-t})(1-\lambda ^{T-t-k})} \, \sum_{m=0}^{T - t - k-1} (\gamma \lambda)^{2m} \Bigl(1 - 2\lambda^{T - t - k-m} + \lambda^{2(T - t - k-m)}\Bigr) \\
&= \frac{\sigma_\delta^2 (\gamma \lambda)^k }{(1-\lambda ^{T-t})(1-\lambda ^{T-t-k})} \Bigg [ 
    \underbrace{\sum_{m=0}^{T - t - k-1} (\gamma \lambda)^{2m}}_{=:S_1}
    \;-\;
    2 \underbrace{\sum_{m=0}^{T - t - k-1} (\gamma \lambda)^{2m} \lambda^{T - t - k-m}}_{=:S_2}
     \\&
      \qquad + \underbrace{\sum_{m=0}^{T - t - k-1} (\gamma \lambda)^{2m} \lambda^{2(T - t - k-m)}}_{=:S_3}
    \Bigg],
\end{align*}
with
\begin{align*}
S_1&=\sum_{m=0}^{T-t-k-1}(\gamma\lambda)^{2m}
=\frac{1-(\gamma\lambda)^{2(T-t-k)}}{1-(\gamma\lambda)^2},
\\
S_2&=\sum_{m=0}^{T-t-k-1}(\gamma\lambda)^{2m}\lambda^{T-t-k-m}
=\lambda^{T-t-k}\sum_{m=0}^{T-t-k-1}\gamma^{2m}\lambda^{m}
=\lambda^{T-t-k}\frac{1-(\gamma^2\lambda)^{T-t-k}}{1-\gamma^2\lambda},
\\
S_3&=\sum_{m=0}^{T-t-k-1}(\gamma\lambda)^{2m}\lambda^{2(T-t-k-m)}
=\lambda^{2(T-t-k)}\sum_{m=0}^{T-t-k-1}\gamma^{2m}
=\lambda^{2(T-t-k)}\frac{1-\gamma^{2(T-t-k)}}{1-\gamma^2}.
\end{align*}
Inserting $S_1,S_2,S_3$ into the covariance expression yields exactly \eqref{eq:pairwise_cov_fixedtime_gae}.
\end{proof}


\begin{lemma}[Truncated GAE Covariances are bigger]
\label{lem:cov_dom_fix_vs_std}
Let $(\hat \A_t)_{t=0}^{T-1}$ and $(\hat \A_t^{T})_{t=0}^{T-1}$ be the sequences given by \eqref{eq: toy gae_std_def} and \eqref{eq:toy gae_fix_def}.
Under Assumption \ref{ass:td-iid-centered}, for any $t  \in \{0, \dots, T-1\}$ and $k \in \{ 0, \dots ,T-t-1\}$:
\[
0\le \Cov[\hat \A_t^{T},\hat \A_{t +k}^{T}]
\;\le\;
\Cov[\hat \A_t,\hat \A_{t +k}].
\]
\end{lemma}
Insights into the structural origin of the variance behavior are provided by the covariance heatmaps in \autoref{fig:gae-heatmaps}, which visualize the covariance matrices induced by truncated finite-horizon GAE and our finite-time (renormalized) GAE variant under our toy assumptions.
Across all $(T,\lambda)$ configurations, finite-time exhibits uniformly smaller variances and covariances, i.e., $\mathrm{Cov}[\hat \A_s^{T},\hat \A_t^{T}] \le \mathrm{Cov}[\hat \A_s,\hat \A_t]$ entrywise.
This agrees with the domination result of Lemma \ref{lem:cov_dom_fix_vs_std} implied by expressing each advantage estimate as a weighted sum of iid TD-errors: fixed-time introduces an additional horizon-dependent attenuation of late TD-errors by multiplicative factors bounded by $1$, which can only reduce second moments. 
Varying $\lambda$ reveals how temporal correlations emerge from the the exponentially decaying TD-error aggregation. 
As $\lambda$ increases, the effective weights $(\gamma\lambda)^k$ decay more slowly with the temporal offset $k=|t-s|$, so advantage estimates at different times share a larger fraction of common TD-error terms.
In the heatmaps, this appears as a widening covariance band around the diagonal: for large $\lambda$, substantial covariance persists across larger time separations, whereas for smaller $\lambda$ the covariance is concentrated near the diagonal.

\begin{proof}[Proof of Lemma \ref{lem:cov_dom_fix_vs_std}]
Fix $0\le  t\le T-1$ and $0 \leq k \leq  T-t-1$.
From the TD-error representation of the finite-time estimator (see the proof of Lemma~\ref{lem:cov_fixedtime}), we have 
\[
\Cov[\hat \A_t^{T},\hat \A_{t+k}^{T}]
= \sigma_\delta^2\sum_{m=0}^{T-t-k-1}  
    (\gamma \lambda)^{k+m} \,  \frac{1-\lambda^{T-t-k-m}} {1-\lambda ^{T-t}} \,
    (\gamma \lambda)^{m} \,  \frac{1-\lambda^{T-t-k-m}} {1-\lambda ^{T-t-k}} .
\]
All summands are nonnegative, hence $\Cov[\hat \A_s^{T},\hat \A_t^{T}]\ge 0$.
Moreover, the weights are bounded,
\[
(\gamma \lambda)^{k+m} \,  \frac{1-\lambda^{T-t-k-m}} {1-\lambda ^{T-t}} \,
    (\gamma \lambda)^{m} \,  \frac{1-\lambda^{T-t-k-m}} {1-\lambda ^{T-t-k}}
\le (\gamma\lambda)^{k+m}(\gamma\lambda)^m
=(\gamma\lambda)^{k+2m}
\]
and therefore,
\[
\Cov[\hat \A_t^{T},\hat \A_{t+k}^{T}]
\le \sigma_\delta^2\sum_{m=0}^{T-t-k-1}(\gamma\lambda)^{k+2m}
= \Cov[\hat \A_t,\hat \A_{t+k}],
\]
where the last equality follows from the explicit overlap formula for truncated GAE 
(cf.\ proof of Lemma~\ref{lem: cov_std}).
\end{proof}

The discrepancy between truncated and fixed-time covariances is strongly localized near the end of the rollout (upper-right region of the matrices).
This localization follows directly from the fixed-time reweighting, which replaces standard geometric weighting by a renormalized scheme that downweights TD-errors close to the horizon by factors of the form $(1-\lambda^{L-j})/(1-\lambda^{L})$, where $L=T-t$ is the remaining horizon.
When $t$ is far from the terminal boundary (large $L$), these factors are close to $1$ over most of the relevant TD-errors, so the covariance structure matches truncated GAE in the bulk of the matrix.
When $t$ is near the boundary (small $L$), late TD-errors are substantially suppressed, yielding a pronounced covariance reduction that is visually strongest in the upper-right corner.
As $T$ increases, the fraction of indices that are close to the horizon shrinks, so the region where fixed-time materially differs from truncated becomes relatively smaller, even though entrywise domination continues to hold.
\newpage
\begin{figure}[h]
    \centering
    \includegraphics[]{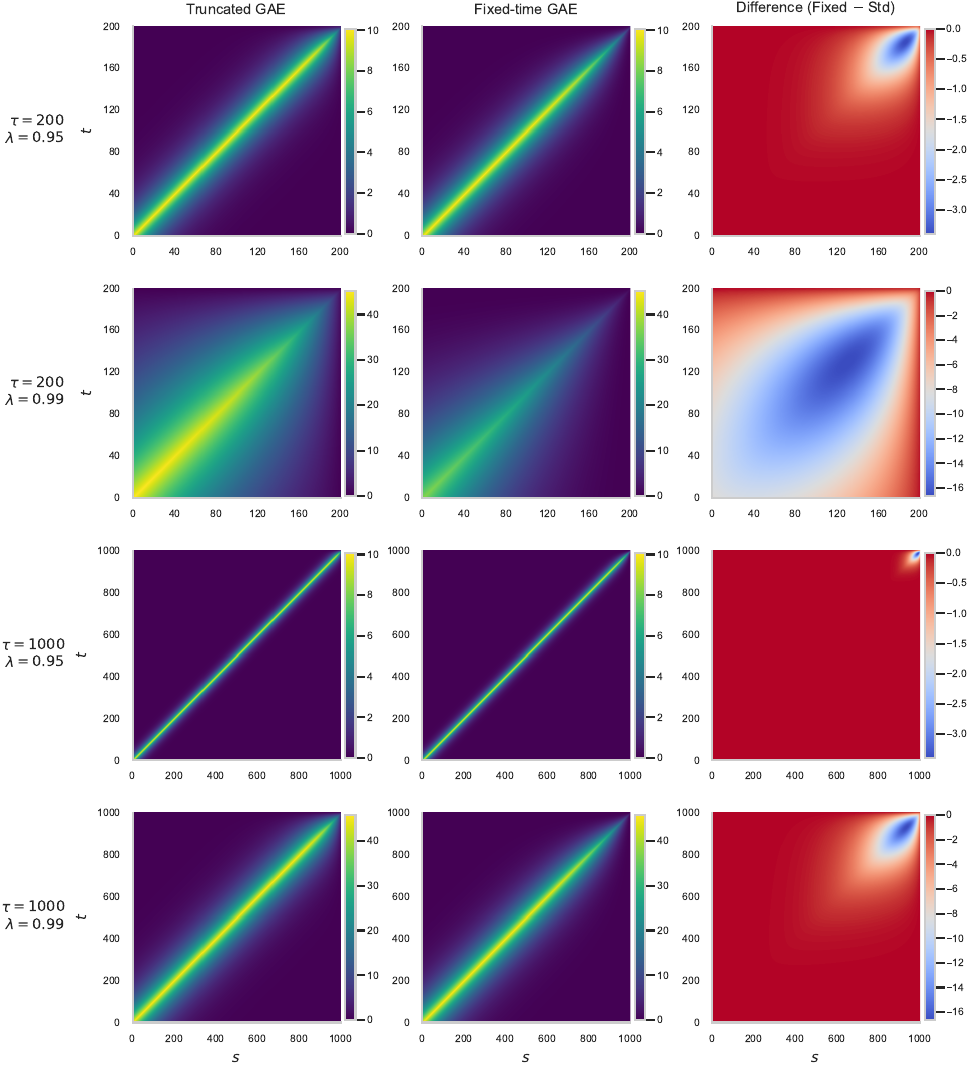}
    \caption{
        Covariance structure for trucnated vs.\ fixed-time GAE advantages for fixed $\gamma = 0.999$ and $\sigma_\delta^2 = 1.0.$.  
        For each $(T,\lambda)$ configuration (rows), we plot $\mathrm{Cov}[\hat \A_s, \hat \A_t]$ for truncated GAE (left) and finite-time GAE (middle), and the entrywise difference (right). 
    As predicted by the weight domination property, fixed-time covariances are uniformly smaller, with the largest reductions occurring near the end of the horizon.
    }
    \label{fig:gae-heatmaps}
\end{figure}

\end{document}